\documentclass[journal]{IEEEtran}

\usepackage{amsmath, amsthm, amssymb}

\usepackage{bm}
\usepackage{bbm}
\usepackage{dsfont}

\usepackage{graphicx}
\usepackage{color}
\usepackage[dvipsnames]{xcolor}
\usepackage{import}
\usepackage{subcaption}
\usepackage{algorithm}
\usepackage[bookmarks=true, bookmarksnumbered=true, pdfpagemode=UseOutlines]{hyperref}

\usepackage{algorithm}
\usepackage{algpseudocode}
\makeatletter
\let\OldStatex\Statex
\renewcommand{\Statex}[1][0]{%
  \setlength\@tempdima{\algorithmicindent}%
  \OldStatex\hskip\dimexpr#1\@tempdima\relax}
\makeatother

\algnewcommand\algorithmicinput{\textbf{Input:}}
\algnewcommand\Input{\item[\algorithmicinput]}
\algnewcommand\algorithmicoutput{\textbf{Output:}}
\algnewcommand\Output{\item[\algorithmicoutput]}
\algnewcommand\algorithmicglobalvars{\textbf{Global Variables:}}
\algnewcommand\GlobalVars{\item[\algorithmicglobalvars]}

\usepackage[numbers, sort&compress]{natbib}
\newcommand{\comment}[1]{}

\newtheorem{thm}{Theorem}

\newtheorem{lem}[thm]{Lemma}

\theoremstyle{definition}

\newtheorem{problem}{Problem}

\usepackage{accents}  %
\newcommand{\ubar}[1]{\underaccent{\bar}{#1}}

\newcommand{\R}{\mathbb{R}}

\DeclareMathOperator{\Sym}{\mathbb{S}}
\def \PSD{\mathbb{S}_{+}}

\DeclareMathOperator{\tr}{tr}
\DeclareMathOperator{\Span}{span}

\DeclareMathOperator{\Diag}{Diag}

\def \ones{\mathbbm{1}}
\def \transpose{^\mathsf{T}}

\DeclareMathOperator{\Graph}{graph}

\DeclareMathOperator{\Orthogonal}{O}
\DeclareMathOperator{\SO}{SO}

\DeclareMathOperator*{\argmax}{argmax}

\DeclareMathOperator{\Gaussian}{\mathcal{N}}  %
\DeclareMathOperator{\Langevin}{Langevin}  %

\def \Graph {G}  %
\def \Nodes {\mathcal{V}}  %
\def \Edges {\mathcal{E}}  %
\newcommand{\directed}[1]{\vec{#1}}  %
\def \dEdges{\directed{\Edges}}  %

\def \edge{\lbrace i,j \rbrace}  %
\def \dedge{(i,j)}  %

\def \incEdges{\delta}

\def\Lap{L}  %

\def \tran{t}
\def \rot{R}

\newcommand{\true}[1]{\ubar{#1}}
\newcommand{\noisy}[1]{\tilde{#1}}
\newcommand{\est}[1]{\hat{#1}}

\def \ttran{\true{\tran}}
\def \trot{\true{\rot}}

\def \optsym{*}

\def \ntran{\noisy{\tran}}
\def \nrot{\noisy{\rot}}

\def \rotsym{\rho}  %
\def \transym{\tau}  %

\def \RotW{W^{\rotsym}}  %
\def \LapRotW{\Lap}  
\def \LapRotWO{\Lap^{f}}
\def \LapRotWC{\Lap^{c}}

\def \projsym{\Pi}
\newcommand{\rounded}[1]{\projsym(#1)}
\newcommand{\nearestRounded}[1]{\projsym_{N}{(#1)}}

\def \Sphere{Sphere}
\def \Grid{Grid}
\def \Torus{Torus}

\def \selection{x}

\def \EO{\Edges^{f}}
\def \EC{\Edges^{c}}
\def \Eopt{\Edges^{\optsym}}

\def \objectiveF{f}

\def \dualF{\objectiveF_{D}}

\def \optValAC{p^{\optsym}}

\def \supergradient{g}
\def \Laplacian{\Lap}
\def \weight{w}

\def \City10K{\emph{City10K}}
\def \Intel{\emph{Intel}}
\def \AIS2Klinik{\emph{AIS2Klinik}}

\def \Sphere{\emph{Sphere}}
\def \Torus{\emph{Torus}}
\def \Grid{\emph{Grid}}

\def \MAC{MAC}

\title{MAC: Graph Sparsification by Maximizing Algebraic Connectivity}

\author{Kevin J. Doherty$^{1}$, Alan Papalia$^{1}$, Yewei Huang$^{2}$, David M. Rosen$^{3}$, Brendan Englot$^{2}$, and John J. Leonard$^{1}$%
  \thanks{$^{1}$Massachusetts Institute of Technology (MIT), Cambridge, MA 02139.
    \texttt{\{kdoherty, apapalia, jleonard\}@mit.edu}.\\
    $^{2}$Stevens Institute of Technology, Hoboken, NJ 07030. \texttt{\{yhuang85, benglot\}@stevens.edu} \\
    $^{3}$Northeastern University, Boston, MA 02115. \texttt{d.rosen@northeastern.edu}}%
}%

\begin{document}
\maketitle

\begin{abstract}
  Simultaneous localization and mapping (SLAM) is a critical capability in
  autonomous navigation, but memory and computational limits make long-term
  application of common SLAM techniques impractical; a robot must be able to
  determine what information should be retained and what can safely be
  forgotten. In graph-based SLAM, the number of edges (measurements) in a pose
  graph determines both the memory requirements of storing a robot's
  observations and the computational expense of algorithms deployed for
  performing state estimation using those observations, both of which can grow
  unbounded during long-term navigation. Motivated by these challenges, we
  propose a new general purpose approach to sparsify graphs in a manner that
  maximizes \emph{algebraic connectivity}, a key spectral property of graphs
  which has been shown to control the estimation error of pose graph SLAM
  solutions. Our algorithm, MAC (for \emph{maximizing algebraic connectivity}),
  is simple and computationally inexpensive, and admits formal \emph{post hoc}
  performance guarantees on the quality of the solution that it provides. In
  application to the problem of pose-graph SLAM, we show on several benchmark
  datasets that our approach quickly produces high-quality sparsification
  results which retain the connectivity of the graph and, in turn, the quality
  of corresponding SLAM solutions.
  
\end{abstract}


\section*{Supplementary Material}
\label{sec:supplemental-material}

Open source code implementing \MAC{} and reproducing our experimental results on
benchmark pose-graph SLAM datasets has been made available
\url{https://github.com/MarineRoboticsGroup/mac.git}.

\section{Introduction}
\label{sec:intro}

The problem of simultaneous localization and mapping (SLAM), in which a robot
aims to jointly infer its pose and the location of environmental landmarks, is a
critical capability in autonomous navigation. However, as we aim to scale SLAM
algorithms to the setting of lifelong autonomy, particularly on compute- or
memory-limited platforms, a robot must be able to determine \emph{what
  information should be kept, and what can safely be forgotten}
\cite{rosen2021advances}. In particular, in the setting of graph-based SLAM and
rotation averaging, the number of edges in a measurement graph determines both
the memory required to store a robot's observations as well as the computation
time of algorithms employed for state estimation using this measurement graph.

While there has been substantial work on the topic of measurement pruning (or
\emph{sparsification}) in lifelong SLAM (e.g. \cite{carlevarisbianco13iros,
  carlevaris2014conservative, johannsson12rssw, kurz2021geometry,
  kretzschmar2012information}), most existing methods rely on heuristics for
sparsification whereby little can be said about the quality of the statistical
estimates obtained from the sparsified graph versus the original. Recent work on
performance guarantees in the setting of pose-graph SLAM and rotation averaging
identified the spectral properties---specifically the \emph{algebraic
  connectivity} (also known as the \emph{Fiedler value})---of the measurement
graphs encountered in these problems to be central objects of interest,
controlling not just the \emph{best possible} expected performance (per earlier work on
Cram\'er-Rao bounds \cite{boumal2014cramer, chen2021cramer,
  khosoussi2014novel}), but also the \emph{worst-case} error of estimators
\cite{rosen2019se, doherty2022performance}.
These observations suggest the algebraic connectivity as a natural measure of
graph quality for assessing SLAM graphs. This motivates our use of the algebraic
connectivity as an \emph{objective} in formulating the graph sparsification
problem.

Specifically, we propose a spectral approach to pose graph sparsification which
maximizes the algebraic connectivity of the measurement graph subject to a
constraint on the number of allowed edges.\footnote{Our method is closely
  related to, but should not be confused with the algorithms typically employed
  for \emph{spectral sparsification} \cite{spielman2011spectral} (see Section
  \ref{sec:lap-eig-opt} for a discussion).} This corresponds to the well-known \emph{E-optimal}
design criterion from the theory of optimal experimental design (TOED)
\cite{pukelsheim2006optimal}. The specific problem we formulate turns out to be
an instance of the \emph{maximum algebraic connectivity augmentation} problem,
which is NP-Hard \cite{mosk2008maximum}. Consequently, computing globally
optimal solutions may not be feasible in general. To address this, we propose to
solve a computationally tractable convex relaxation and \emph{round} solutions obtained
to the relaxed problem to approximate feasible solutions of the original
problem. Relaxations of this form have been considered previously; in particular,
\citet{ghosh2006growing} developed a semidefinite program relaxation to solve
problems of the form we consider. However, these techniques do not scale to the
size of typical problems encountered in graph-based SLAM. To this end, we
propose a first-order optimization approach that we show is practically fast for
even quite large SLAM problems. Moreover, we show that the \emph{dual} to our
relaxation provides tractable, high-quality bounds on the suboptimality of the
solutions we provide \emph{with respect to the original problem}.

The \MAC{} algorithm was initially proposed and presented in our prior work
\cite{doherty2022spectral}. Since our original proposal, it has been
successfully applied and adapted to several new settings, such as measurement
selection for multi-robot navigation under bandwidth constraints in
\emph{Swarm-SLAM} \cite{lajoie2024swarm} and the \emph{OASIS} algorithm for
sensor arrangement \cite{kaveti2023oasis}. The recent work of
\citet{nam2023spectral} extended the MAC algorithm by removing the edge budget
parameter and instead introducing a regularization term based on the largest
eigenvalue of the graph adjacency matrix (see Section \ref{sec:applications}).
This manuscript expands upon our prior work \cite{doherty2022spectral} in
several key ways: it is self-contained, including all proofs and derivations; we
implement a new dependent randomized rounding procedure which we show offers a
dramatic improvement in the quality of results over our initially proposed
approach; and we provide several new experimental results, including new
comparisons with prior work \cite{khosoussi2019reliable} which specifically
optimizes the \emph{D-optimality} criterion, new benchmark datasets, and a
multi-session lidar SLAM application.

The remainder of this paper proceeds as follows: In Section
\ref{sec:related-work} we discuss recent results on the importance of algebraic
connectivity in the context of pose-graph SLAM, previous work on the maximum
algebraic connectivity augmentation problem, and prior work on network design
with applications in pose graph sparsification. Section \ref{sec:background}
provides background on relevant mathematical preliminaries in graph theory, our
application to pose-graph SLAM, and an introduction to the problem of maximum
algebraic connectivity augmentation. In Section \ref{sec:methods}, we formulate
the \MAC{} algorithm: we present a convex relaxation of the algebraic
connectivity maximization problem, present a first-order optimization approach
for solving the relaxation, and describe simple rounding procedures for
obtaining approximate solutions to the original problem from a solution to the
relaxed problem. In Section \ref{sec:experimental-results} we demonstrate our
approach in application to pose-graph SLAM, comparing our approach with the
greedy \emph{D-optimal} sparsification procedure of
\citet{khosoussi2019reliable} as well as a na\"ive ``topology unaware''
baseline. We show on several benchmark datasets and a multi-session SLAM
scenario using real data from the University of Michigan North Campus Long Term
(NCLT) dataset that our method is capable of \emph{quickly} producing
well-connected sparse graphs that retain the accuracy of the maximum-likelihood
estimators employed to solve pose-graph SLAM using the sparsified graphs.

\section{Related Work}
\label{sec:related-work}

Laplacian eigenvalue optimization is a rich area of study with numerous
applications to numerical computing (such as in the design of preconditioners),
networking, and robot perception, control, and estimation. A comprehensive
overview of spectral graph theory and its relevant application areas is well
beyond the scope of this paper. In our discussion of related works, we loosely
group previous works into the more general aspects of Laplacian eigenvalue
optimization problems (Section \ref{sec:lap-eig-opt}) including problem
formulations and solution techniques, and robotics-specific applications
(Section \ref{sec:applications}) related to spectral properties of graphs (and
algorithms that make use of them).\footnote{For additional background and contemporary applications in robotics, we would also refer the interested reader to the \emph{Robotics: Science and Systems} 2023 Workshop on Spectral Graph-Theoretic Methods for Estimation and Control: \url{https://sites.google.com/view/sgtm2023}.} Many papers considering fundamental aspects
of Laplacian eigenvalue optimization problems do so in the context of an
application of interest (and likewise, many applications motivate new solution
techniques, for example), and so necessarily there is considerable overlap in
our own discussions.

We particularly emphasize existing works making use of the algebraic
connectivity, but we also aim to place the algebraic connectivity (as an object
of study, and as a quantity to be optimized) in context as it relates to other
important spectral properties of graphs (and their applications in robotics).

\subsection{Laplacian eigenvalue optimization}\label{sec:lap-eig-opt}

The spectral properties of graph Laplacians are often closely related to
properties of problems parameterized by those graphs. For example, spectral
properties are key determiners of mixing times in Markov processes, effective
resistance of a network, and error in consensus estimation problems
\cite{boyd2006convex}. The importance of algebraic connectivity specifically has
been recognized since at least 1973, with the seminal work of
\citet{fiedler1973algebraic}. This, in turn, makes the spectral properties of a
graph's Laplacian a natural choice of objective for optimization when
considering problems of graph design, where we would like to find the
\emph{best} graph for a particular task.

The theory of optimal experimental design (TOED) \cite{pukelsheim2006optimal}
exposes another close connection with Laplacian eigenvalue optimization.
Specifically, TOED, as its name suggests, seeks to identify a set of
measurements that will enable a parameter of interest to be determined as
accurately as possible. So-called A-optimality, T-optimality, E-optimality, and
D-optimality are common criteria, each of which corresponds to optimizing a
different property of the \emph{information matrix} describing the distribution
of interest. Briefly, A-optimal designs minimize the trace of the inverse of the
information matrix, D-optimal designs maximize the determinant of the
information matrix, E-optimal designs maximize the smallest eigenvalue of the
information matrix, and T-optimal designs maximize the trace of the information
matrix. For statistical estimation problems parameterized by graphs, the
information matrix commonly corresponds closely with the graph Laplacian (see,
e.g. \cite{khosoussi2014novel,chen2021cramer} for the case of SLAM), making each of these
optimality criteria functions of the Laplacian spectrum. For example,
D-optimality corresponds to maximize the determinant of the Laplacian, and
E-optimal designs maximize the smallest (nonzero) eigenvalue of the graph
Laplacian, which is precisely the algebraic connectivity of the graph.

The combinatorial nature of graphs, as well as the nonlinear relationship between the
edges of a graph and the spectrum of its Laplacian, make Laplacian eigenvalue
optimization problems challenging to solve computationally (see, e.g. \cite{mosk2008maximum}). To
alleviate this challenge, many previous works have considered convex (especially
semidefinite program (SDP)) relaxations of otherwise challenging optimization
problems involving Laplacian eigenvalues (see \citet{boyd2006convex} for an
overview). These can be viewed as \emph{eigenvalue optimization} problems more
generally, the properties of which have been studied extensively (see, e.g.
\cite{lewis1996eigenvalue,lewis2003mathematics,shapiro1995eigenvalue}).
Obtaining a globally optimal solution to the original problem from a solution to
the relaxation is also nontrivial. Commonly, ``rounding'' procedures are
introduced which allow us to compute an approximate solution to the original
problem from a solution to the relaxation. In some cases, search methods (such
as branch-and-bound) or mixed-integer programming can be used to solve the
original combinatorial problem while gaining some efficiency through the use of
relaxations to prune parts of the solution space.

The specific problem of maximizing the algebraic connectivity subject to
cardinality constraints, which we consider in this paper, has been explored
previously for a number of related applications. \citet{ghosh2006growing}
consider a semidefinite program relaxation of the algebraic connectivity
maximization problem. \citet{nagarajan2018maximizing}, and more recently
\citet{somisetty2023optimal}, considered a mixed-integer optimization approach.
While the \MAC{} algorithm could make use of any solution method for the
relaxation we consider (or the original, constrained problem), unfortunately
neither of these methods scales to the types of problems we are considering. For
example, \citet{nagarajan2018maximizing} provides methods which are capable of
solving algebraic connectivity maximization problems to optimality, but for
graphs with tens of nodes, computation times for the these methods are on the
order of seconds to minutes. The graphs commonly encountered in SLAM (and which
we consider here) typically have \emph{thousands} of nodes and edges.\footnote{The
  monograph of \citet{boyd2006convex} discusses practical scalability issues
  with semidefinite program relaxations of Laplacian eigenvalue optimization
  problems and the potential computational improvements afforded by a
  subgradient approach for large-scale problems, which is exactly the approach
  we take in our development of \MAC{}.}

A related line of research is that of \emph{spectral sparsifiers}
\cite{spielman2011spectral}. Spectral sparsifiers are typically constructed as
sparse (reweighted) subgraphs (often by a combination of random sampling and
reweighting of edges) with the goal of preserving, to within a desired
tolerance) the \emph{entire} spectrum of the Laplacian. In contrast, the \MAC{}
algorithm proposed here aims to preserve the algebraic connectivity (a single
eigenvalue) of the Laplacian. While the construction of the MAC algorithm
affords it some weak spectral preservation properties, we are unable to, for
example, guarantee approximation of the Laplacian spectrum to within some
tolerance specified \emph{a priori} in the same way that spectral sparsifiers
aim to do.

\subsection{Applications of graph sparsification in robotics}\label{sec:applications}

Spectral properties of graphs encountered in SLAM appear in both error analysis
of maximum-likelihood estimators (e.g. in Cram\'er-Rao lower bounds) as well as
in the design of high-quality graphs (whether through graph sparsification
\cite{khosoussi2014novel, khosoussi2019reliable} or \emph{active SLAM}
\cite{chen2021cramer}). Algebraic connectivity specifically has appeared in the
context of rotation averaging \cite{boumal2014cramer}, linear SLAM problems and
sensor network localization \cite{khosoussi2014novel, khosoussi2019reliable},
and pose-graph SLAM \cite{chen2021cramer} as a key quantity controlling
estimation performance. In particular, \citet{boumal2014cramer} observed that
the inverse of the algebraic connectivity bounds (up to constants) the
Cram\'er-Rao lower bound on the expected mean squared error for rotation
averaging. More recently, it also appeared as the key quantity controlling the
\emph{worst-case} error of estimators applied to measurement graphs in
pose-graph SLAM and rotation averaging \cite{rosen2019se,
  doherty2022performance} (where \emph{larger} algebraic connectivity is
associated with lower error). More practically, the E-optimality
criterion is both less computationally expensive to compute \emph{and} to
optimize.

\citet{khosoussi2019reliable} established many of the first results for optimal
graph sparsification (i.e. measurement subset selection) in the setting of SLAM
(notably, they examine the somewhat broader problem of \emph{estimation over
  graphs}). They consider an approach based on the D-optimality criterion
(corresponding to the product of the Laplacian eigenvalues). The convex
relaxation they consider is perhaps the closest existing work in the SLAM
literature to ours. However, being based on an SDP relaxation, its runtime makes
its application impractical at the scale of many SLAM problems. They also
introduce the Greedy ESP algorithm, based on greedy submodular optimization,
which offers better runtime and, as they show, often provides solutions close in
quality to their convex relaxation. Our experimental results provide a
comparison with this approach, demonstrating that \MAC{} and Greedy ESP often
produce very similar results in practice, but \MAC{} can often be significantly
faster. The reasons for this speedup are two-fold. First, in each iteration Greedy ESP calculates the value of the objective that would be attained after adding \emph{each individual candidate edge} to the current solution set; this requires a total of $C$ objective evaluations (where $C$ is the size of the current candidate set), each of which entails a large-scale eigenvalue computation (which is the single most expensive computation in both \MAC{} and Greedy ESP, even when allowing for incremental matrix factorization of $\Laplacian$).  Consequently, Greedy ESP requires a total of $\mathcal{O}(KM)$ evaluations of the objective over the algorithm's execution.  On the other hand, \MAC{} only requires \emph{one} eigenvalue computation per iteration, and typically converges to high-quality solutions to the relaxation in $T \ll K$ iterations, therefore \emph{significantly} reducing the total number of eigenvalue computations required to obtain a solution. Second, \MAC{} also benefits from its use of the E-optimality objective as opposed to the D-optimality objective in Greedy ESP, since computing a single eigenvalue is typically \emph{much} faster than computing the Laplacian
determinant.\footnote{We also observe that since the E-optimality criterion reports the \emph{smallest} eigenvalue of the information matrix, while the D-optimality criterion used in \cite{khosoussi2019reliable} is the product of \emph{all} of the eigenvalues, maximizing the former can be interpreted as maximizing an efficiently-computable \emph{lower bound} on the latter.  This may (at least partially) account for the similar performance of the sparsified graphs returned by each of these methods.}

In recent work, \citet{nam2023spectral} presented an extension of \MAC{} which
does not require specifying an explicit budget constraint, but instead
regularizes the number of edges based on the largest eigenvalue of the adjacency
matrix. They similarly evaluate their approach in the setting of pose-graph
SLAM. Since their work is closely related to \MAC{}, it inherits many of the
computational benefits afforded by \MAC{}. The key insight to their approach is
that the largest eigenvalue of the adjacency matrix is a \emph{convex} and
increasing function of the edges. Therefore, its negation is a concave and
\emph{decreasing} function of the edges, making it straightforward to add this
term in our relaxation to implicitly select for subgraphs with fewer edges. They
also show that subgradients of the largest eigenvalue of the adjacency matrix
admit an elegant and simple construction that is almost identical to that of
Fiedler value subgradients. The difficulty in applying their approach lies in
the interpretation of the regularization parameter which trades off between
well-connected graphs and sparse graphs. Proper normalization of the competing
objectives is generally required, which in their work requires computing
eigenvalues of the Laplacian and adjacency matrix for the ``full'' graph
(containing all of the candidate edges). Since these matrices may be dense, this
computation can be expensive. We instead focus on formulations with explicit
budget constraints because they have a clear interpretation and they allow us to
avoid the expense of forming (or performing computations with) the Laplacian for
the full (potentially dense) graph. Nonetheless, in Section \ref{sec:conclusion}
we discuss potential extensions of our work in a similar spirit to the work of
\citet{nam2023spectral}, some of the computational considerations, and potential
paths forward.

Several methods have been proposed to reduce the \emph{number of states} which
need to be estimated in a SLAM problem (e.g. \cite{carlevarisbianco13iros,
  carlevaris2014conservative, johannsson12rssw, huang13ecmrb}), typically by
marginalizing out state variables. This procedure is usually followed by an edge
pruning operation to mitigate the unwanted increase in graph density. Previously
considered approaches rely on linearization of measurement models at a
particular state estimate in order to compute approximate marginals and perform
subsequent pruning. Consequently, little can be said concretely about the
quality of the statistical estimates obtained from the sparsified graph compared
to the original graph. In contrast, our approach does not require linearization,
and provides explicit performance guarantees on the graph algebraic connectivity
as compared to the globally optimal algebraic connectivity (which is itself
linked to both the \emph{best} and \emph{worst} case performance of estimators
applied to the SLAM problem).

\citet{tian2023spectral} applied spectral sparsifiers
\cite{spielman2011spectral} in the setting of multi-robot rotation averaging and
translation estimation. They focus on spectral sparsification as a means to
achieve communication efficiency during distributed optimization. That is to
say, their application of spectral sparsification is part of the distributed
estimation algorithm they design. In contrast, our application of \MAC{} to SLAM
relies on sparsification as a kind of preprocessing step that occurs
\emph{before} the maximum-likelihood estimation procedure, and in which edges
that are discarded are not recovered.

Many other applications in robotics exist for which the algebraic connectivity
or other spectral properties of graphs play a key role. For example,
\citet{somisetty2023optimal} consider as their application of interest the
problem of cooperative localization. \citet{Kim09oceans} consider planning
underwater inspection routes informed by the Cram\'er-Rao lower bound. OASIS
\cite{kaveti2023oasis} is an algorithm very similar to \MAC{} for determining
approximate E-optimal sensor arrangements. \citet{papalia2022prioritized} use
optimal design criteria in the development of a trajectory planning approach for
multi-robot systems equipped with range sensors. These spectral properties of
graphs appear quite commonly in problems of multi-agent formation control (see,
e.g., \cite{mesbahi_egerstedt2010}). Motivated by this rich problem space, we
suspect that the \MAC{} algorithm or some of the insights from its application
to SLAM may be relevant in new application areas.

\section{Background}
\label{sec:background}

This section introduces notation and relevant background for understanding both
the MAC algorithm itself and its application to graph-based SLAM problems.
Sections \ref{sec:lin-alg} and \ref{sec:graph-theory} give a brief overview of
the notation and concepts from linear algebra and graph theory that will be
important for our algorithm, most critically the Laplacian of a weighted graph
and its spectrum. Section \ref{sec:cvx-opt} introduces some preliminaries in
convex optimization that will be useful for the exposition of our algorithm in
Section \ref{sec:methods}. Section \ref{sec:slam} gives an overview of the
problems of pose-graph SLAM and rotation averaging, which are the robotics
applications of interest in this paper.

This section is intended to provide
background helpful for understanding \emph{why} the algebraic connectivity is a
natural choice of objective function for ``good'' graphs in SLAM problems, as
well as our experimental results in Section \ref{sec:experimental-results}, but
is not necessary for understanding the \MAC{} algorithm itself.

\subsection{Linear algebra}\label{sec:lin-alg}

For a real, symmetric, $n \times n$ matrix $S$, $\lambda_1(S), \ldots, \lambda_n(S)$
denote the (necessarily real) eigenvalues of $S$ in increasing order. $\PSD^{n}$ is
the set of $n \times n$ symmetric positive-semidefinite matrices. $\otimes$ denotes the
\emph{Kronecker product}. $I_n \in \R^{n \times n}$ is the identity matrix, and
$\mathds{1}_n \in \R^n$ is the all-ones vector; we may drop the $n$ subscript
when the dimension is clear from context. For a matrix $A$, $\ker(A)$ denotes
the \emph{kernel} (nullspace) of $A$. Given a set $V$ of vectors, the
\emph{span} $\Span(V)$ of $V$ is the set of all linear combinations of $v_i \in
V$. If $V$ contains a single vector $v$, then $\Span(v) = \{\alpha v \mid \alpha
\in \R\}$ is simply the set containing all scalar multiples of $v$.

\subsection{Graph theory}\label{sec:graph-theory}

A \emph{undirected graph} is a pair $\Graph = (\Nodes, \Edges)$ comprised of a
finite set $\Nodes$ of elements called \emph{nodes}, a set of unordered pairs
$\edge \in \Edges \subset \Nodes \times \Nodes$ called \emph{edges}. Similarly,
a \emph{directed graph} is a pair $\directed{\Graph} = (\Nodes, \dEdges)$ of
nodes and ordered pairs $\dedge \in \dEdges$ called \emph{directed edges}.

A \emph{weighted undirected graph} is a triple $\Graph = (\Nodes, \Edges,
\weight)$ comprised of a finite set of nodes $\Nodes$, a set of edges $\edge \in
\Edges \subset \Nodes \times \Nodes$, and a set of \emph{weights} $\weight_{ij}
\in \R$ in correspondence with each edge $\edge$.

The \emph{Laplacian matrix} associated with a weighted undirected graph (with $n
= |\Nodes|$) is a symmetric $n \times n$ matrix $\Laplacian \in \PSD^{n}$ with $i,j$-entries:
\begin{equation}
    \Laplacian_{ij} = \begin{cases} \sum_{e \in \incEdges(i)} \weight_e,& i = j, \\
      -\weight_{ij},& \edge \in \Edges, \\
      0,& \edge \notin \Edges. \end{cases}
    \label{eq:lap}
\end{equation}
  where $\incEdges(i)$ denotes the set of edges \emph{incident to} node $i$. The
  Laplacian of a graph has several well-known properties that we will use here.

  The Laplacian of a graph, denoted $\Laplacian(\Graph)$ (or simply $\Laplacian$
  when the corresponding graph is clear from context) can be written as a sum of
  the Laplacians of the subgraphs induced by each of its individual edges. A Laplacian is
  always positive-semidefinite, and the ``all ones'' vector $\mathds{1}$ of
  length $n$ is always in its kernel. We will write the eigenvalues of a
  Laplacian as $\lambda_1(\Laplacian) = 0 \leq \lambda_2(\Laplacian) \leq \dots
  \leq \lambda_n(\Laplacian)$.

  The second smallest eigenvalue of the Laplacian, $\lambda_2(\Laplacian)$, is
  the \emph{algebraic connectivity} or \emph{Fiedler value}. An eigenvector
  attaining this value is called a \emph{Fiedler vector}. The Fiedler value
  $\lambda_2(\Laplacian)$ may not be a \emph{simple} eigenvalue; there may be
  multiple eigenvectors with corresponding eigenvalues equal to
  $\lambda_2(\Laplacian)$. The algebraic connectivity is a non-decreasing
  function of the edges of $\Graph$ \cite{fiedler1973algebraic}: for two graphs
  $\Graph$ and $H$ with edge sets $\Edges_{\Graph} \subseteq \Edges_{H}$, we
  have:
  \begin{equation}\label{eq:alg-conn-monotonicity}
    \lambda_2(\Laplacian(\Graph)) \leq \lambda_2(\Laplacian(H)).
  \end{equation}
  A graph has positive algebraic connectivity $\lambda_2(\Laplacian) > 0$ if and
  only if it is connected, and more generally the number of connected components
  of a graph is equal to the number of zero eigenvalues of its Laplacian.

\subsection{Convex optimization}\label{sec:cvx-opt}

A function $f: \R^n \rightarrow \R$ is \emph{convex} if and only if
for all $x,y \in \R^n$ and $\theta \in [0,1]$:
\begin{equation}
  f(\theta x + (1 - \theta) y) \leq \theta f(x) + (1- \theta) f(y).
\end{equation}
The function $f$ is said to be \emph{concave} if $-f$ is convex \cite[Ch.
3]{boyd2004convex}. If $f$ is differentiable everywhere, then $f$ is convex if
and only if
\begin{equation}\label{eq:cvx-first-order-defn}
  f(y) \geq f(x) + \nabla f(x)\transpose (y - x),
\end{equation}
for all $x, y \in \R^n$. The geometric interpretation of eq.
\eqref{eq:cvx-first-order-defn} is that the hyperplane tangent to the graph of $f$ at $x$
must lie \emph{below} the graph of $f$ itself. The inequality \eqref{eq:cvx-first-order-defn}
holds in the opposite direction for \emph{concave} functions, and likewise the
geometric interpretation is that the hyperplane tangent to the graph of a concave function at
any point must lie \emph{above} the function's graph at all points.

If $f$ is convex but is \emph{not} differentiable at a point $x$, its
gradient is not defined, and therefore \eqref{eq:cvx-first-order-defn} does not
hold. Rather, there will be a \emph{set} of vectors $g \in \R^n$
that satisfy:
\begin{equation}\label{eq:subgradient}
  f(y) \geq f(x) + g\transpose (y - x).
\end{equation}
We call each vector $g$ for which \eqref{eq:subgradient} holds a
\emph{subgradient} of $f$ at $x$, and the set of all such vectors, denoted
$\partial f(x)$, the \emph{subdifferential} of $f$ at $x$ (cf.\ e.g.\
\cite[Sec. 23]{rockafellar2015convex}). Geometrically, each subgradient $g$ determines a \emph{supporting} hyperplane to the graph of $f$ at $x$.\footnote{Note that in contrast to the differentiable case, in which there is a \emph{unique} tangent hyperplane (determined by the gradient), a convex but nonsmooth function can have \emph{several} supporting hyperplanes at a given point $x$.  Indeed, a convex (but not necessarily smooth) function $f$ is differentiable at $x$ if and only if the subdifferential $\partial f(x) = \lbrace g \rbrace$ is a \emph{singleton}, in which case $\nabla f(x) = g$.} The equivalent expressions for concave functions are sometimes
called \emph{supergradients} and \emph{superdifferentials}, and similarly admit
an interpretation as a set of approximating hyperplanes that lie \emph{above} the graph of $f$. Since this work is concerned with optimization of
\emph{concave} functions, we will follow the common convention and slightly abuse notation by using ``$\partial f(x)$" to
refer to the \emph{superdifferential} of $f$ at $x$.

\subsection{Pose-graph SLAM and rotation averaging}\label{sec:slam}

Our primary application of interest in this paper is pose-graph SLAM. Pose-graph
SLAM is the problem of estimating $n$ unknown translations $t_i, \ldots, t_n \in
\R^d$ and rotations $\rot_i, \ldots, \rot_n \in \SO(d)$ given a subset of
measurements of their pairwise relative transforms $(\ntran_{ij}, \nrot_{ij})$.
This problem admits a natural graphical representation $\Graph \triangleq
(\Nodes, \dEdges)$ where the nodes $\Nodes$ correspond to latent (unknown) poses
$(\tran_i, \rot_i)$ and the edges $(i,j) \in \dEdges$ correspond to the noisy measurements. For each edge $\dedge \in \dEdges$, we assume that the corresponding measurement $(\ntran_{ij}, \nrot_{ij})$ is sampled according to:
\begin{subequations}
  \begin{equation}
    \nrot_{ij} = \trot_{ij}\rot_{ij}^\epsilon, \quad \rot_{ij}^{\epsilon} \sim \Langevin(I_d, \kappa_{ij})
  \end{equation}%
  \begin{equation}
    \ntran_{ij} = \ttran_{ij} + \tran_{ij}^\epsilon, \quad \tran_{ij}^{\epsilon} \sim \Gaussian(0, \transym_{ij}^{-1}I_d),
  \end{equation}%
  \label{eq:gen-model-pgo}
\end{subequations}%
where $(\ttran_{ij}, \trot_{ij})$ is the true relative transform between nodes
$i$ and $j$. Under this noise model, the maximum-likelihood estimation (MLE) problem for pose-graph SLAM 
takes the following form (cf. \cite[Problem 1]{rosen2019se}):
\begin{problem}[MLE for pose-graph SLAM]
  \label{se-mle}
  \begin{equation}
    \min_{\substack{{\tran_i \in \R^d} \\ {\rot_i \in \SO(d)}}} \sum_{(i,j) \in \dEdges} \kappa_{ij} \| \rot_j - \rot_i\nrot_{ij}\|^2_F + \transym_{ij}\|t_j - t_i - R_i\ntran_{ij}\|_2^2.
  \end{equation}%
\end{problem}%
\noindent Problem \ref{se-mle} also directly captures the problem of rotation averaging
under the Langevin noise model simply by taking all $\tau_{ij} = 0$
\cite{dellaert2020shonan}.

Prior work \cite{doherty2022performance, rosen2019se} showed that the algebraic
connectivity of the \emph{rotational weight graph} $\RotW = (\Nodes, \Edges,
\kappa)$ (i.e. the graph with nodes corresponding to robot poses and edge
weights equal to each $\kappa_{ij}$ in Problem \ref{se-mle}) controls the
\emph{worst-case} error of solutions to Problem \ref{se-mle}.\footnote{More precisely,
  the eigenvalue appearing in the previous results \cite{doherty2022performance,
    rosen2019se} is a kind of \emph{generalized} algebraic connectivity,
  obtained as an eigenvalue of a matrix comprised of the (latent) \emph{ground
    truth} problem data. The Fiedler eigenvalue of the rotational weight graph
  gives a lower bound on this quantity which can be computed \emph{without
    access to the ground truth} (see, e.g., \citet[Appendix C.3]{rosen2019se}).}
In the setting of SLAM, these results motivate the use of the algebraic
connectivity as a measure of graph ``quality'' for the purposes of
sparsification. That is to say, one way to define a ``good'' approximating pose
graph is as the sparse subgraph with the \emph{largest algebraic connectivity} among
all subgraphs with the desired number of edges, which is precisely how we apply
our algorithm in Section \ref{sec:experimental-results}.

\section{The MAC Algorithm}
\label{sec:methods}

\subsection{Algebraic connectivity maximization}

It will be convenient to partition the edges as $\Edges = \EO \cup \EC,\ \EO
\cap \EC = \emptyset$ into a \emph{fixed} set of edges $\EO$ and a set of $m$
\emph{candidate} edges $\EC$, and where $\LapRotWO$ and $\LapRotWC$ are the
Laplacians of the subgraphs induced by $\EO$ and $\EC$. In our applications to
SLAM, the subgraph induced by $\EO$ on $\Nodes$ will typically be constructed
from sequential odometric measurements (therefore, $|\EO| = n - 1$), but this is
not a requirement of our general approach.\footnote{In particular, to apply our
  approach we should select $\LapRotWO$ and $K$ to guarantee that the feasible
  set for Problem \ref{prob:max-aug-alg-conn} contains at least one tree. Then,
  it is clear that the optimization in Problem \ref{prob:max-aug-alg-conn} will
  always return a connected graph, since $\lambda_2(\LapRotW(\selection)) > 0$
  if and only if the corresponding graph is connected. Note that this condition
  is always easy to arrange: for example, we can start with $\LapRotWO$
  constructed from a tree of $\Graph$, as we do here, or (even more simply) take
  $\EO$ to be the empty set and simply take $K \geq n - 1$.} It will be helpful
in the subsequent presentation to ``overload'' the definition of $\LapRotW$.
Specifically, let $\LapRotW : [0,1]^m \rightarrow \PSD^{n}$ be the affine
map constructing the total graph Laplacian from a weighted combination of edges
in $\EC$:
\begin{equation}
\label{graph_laplacian_function_eq}
  \LapRotW(\selection) \triangleq \LapRotWO + \sum_{k = 1}^m \selection_{k}\LapRotWC_{k},
\end{equation}
where $\LapRotWC_k$ is the Laplacian of the subgraph induced by the weighted
edge $e_k = \{i_k, j_k\}$ of $\EC$. Our goal in this work will be to identify a
subset of $\Eopt \subseteq \EC$ of fixed size $|\Eopt| = K$ (equivalently, a
$K$-sparse binary vector $\selection$) that \emph{maximizes} the algebraic
connectivity $\lambda_2(\LapRotW(\selection))$. This corresponds to the
following optimization problem:
\begin{problem}[Algebraic connectivity maximization]
  \label{prob:max-aug-alg-conn}
  \begin{equation}
    \begin{gathered}
      \optValAC = \max_{\selection \in \{0,1\}^m} \lambda_{2}(\LapRotW(\selection)) \label{eq:connectivity-rot-only} \\
      \sum_{k=1}^m \selection_k = K.
    \end{gathered}
  \end{equation}
\end{problem}

Problem \ref{prob:max-aug-alg-conn} is a variant of the \emph{maximum algebraic
  connectivity augmentation problem}, which is NP-Hard \cite{mosk2008maximum}.
The difficulty of Problem \ref{prob:max-aug-alg-conn} stems, in particular, from
the integrality constraint on the elements of $\selection$. Consequently, our
general approach will be to solve a simpler problem obtained by relaxing the
integrality constraints of Problem \ref{prob:max-aug-alg-conn}, and, if
necessary, \emph{rounding} the solution to the relaxed problem to a solution in
the feasible set of Problem \ref{prob:max-aug-alg-conn}. In particular, we
consider the following \emph{Boolean relaxation} of Problem
\ref{prob:max-aug-alg-conn}:

\begin{problem}[Boolean relaxation of Problem \ref{prob:max-aug-alg-conn}]
  \begin{equation}\label{eq:relaxation}
    \begin{gathered}
      \max_{\selection \in [0,1]^m} \lambda_2(\LapRotW(\selection)) \\
      \mathds{1}\transpose \selection = K.
    \end{gathered}
  \end{equation}
  \label{prob:relaxation}
\end{problem}

Relaxing the integrality constraints of Problem \ref{prob:max-aug-alg-conn}
dramatically alters the difficulty of the problem. In particular, we know (see,
e.g. \cite{ghosh2006growing}):
\begin{lem}\label{lem:concavity}
  The function $\objectiveF(\selection) = \lambda_2(\LapRotW(\selection))$ is
  \emph{concave} on the set $\selection \in [0,1]^m,
  \mathds{1}\transpose\selection = K$.
\end{lem}
Consequently, solving Problem \ref{prob:relaxation} amounts to maximizing
a concave function over a convex set; this is in fact a convex optimization
problem (one can see this by simply considering minimization of the objective
$-\objectiveF(\selection)$) and hence efficiently \emph{globally solvable} (cf.\ 
e.g.\ \cite{boyd2004convex, bertsekas2016nonlinear}). However, since a solution to Problem
\ref{prob:relaxation} need not be feasible for the original problem (Problem \ref{prob:max-aug-alg-conn}), we must subsequently
\emph{round} solutions of the relaxation \eqref{eq:relaxation} by projecting them onto
the feasible set of Problem \ref{prob:max-aug-alg-conn}.

\subsection{Solving the relaxation}

\begin{small}
  \begin{algorithm}[t]
    \caption{Frank-Wolfe Method\label{alg:frank-wolfe-general}}
    \begin{algorithmic}[1]
      \Input An initial iterate $x^{(0)} \in \mathcal{F}$ in a compact convex feasible
      set $\mathcal{F}$ \newline
      A concave function $f: \mathcal{F} \rightarrow \R$ \newline
      A (super)gradient function $\supergradient: x \mapsto v,\ v \in
      \partial f(x)$
      \Output An approximate solution to $\max_{x \in \mathcal{F}} f(x)$.
      \Function{FrankWolfe}{$x^{(0)}$, $f$, $g$}
      \For {$t = 0, \dotsc, T-1$}
      \State $s^{(t)} \leftarrow \argmax_{s \in \mathcal{F}} s\transpose \supergradient(x^{(t)}) $
      \State $\alpha \leftarrow 2 / (2 + t)$
      \Comment{Compute step size}
      \State $x^{(t+1)} \leftarrow x^{(t)}+ \alpha \left( s^{(t)} - x^{(t)}\right)$
      \EndFor
      \State \Return $x^{(T)}$
      \EndFunction
    \end{algorithmic}
  \end{algorithm}
\end{small}

\begin{small}
  \begin{algorithm}[t]
    \caption{MAC Algorithm \label{alg:mac}}
    \begin{algorithmic}[1]
      \Input An initial iterate $\selection^{(0)} \in [0,1]^m,
      \mathds{1}\transpose \selection^{(0)} = K$
      \Output An approximate solution to Problem \ref{prob:max-aug-alg-conn}
      \Function{MAC}{$\selection^{(0)}$}
      \State Define $\objectiveF: \selection \mapsto \lambda_2(\Laplacian(\selection))$
      \State Let $q_2(\selection)$ be any normalized eigenvector of
      $\Laplacian(\selection)$ with corresponding eigenvalue $\lambda_2(\Laplacian(\selection))$
      \State Define $\supergradient(x)$ with $\supergradient_k(\selection) =
      q_2(\selection)\transpose\LapRotWC_kq_2(\selection)$\Comment{Eq. \eqref{eq:supergradient}}
      \State $\selection^{(T)} \leftarrow $\textsc{FrankWolfe}$(\selection^{(0)},
      \objectiveF, \supergradient)$
      \Comment{Solve Problem \ref{prob:relaxation}}
      \State \Return $\rounded{\selection^{(T)}}$
      \Comment{Round solution (Sec. \ref{sec:rounding})}
      \EndFunction
    \end{algorithmic}
  \end{algorithm}
\end{small}

There are several methods which could, in principle, be used to solve the
relaxation in Problem \ref{prob:relaxation} (see Sec. \ref{sec:related-work}).
For example, Ghosh and Boyd \cite{ghosh2006growing} consider solving an
equivalent semidefinite program. This approach has the advantage of fast
convergence (in terms of the number of iterations required to compute an optimal
solution), but can nonetheless be slow for the large problem instances ($m >
1000$) typically encountered in the SLAM setting. Instead, our algorithm for
\emph{maximizing algebraic connectivity} (\MAC{}), summarized in Algorithm
\ref{alg:mac}, employs an inexpensive subgradient (more precisely,
\emph{supergradient}) approach to solve the relaxation in Problem
\ref{prob:relaxation}, then rounds its solution to an element of the original
feasible set in Problem \ref{prob:max-aug-alg-conn} (see Section \ref{sec:rounding}).

In particular, \MAC{} uses the \emph{Frank-Wolfe method} (also known as the
conditional gradient method), a classical approach for solving convex
optimization problems of the form in Problem \ref{prob:relaxation} (see, e.g.,
\cite[Sec. 2.2]{bertsekas2016nonlinear}). We use a variant of the Frank-Wolfe
method for maximization of concave functions which are not uniquely
differentiable everywhere, simply replacing gradients with supergradients. The
Frank-Wolfe method adapted to this setting is summarized in Algorithm
\ref{alg:frank-wolfe-general}. In words, at each iteration, the Frank-Wolfe
method requires (1) linearizing the objective $f$ at a particular point
$\selection$, (2) maximizing the linearized objective over the (convex) feasible
set, and (3) taking a step in the direction of the solution to the linearized
problem.

The Frank-Wolfe method is particularly advantageous in this setting since the
feasible set for Problem \ref{prob:relaxation} is the intersection of the
hypercube with the linear subspace determined by $\mathds{1}\transpose
\selection = K$ (a linear equality constraint). Consequently this problem
amounts to solving a linear program, which can be done easily (and in fact, as
we will show, admits a simple \emph{closed-form} solution). Formally, the
\emph{direction-finding subproblem} is defined as follows:
\begin{problem}[Direction-finding subproblem]\label{prob:dir-subproblem}
  Fix an iterate $\selection \in [0,1]^m,\ \mathds{1}\transpose \selection = K$
  and let $\supergradient \in \R^m$ be any supergradient of
  $\lambda_2(\Laplacian(\selection))$ at $\selection$. The direction-finding
  subproblem is the following linear program:
  \begin{equation}
    \begin{gathered}
      \max_{s \in [0,1]^m} \supergradient\transpose s, \\
      \mathds{1}\transpose  s = K.
    \end{gathered}
  \end{equation}
\end{problem}
In order to form the linearized objective in Problem \ref{prob:dir-subproblem}
we require a supergradient of the objective function. It turns out, we can
always recover a supergradient of $\lambda_2(\Laplacian(\selection))$ in terms
of a Fiedler vector of $\Laplacian(\selection)$. Specifically, we have the
following theorem (which we prove in the Appendix \ref{app:gradients}):

\begin{thm}[Supergradients of $\lambda_2(\LapRotW(\selection))$]\label{thm:supergradient}
  Let $q_2 \perp \mathds{1}$ be a normalized eigenvector of $\lambda_2(\LapRotW(\selection))$. Then
  the vector $\supergradient \in \R^m$ whose $k$th element is defined by:
\begin{equation}\label{eq:supergradient}
    \supergradient_k = q_2\transpose\LapRotWC_kq_2,
\end{equation}
for all $1 \le k \le m$ is a \emph{supergradient} of $\objectiveF$ at $\selection$. Equivalently:
\begin{equation}\label{eq:supergradient-alt}
  \supergradient = (I_m \otimes q_2) \transpose \begin{pmatrix} \LapRotWC_1 \\ \vdots \\ \LapRotWC_m \end{pmatrix} q_2.
\end{equation}
\end{thm}
Therefore, supergradient computation can be performed by simply recovering an
eigenvector of $\LapRotW(\selection)$ corresponding to
$\lambda_2(\LapRotW(\selection))$.

Problem \ref{prob:dir-subproblem} is a linear program, for which several efficient numerical
solution techniques exist \cite{bertsekas2016nonlinear}. However, in our case,
Problem \ref{prob:dir-subproblem} turns out to admit a simple, \emph{closed-form} solution
$s^*$ (which we prove in Appendix \ref{app:dir-subproblem}):
\begin{thm}[A closed-form solution to Problem \ref{prob:dir-subproblem}]\label{thm:dir-subproblem}
  Let $\mathcal{S}^*,\ |\mathcal{S}^*| = K$ be the set containing the indices of
  the $K$ \emph{largest} elements of $\supergradient(\selection)$, breaking
  ties arbitrarily where necessary. The vector $s^* \in \R^n$ with element $k$
  given by:
  \begin{equation}\label{eq:dir-subproblem-opt}
    s^*_k = \begin{cases} 1, &  k \in \mathcal{S}^*, \\
      0, &  \text{otherwise,} \end{cases}
  \end{equation}
  is a maximizer for Problem \ref{prob:dir-subproblem}.
\end{thm}
In this work, we use a simple decaying step size $\alpha$ to update $\selection$
in each iteration. While in principle, we could instead use a line search method
\cite[Sec. 2.2]{bertsekas2016nonlinear}), this would potentially require many
evaluations of $\objectiveF(\selection)$ within each iteration. Since every
evaluation of $\objectiveF(\selection)$ requires an eigenvalue computation, this
can become a computational burden for large problems.

In general, the Frank-Wolfe algorithm offers \emph{sublinear} (i.e.
$\mathcal{O}(1/T)$ after $T$ iterations) convergence to the globally optimal
solution in the worst case \cite{dunn1978conditional}. However, in this context
it has several advantages over alternative approaches. First, we can bound the
sparsity of a solution after $T$ iterations. In particular, we know that the
solution after $T$ iterations has \emph{at most} $KT$ nonzero entries. Second,
the gradient computation requires only a single computation of the minimal $2$
dimensional eigenspace of an $n \times n$ matrix. This can be performed quickly
using a variety of methods (e.g. the preconditioned Lanczos method or more
specialized procedures, like the TRACEMIN-Fiedler algorithm
\cite{manguoglu2010tracemin,sameh1982trace}, which is what we use in our
implementation). Finally, as we showed, the direction-finding subproblem in
Problem \ref{prob:dir-subproblem} admits a simple \emph{closed-form solution}
(as opposed to a projected gradient method which requires projection onto an
$\ell_1$-ball). In consequence, despite the fact that gradient-based methods may
require many iterations to converge to \emph{globally optimal} solutions,
high-quality approximate solutions can be computed quickly at the scale necessary
for SLAM problems. As we will show in Section \ref{sec:guarantees}, our approach
admits \emph{post hoc} suboptimality guarantees even in the event that we
terminate optimization prematurely (e.g. when a \emph{fast} but potentially
suboptimal solution is required). Critically, these suboptimality guarantees
ensure the quality of the solutions of our approach not only with respect to the
relaxation, but also with respect to the \emph{original problem}.

\subsection{Rounding the solution}\label{sec:rounding}

In the event that the optimal solution to the relaxed problem is integral, we
ensure that we have \emph{also} obtained an optimal solution to the original
problem. However, this need not be the case in general. In the (typical) case
where integrality does not hold, we \emph{project} the solution to the relaxed
problem onto the original constraint set. Any mapping from the feasible
set of the Boolean relaxation to the original feasible set can be used for
rounding. Formally, let:
\begin{equation}\label{eq:generic-rounding}
  \projsym: \left\{ x \in [0,1]^m \mid \mathds{1}\transpose x = K \right\} \rightarrow \left\{ \hat{x} \in \{0,1\}^m \mid \mathds{1}\transpose \hat{x} = K \right\}
\end{equation}
be our generic rounding procedure. In this paper, we consider two possible
instantiations of $\projsym$: a deterministic ``nearest neighbor'' rounding
procedure and the randomized selection algorithm of \citet{madow1949theory}.

\textbf{Nearest-neighbor rounding:} The simplest approach to obtain an
integral solution from a solution $\selection$ to the relaxed problem is to simply round the largest $K$ components of
$s$ to 1, and set all other components to zero:
\begin{equation}\label{eq:rounding}
  \nearestRounded{\selection}_k \triangleq \begin{cases} 1, & \text{if $\selection_k$ is in the largest $K$ elements of $\selection$}, \\
  0, & \text{otherwise}. \end{cases}
\end{equation}
It is straightforward to verify that \eqref{eq:rounding} gives the \emph{nearest
  binary vector} (e.g. in the $\ell_1$ and $\ell_2$ sense) to the relaxed
solution $\selection$ subject to the budget constraint.

\textbf{Randomized rounding:} An alternative approach is to sample a
\emph{random} element of the feasible set, informed by the value of the
relaxed solution. For example, \citet{khosoussi2019reliable} consider constructing an estimate $\hat{\selection} \in \lbrace 0, 1 \rbrace^m$ by sampling it from the distribution $p_{\selection}$ over $m$-dimensional binary vectors defined by:
\begin{equation}
\label{iid_Bernoulli_sampling}
\begin{gathered}
p_{\selection} \colon \lbrace 0,1 \rbrace^m \to [0,1] \\
p_{\selection}(b) = \prod_{k=1}^m \selection_k^{b_k} (1 - \selection_k)^{1 - b_k}.
\end{gathered}
\end{equation}
In words, this is the distribution that \emph{independently} samples each element $b_k \in \lbrace 0, 1 \rbrace$  of $b$ according to $b_k \sim \mathrm{Ber}(\selection_k)$, where $\mathrm{Ber}(\selection_k)$ is the Bernoulli distribution parameterized by the $k$th element $\selection_k \in [0,1]$ of the relaxed solution $\selection$.

While elegant in its simplicity, the consequence of employing the elementwise independent sampling model \eqref{iid_Bernoulli_sampling}
is that we may generate an estimate $\hat{\selection}$ that does
not satisfy the budget constraints of Problem \ref{prob:max-aug-alg-conn}. If we sample fewer than $K$
edges, then the monotonicity of $\lambda_2$ implies we could have obtained a
better solution simply by taking more edges (cf.\ eq.\
\eqref{eq:alg-conn-monotonicity}). On the other hand, we may also sample more
than $K$ edges, violating the budget constraints of the original problem.

\begin{small}
  \begin{algorithm}[t]
    \caption{Madow's Sampling Procedure \label{alg:madow}}
    \begin{algorithmic}[1]
      \Input $\selection \in [0,1]^m, \sum_k \selection_k = K$
      \Output A random sample $\hat{\selection} \sim p_x(\cdot \mid \mathds{1}_m\transpose b = K)$
      \Function{RoundMadow}{$\selection$}
      \State Define $\phi_0 \leftarrow 0$, $\phi_k \leftarrow \phi_{k-1} + \selection_k,\ k = 1, \ldots, m$
      \State Sample $U$ uniformly from the interval $[0,1]$
       \For {$i = 0, \dotsc, K-1$}
      \State Set $\hat{\selection}_k = 1$ if $\phi_{k-1} \leq U + i < \phi_k $
      \EndFor
      \State \Return $\hat{\selection}$
      \EndFunction
    \end{algorithmic}
  \end{algorithm}
\end{small}

To address this, we employ the systematic sampling scheme of
\citet{madow1949theory}, summarized in Algorithm \ref{alg:madow}.  In brief, this procedure generates samples from the \emph{conditional} distribution $p_x(b \mid \mathds{1}_m\transpose b = K)$ obtained from $p_x$ by requiring the generated samples $b$ to have \emph{exactly} $K$ nonzero elements.
Madow's sampling procedure has been used in other work, e.g.
\cite{paria2021leadcache}, in similar randomized rounding applications. In our
experimental results (Sec. \ref{sec:experimental-results}) we demonstrate that
this inexpensive rounding procedure gives a fairly dramatic improvement in
solution quality on graphs obtained from SLAM datasets over the ``nearest
neighbors'' approach. 

It is also important to note that there is no restriction on running Madow's procedure \emph{multiple} times. In our applications, we
run the procedure once for efficiency reasons (each computation of the objective
value amounts to another eigenvalue computation, which is as expensive as an
additional iteration of Frank-Wolfe). However, it would be perfectly reasonable
to run Madow's procedure multiple times, computing the value of the objective
for each sample, and returning only the best.

\subsection{Post-hoc suboptimality guarantees}\label{sec:guarantees}

Algorithm \ref{alg:mac} admits several \emph{post hoc} suboptimality
guarantees. Let $\optValAC$ be the optimal value of the original nonconvex
maximization in Problem \ref{prob:max-aug-alg-conn}. Since Problem
\ref{prob:relaxation} is a relaxation of Problem \ref{prob:max-aug-alg-conn}, in
the event that optimality attains for a vector $\selection^*$, we know:
\begin{equation}
  \objectiveF(\rounded{\selection^*}) \leq \optValAC \leq \objectiveF(\selection^*),
\end{equation}
where we use the generic rounding operation $\rounded{\cdot}$ since the result
holds independent of the choice of rounding scheme used. In turn, the
suboptimality of a rounded solution $\rounded{\selection^*}$ is bounded as
follows:
\begin{equation}
  \optValAC - \objectiveF(\rounded{\selection^*}) \leq \objectiveF(\selection^*) - \objectiveF(\rounded{\selection^*}).
\end{equation}
Consequently, in the event that $\objectiveF(\selection^*) -
\objectiveF(\rounded{\selection^*}) = 0$, we know that $\rounded{\selection^*}$
\emph{must } correspond to an optimal solution to Problem
\ref{prob:max-aug-alg-conn}.

The above guarantees apply in the event that we obtain a \emph{maximizer}
$\selection^*$ of Problem \ref{prob:relaxation}. This would seem to pose an
issue if we aim to terminate optimization before we obtain a verifiable,
globally optimal solution to Problem \ref{prob:relaxation} (e.g. in the presence
of real-time constraints). Since these solutions are not necessarily globally
optimal in the relaxation, we do not know if their objective value is larger or
smaller than the optimal solution to Problem \ref{prob:max-aug-alg-conn}.
However, we can in fact obtain per-instance suboptimality guarantees of the same
kind for \emph{any} feasible point $\selection \in [0,1]^m, \mathds{1}\transpose
\selection = K$ through the \emph{dual} of our relaxation (cf. \citet[Appendix
D]{lacoste2013block}). Here, we give a derivation of the dual upper bound which
uses only the concavity of $\objectiveF$.

Recall that since $\objectiveF$ is concave, for any $x, y \in [0,1]^m,\
\mathds{1}\transpose x = \mathds{1}\transpose y = K$ we have:
\begin{equation}
  \objectiveF(y) \leq \objectiveF(x) + \supergradient \transpose (y - x),
\end{equation}
where $g \in \partial \objectiveF(x)$ is any supergradient of $\objectiveF$ at $x$.
Consider then the following upper bound on the optimal value
$\objectiveF(\selection^*)$ computed at the feasible point $\selection$:
\begin{equation}\label{eq:dual-bound-deriv}
  \begin{aligned}
    \objectiveF(\selection^*) &\leq \objectiveF(\selection) + \supergradient \transpose (\selection^* - \selection) \\
    &\leq \max_{s \in [0,1]^m, \mathds{1}\transpose s = K} \objectiveF(\selection) + \supergradient\transpose(s - \selection) \\
    &= \objectiveF(\selection) - 
    \supergradient\transpose \selection + \max_{s \in [0,1]^m, \mathds{1}\transpose
      s = K} \supergradient \transpose s.
  \end{aligned}
\end{equation}
We observe that the solution to the optimization in the last line of
\eqref{eq:dual-bound-deriv} is \emph{exactly} the solution to the
direction-finding subproblem (Problem \ref{prob:dir-subproblem}). Letting
$\est{s}$ be a vector obtained as a solution to Problem
\ref{prob:dir-subproblem} at $\selection$ (with corresponding supergradient
$\supergradient$), we obtain the following \emph{dual} upper bound:
\begin{equation}\label{eq:dual-upper-bound}
  \dualF(\selection) \triangleq \objectiveF(\selection) + \supergradient \transpose (\est{s} - \selection). \\
\end{equation}
Now, from \eqref{eq:dual-bound-deriv}, we have $\dualF(\selection) \geq
\objectiveF(\selection^*)$ for any $\selection$ in the feasible set. In turn, it
is straightforward to verify that the following chain of inequalities hold for
any point $\selection$ in the feasible set of relaxation in Problem
\ref{prob:relaxation}:\footnote{In fact, this chain of inequalities holds after
  replacing the rounded solution $\rounded{\selection}$ with \emph{any} point
  from the feasible set of the original problem (Problem
  \ref{prob:max-aug-alg-conn}). This part of the inequality is simply a
  consequence of the fact that the feasible set for the original problem is a
  \emph{subset} of the feasible set for the relaxation.}
\begin{equation}
  \objectiveF(\rounded{\selection}) \leq \optValAC \leq \dualF(\selection),
\end{equation}
with the corresponding suboptimality guarantee:
\begin{equation}
  \optValAC - \objectiveF(\rounded{\selection}) \leq \dualF(\selection) - \objectiveF(\rounded{\selection}).
\end{equation}
Moreover, we can always recover a suboptimality bound on $\selection$ with
respect to the optimal value $\objectiveF(\selection^*)$ of the relaxed problem as:
\begin{equation} \label{eq:duality-gap-as-bound}
 \objectiveF(\selection^*) - \objectiveF(\selection) \leq \dualF(\selection) - \objectiveF(\selection)
\end{equation}
The expression appearing on the right-hand side of
\eqref{eq:duality-gap-as-bound} is the (Fenchel) \emph{duality gap}. Equation
\eqref{eq:duality-gap-as-bound} also motivates the use of the duality gap as a
stopping criterion for Algorithm \ref{alg:frank-wolfe-general}: if the gap is
sufficiently close to zero (e.g. to within a certain numerical tolerance), we
may conclude that we have reached an optimal solution $\selection^*$ to Problem
\ref{prob:relaxation}.

\section{Experimental Results}
\label{sec:experimental-results}

\begin{figure*}[t]
    \centering
    \includegraphics[width=1.0\linewidth]{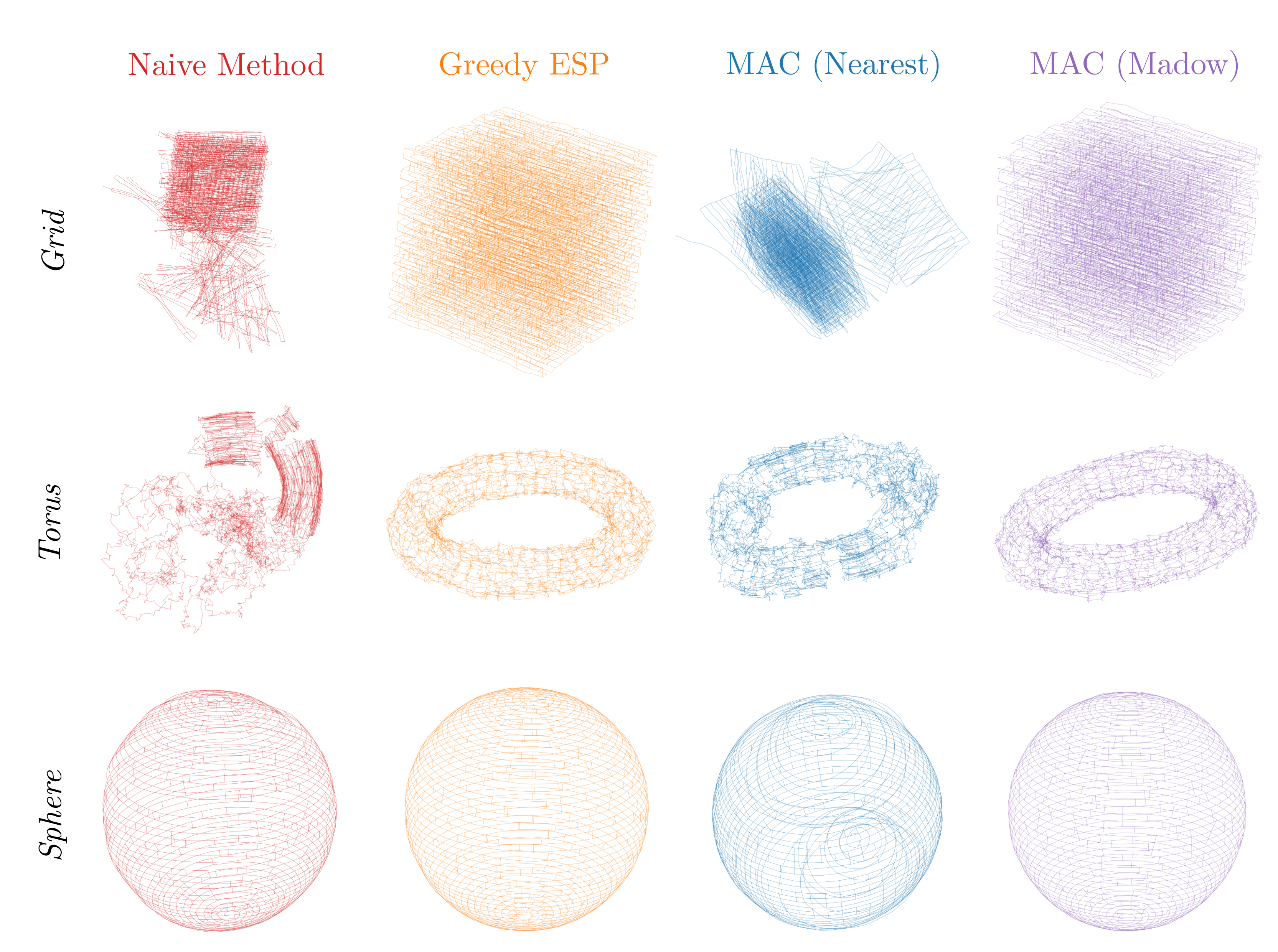}
    \caption{\textbf{Qualitative results for 3D pose-graph sparsification.}
      Pose-graph optimization results for the three 3D benchmark datasets after
      pruning all but 20\% of the original loop closures using each method:
      (Top) \Grid{}, (Middle) \Torus{}, (Bottom) \Sphere{}.}
    \label{fig:qualitative-3d}
\end{figure*}

\begin{figure*}
  \centering
    \begin{subfigure}{1.0\linewidth}
    \centering
    \includegraphics[width=0.25\linewidth]{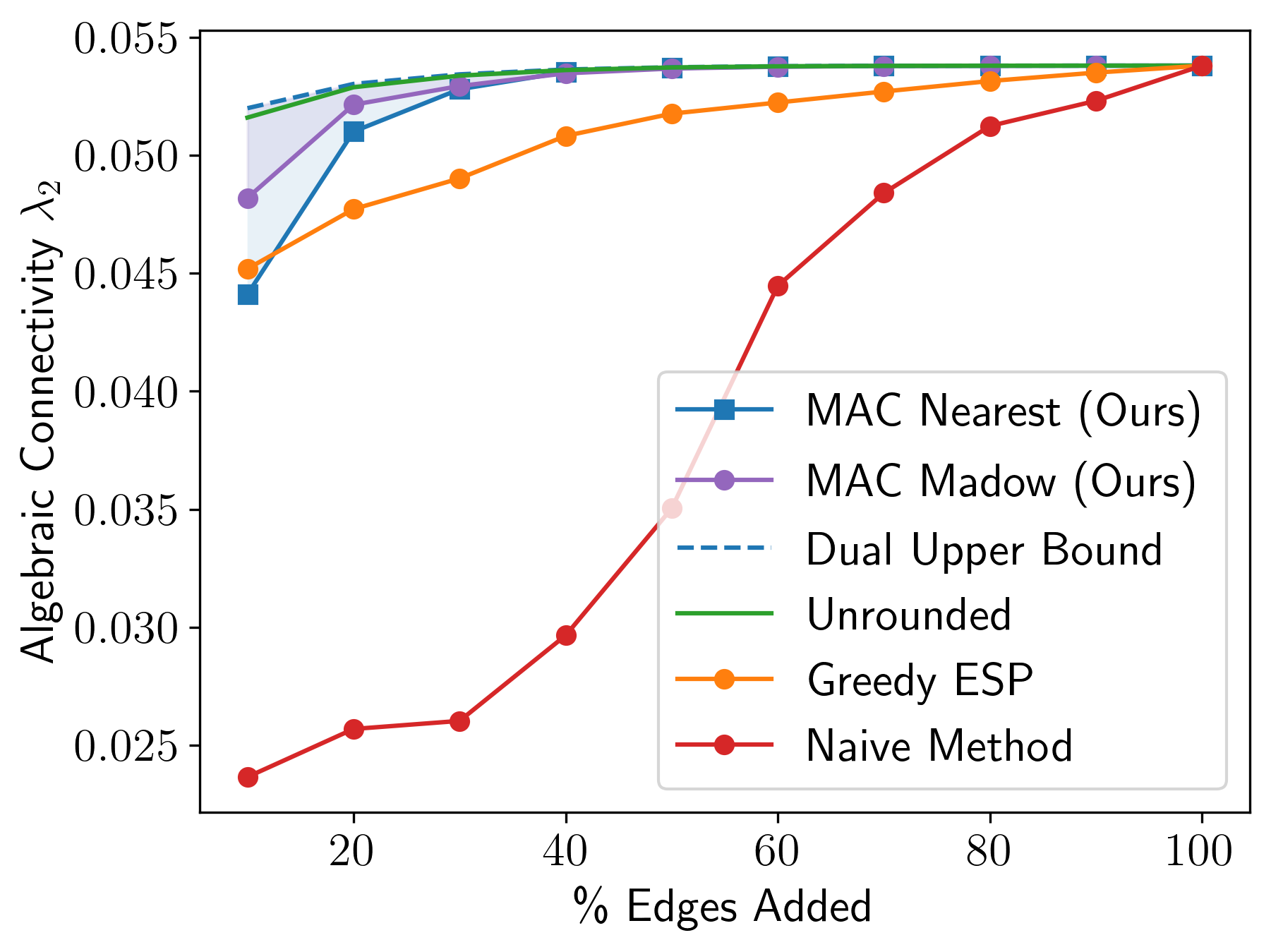}%
    \includegraphics[width=0.25\linewidth]{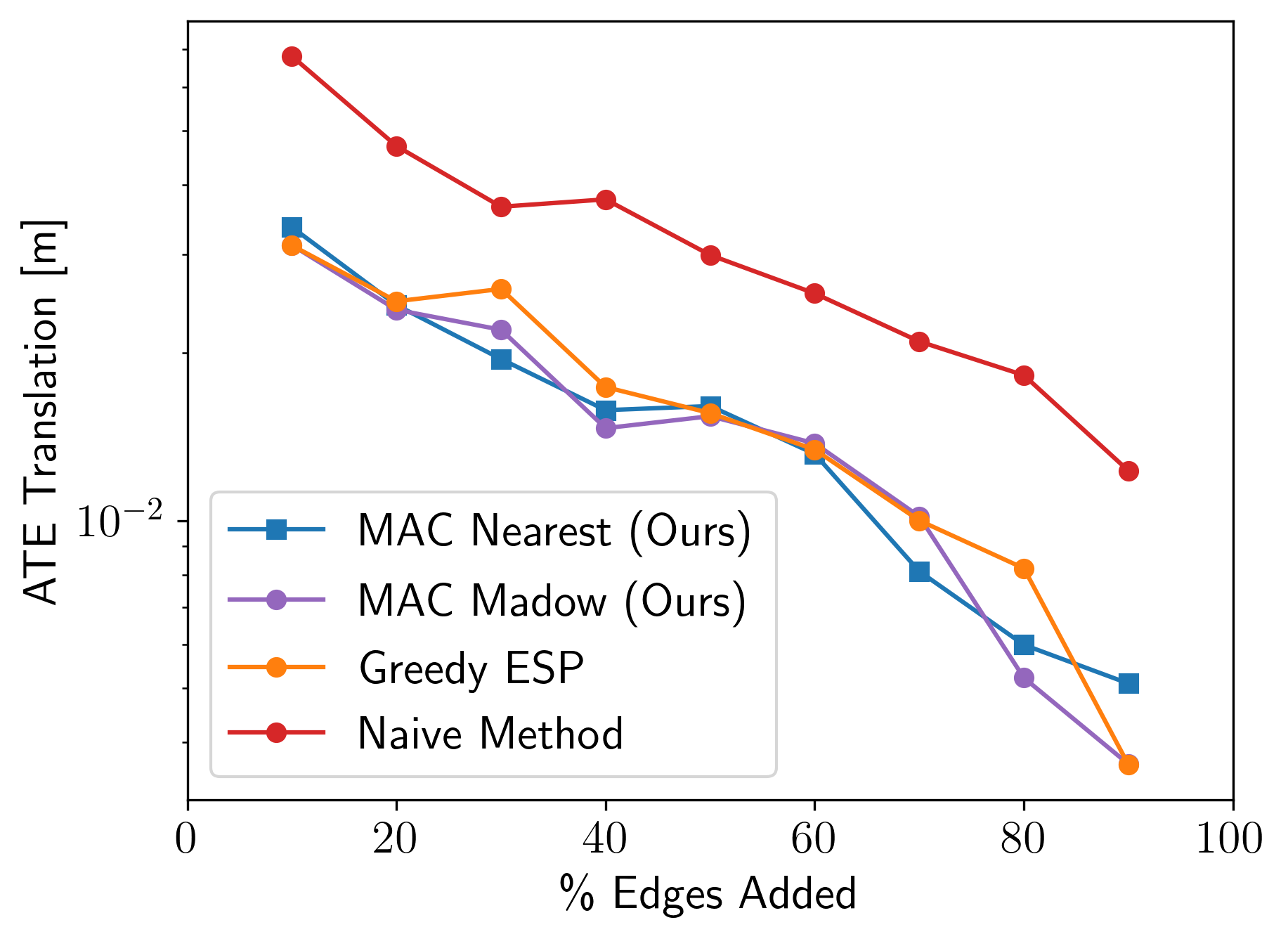}%
    \includegraphics[width=0.25\linewidth]{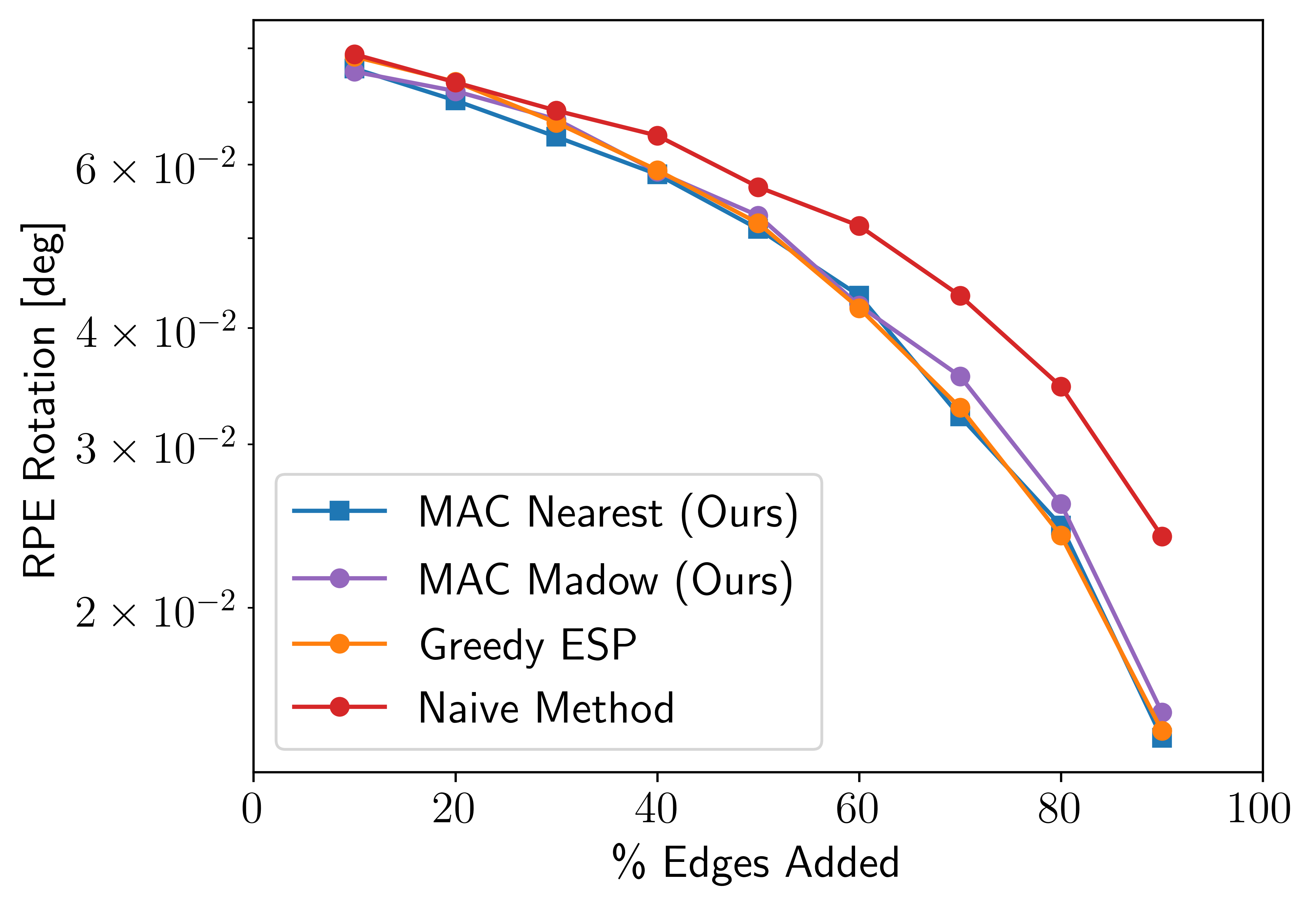}%
    \includegraphics[width=0.25\linewidth]{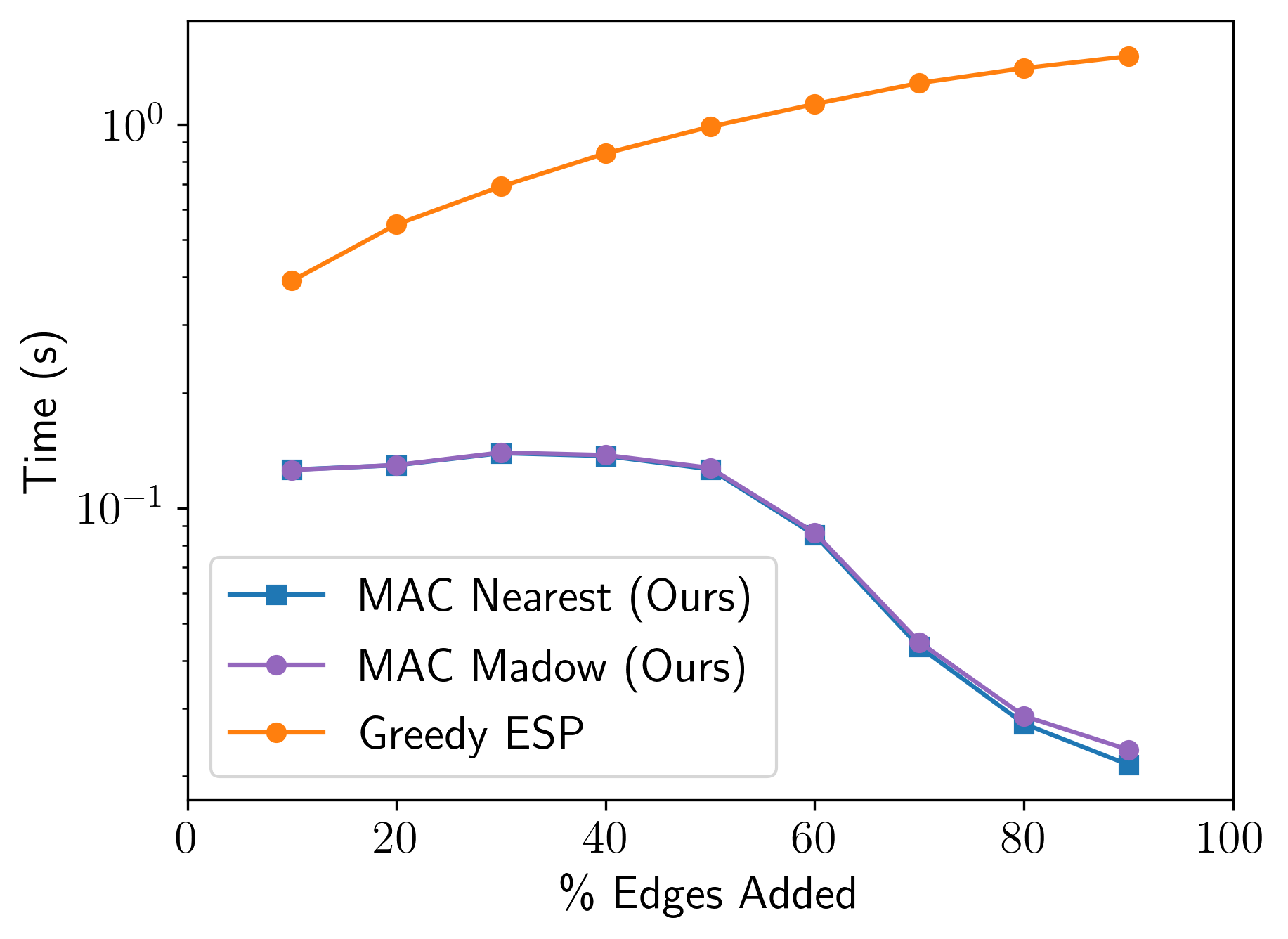}%
    \caption{\Intel{}\\ \label{fig:quantitative:intel}}
  \end{subfigure}
    \begin{subfigure}{1.0\linewidth}
    \centering
    \includegraphics[width=0.25\linewidth]{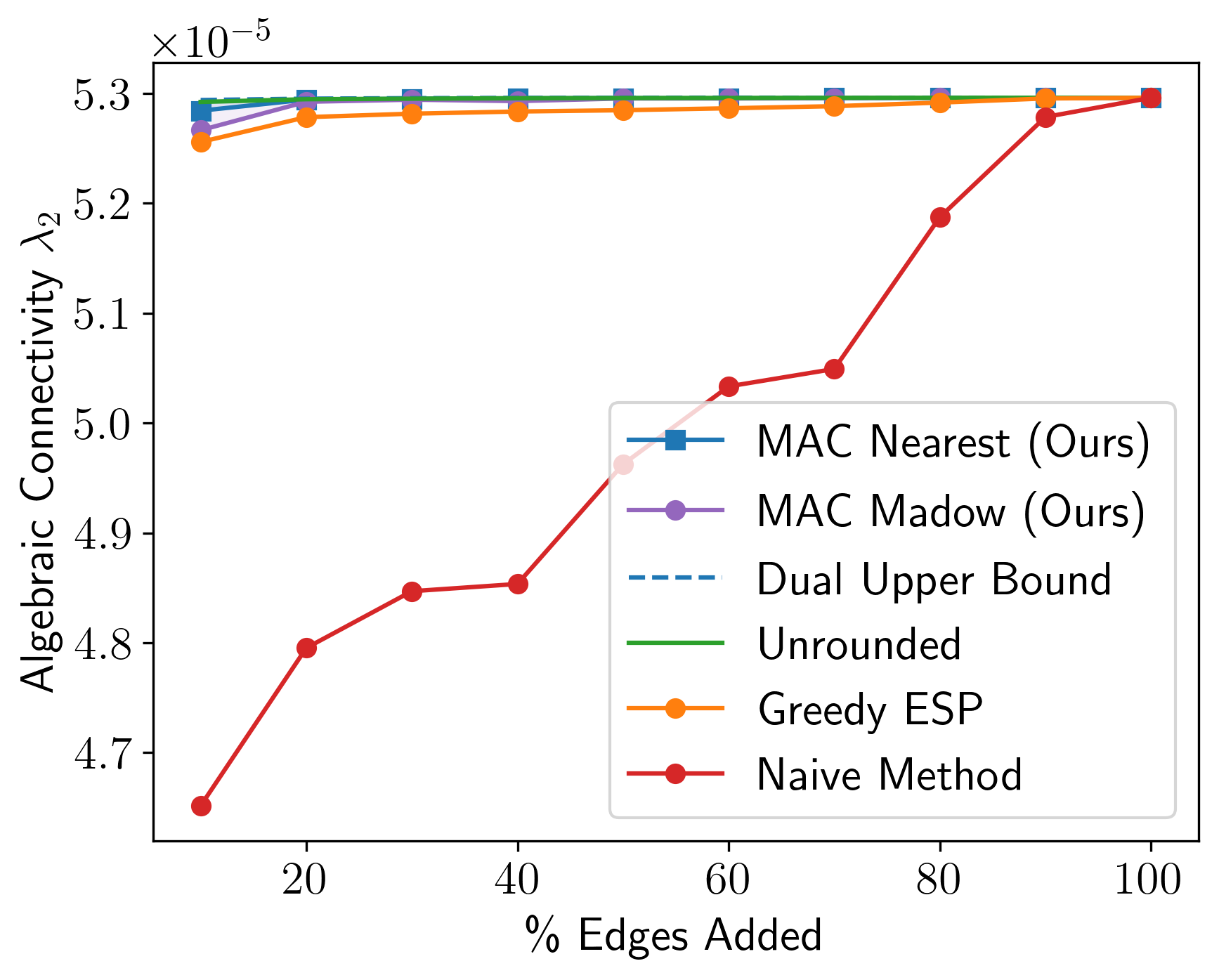}%
    \includegraphics[width=0.25\linewidth]{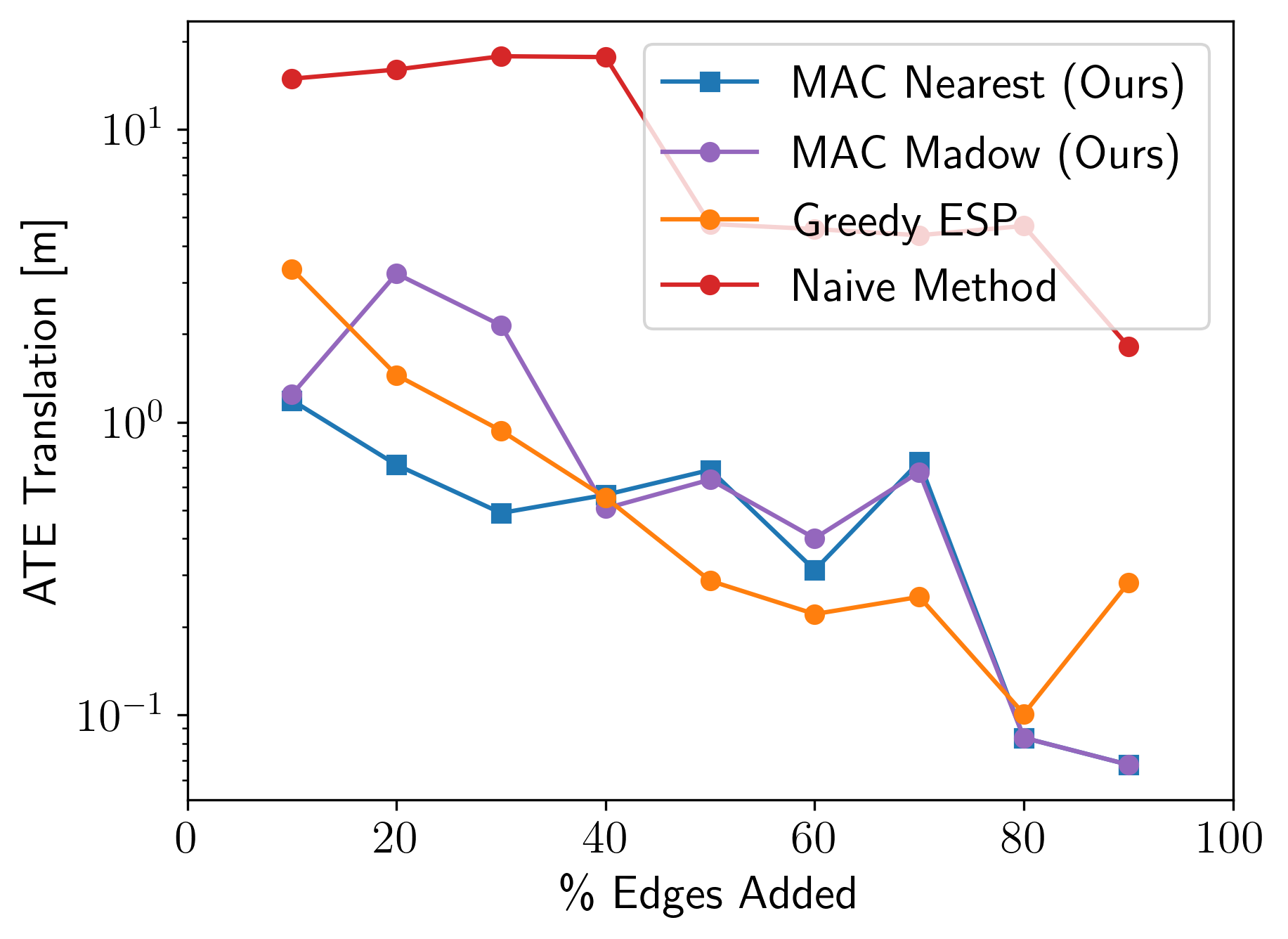}%
    \includegraphics[width=0.25\linewidth]{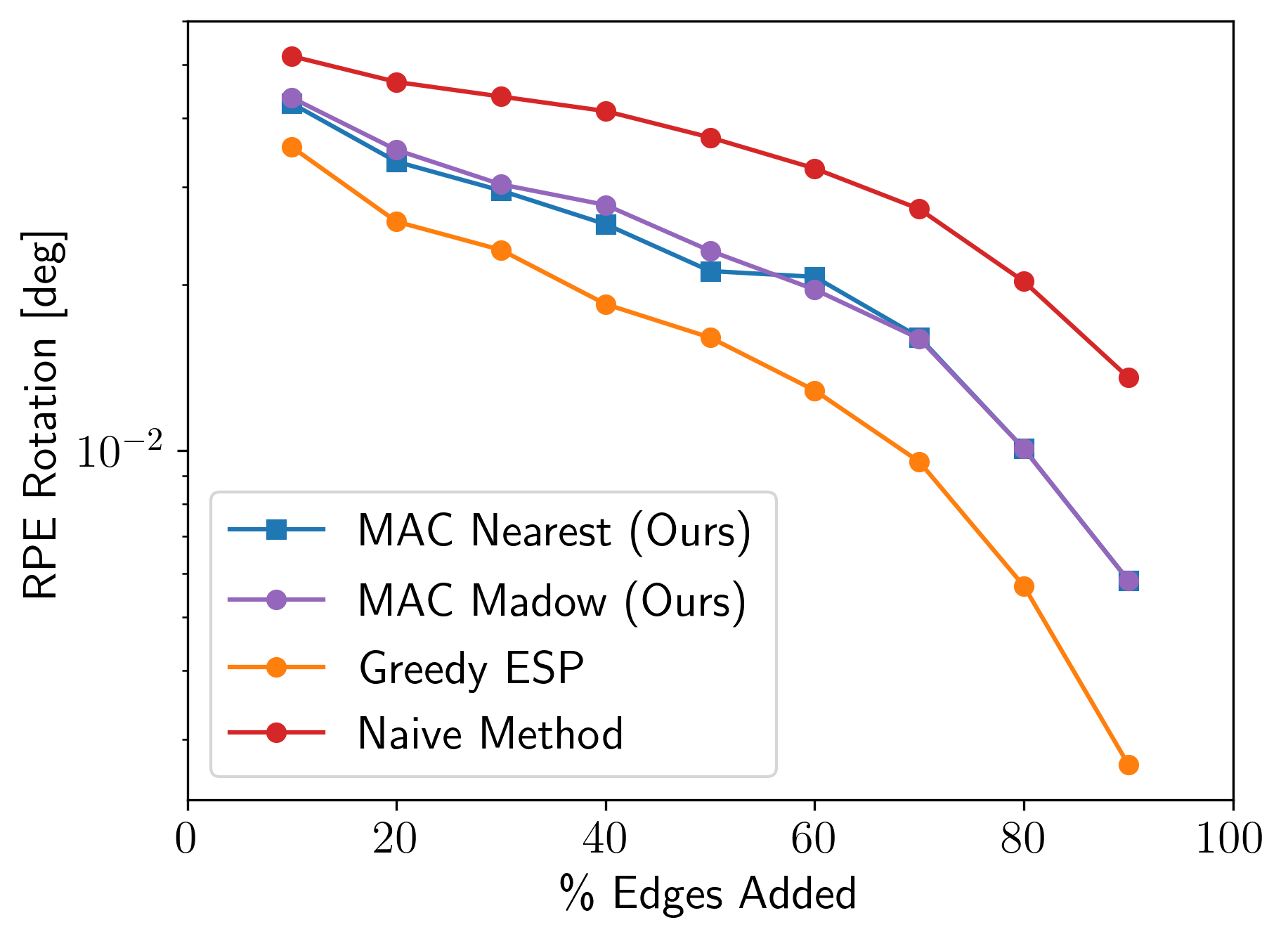}%
    \includegraphics[width=0.25\linewidth]{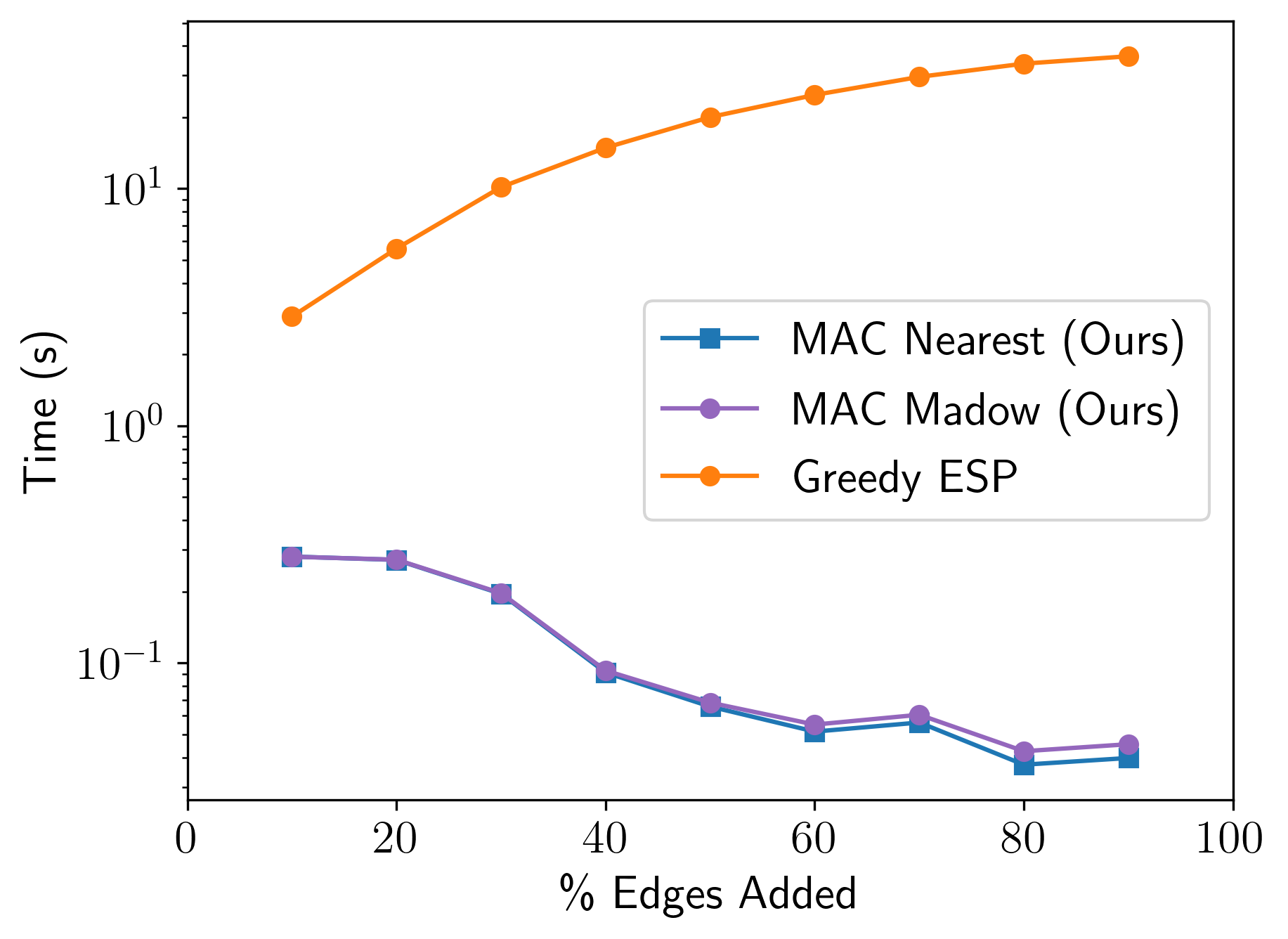}%
    \caption{\AIS2Klinik{}\\ \label{fig:quantitative:ais2klinik}}
  \end{subfigure}
  \begin{subfigure}{1.0\linewidth}
    \centering
    \includegraphics[width=0.25\linewidth]{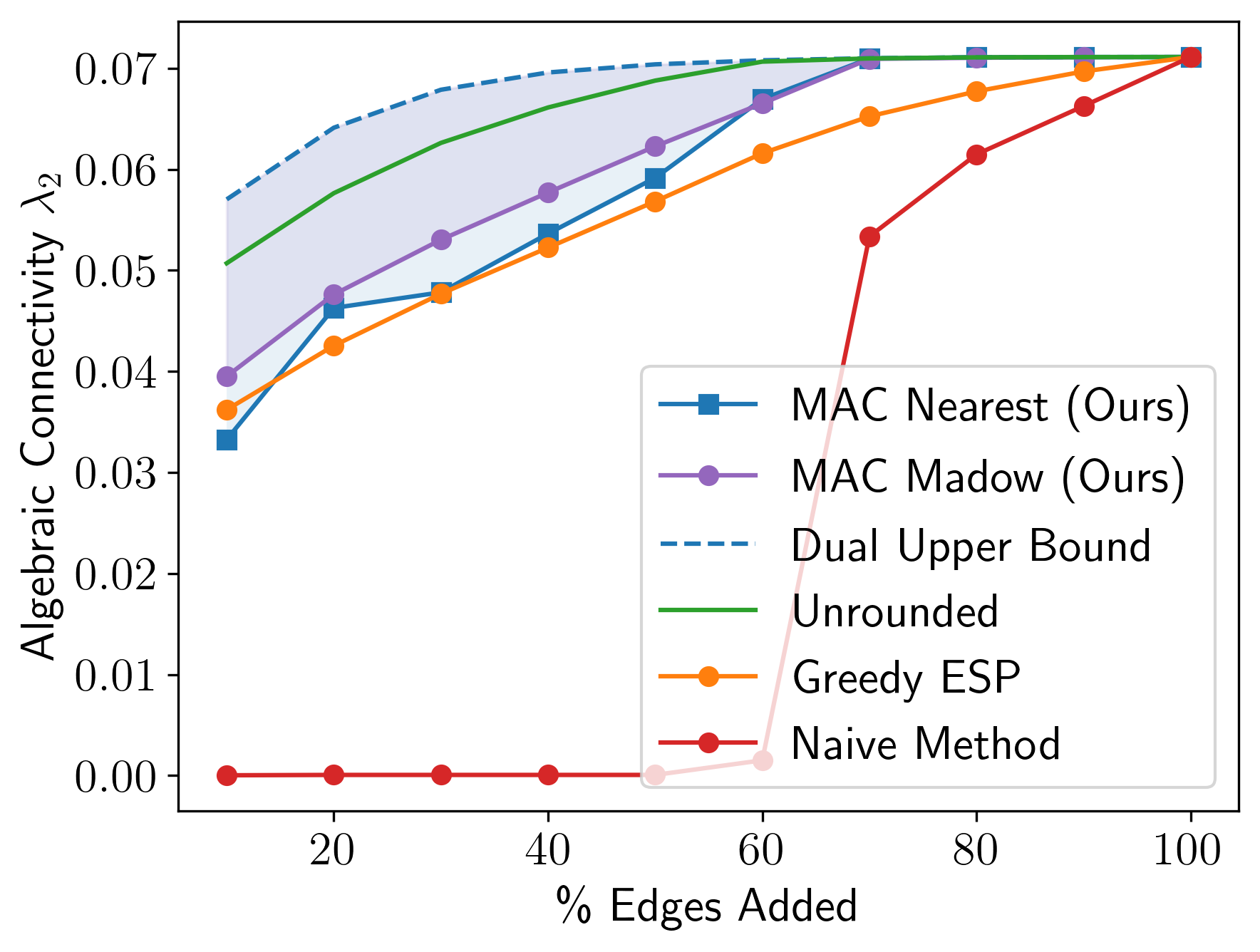}%
    \includegraphics[width=0.25\linewidth]{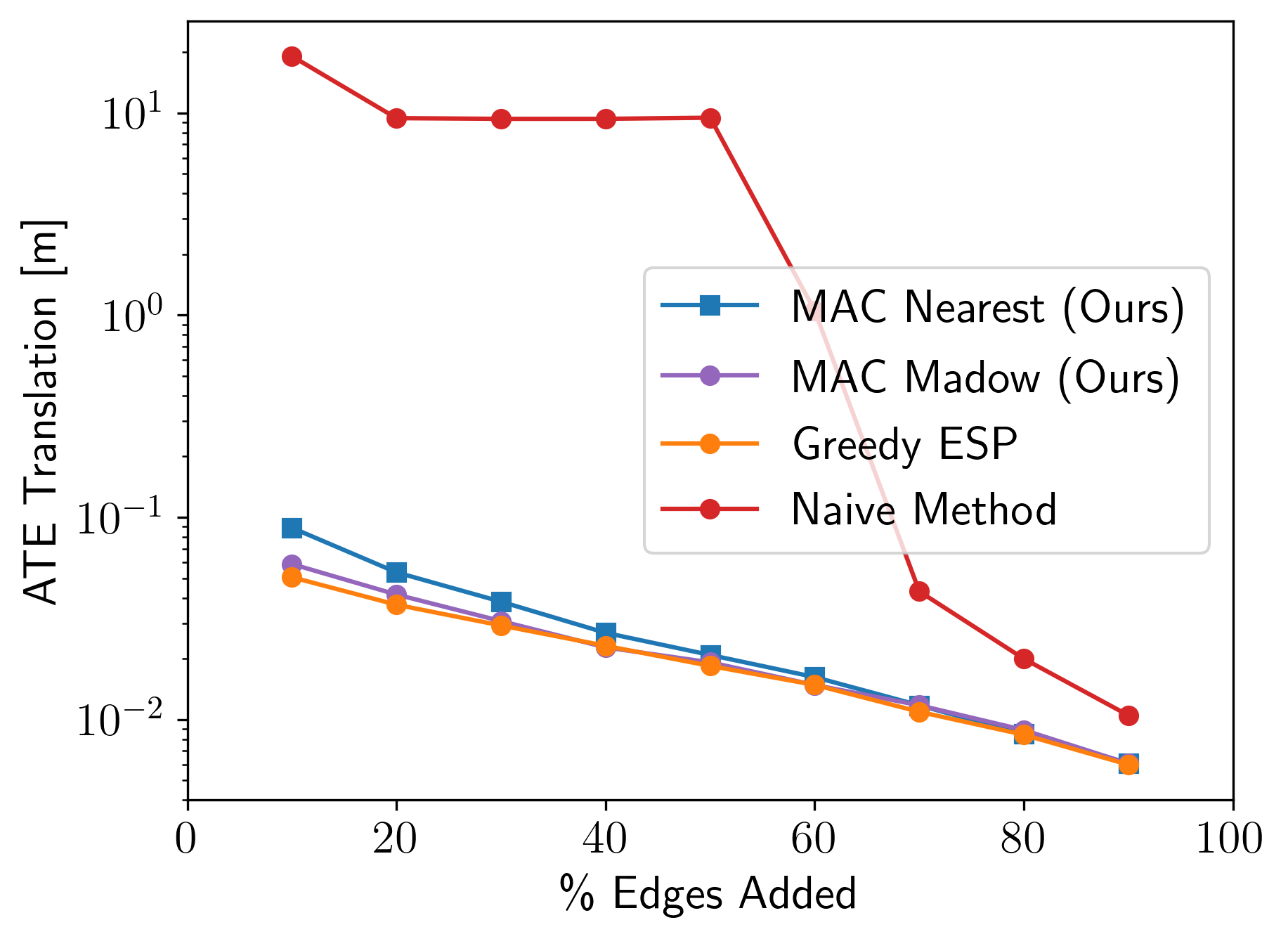}%
    \includegraphics[width=0.25\linewidth]{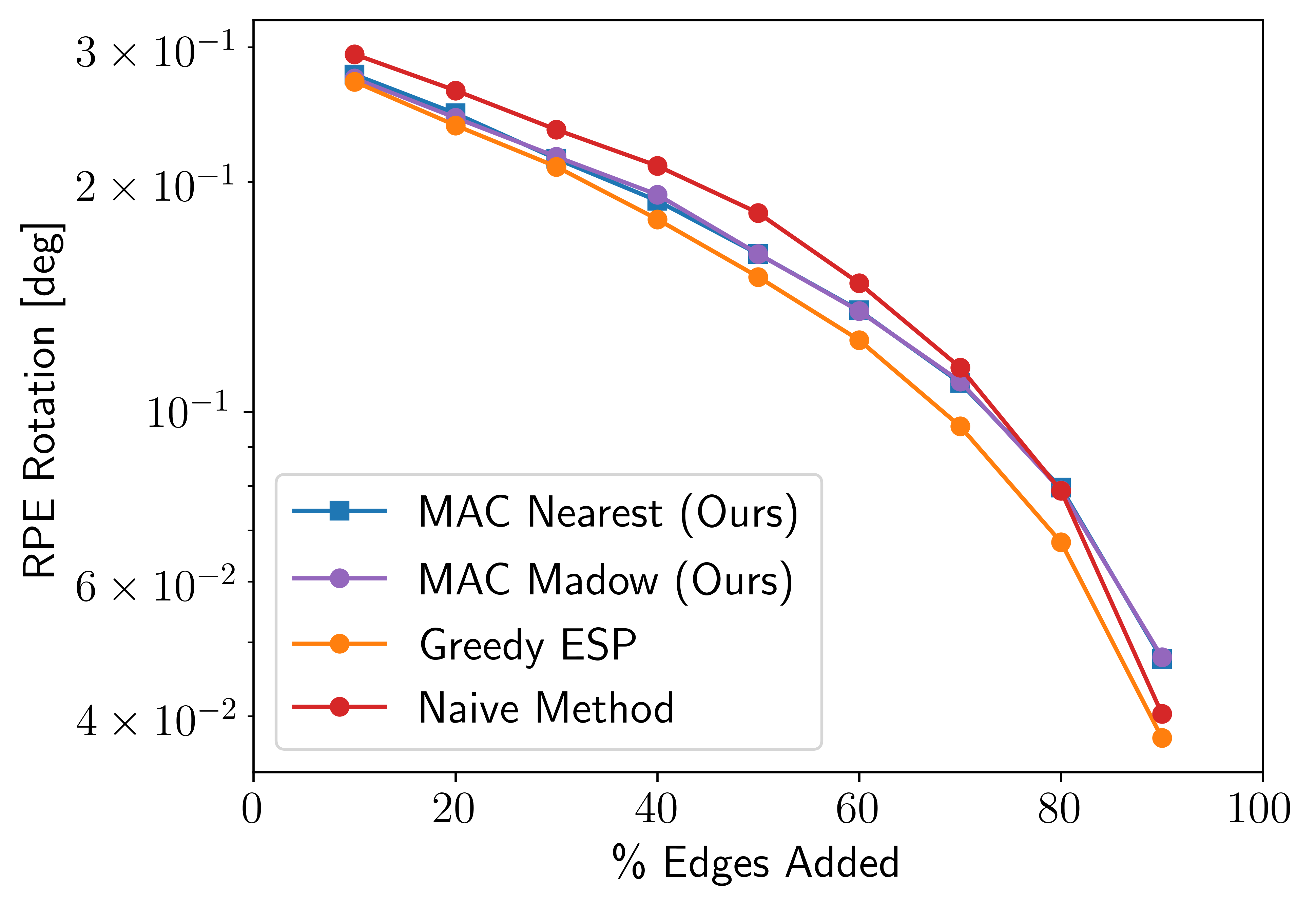}%
    \includegraphics[width=0.25\linewidth]{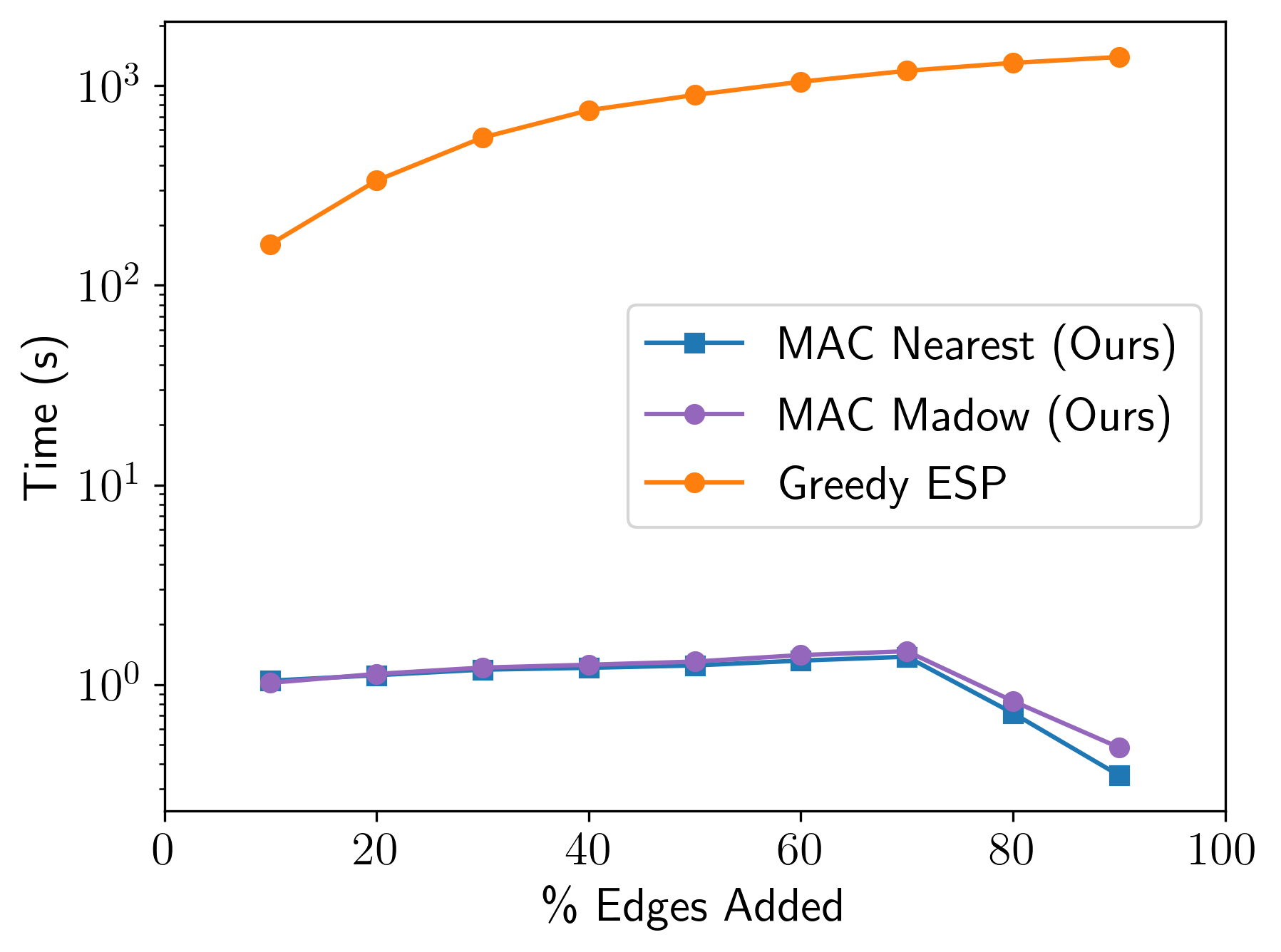}%
    \caption{\City10K{}\\ \label{fig:quantitative:city10k}}
  \end{subfigure}

  \caption{\textbf{Quantitative results for 2D pose-graph sparsification}.
    Pose-graph optimization results for (a) the \Intel{} dataset, (b) the
    \AIS2Klinik{} dataset, and (c) the \City10K{} dataset with varying degrees
    of sparsity (as percent of candidate edges added). Left to right: The algebraic connectivity
    of the graphs obtained using each method (larger is
    better) with the shaded regions indicating the suboptimality gap for each \MAC{} rounding procedure, the mean translation error and relative
    rotation error compared to a maximum-likelihood estimate computed for the
    \emph{full graph}, i.e. with all edges retained (smaller is better; note the
    logarithmic scale), and the computation time (logarithmic scale) for each
    approach. For 100\% loop closures, both algorithms return immediately, so no
    computation time is reported. \label{fig:quantitative-2d}}
\end{figure*}

\begin{figure*}
  \centering
    \begin{subfigure}{1.0\linewidth}
    \centering
    \includegraphics[width=0.25\linewidth]{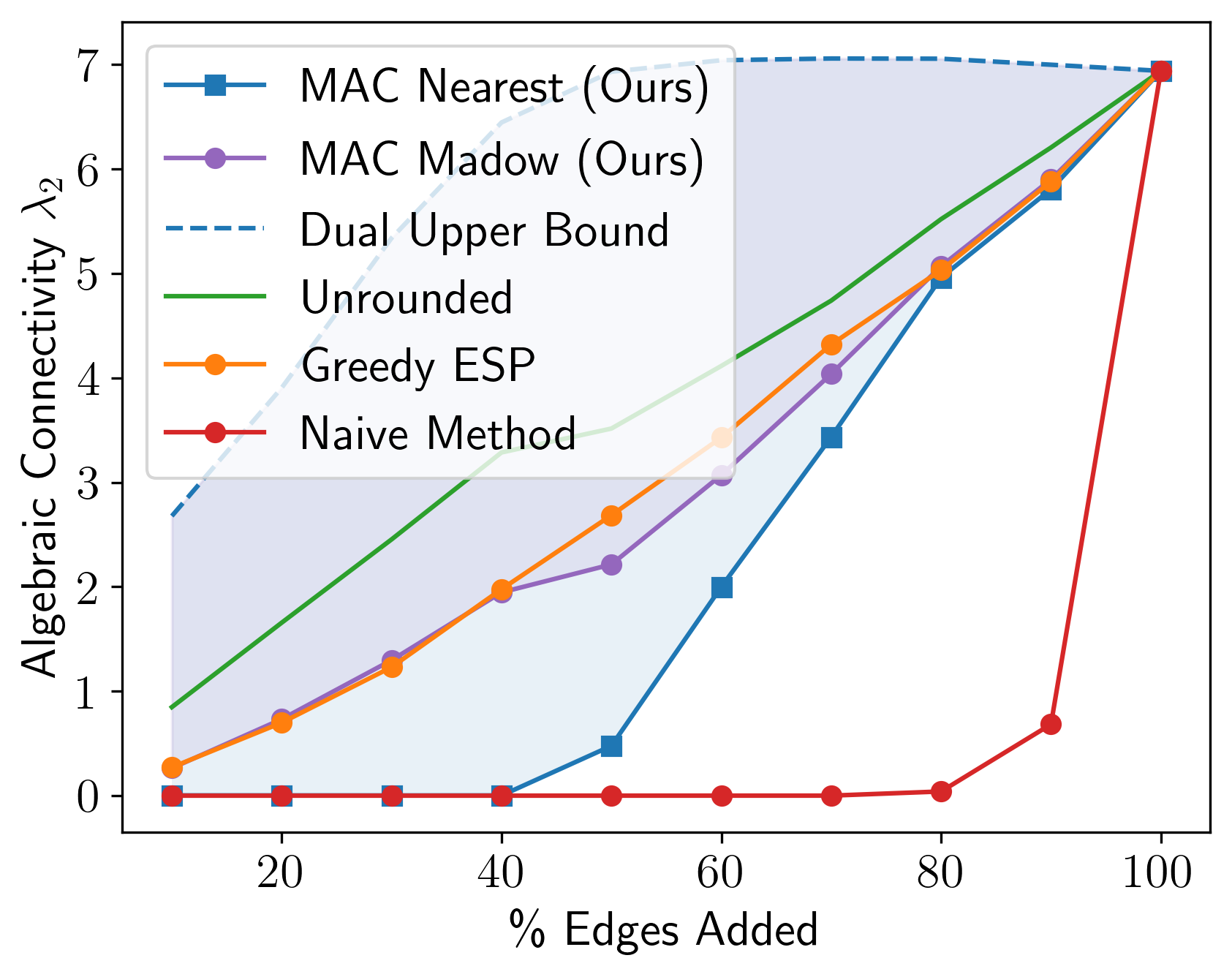}%
    \includegraphics[width=0.25\linewidth]{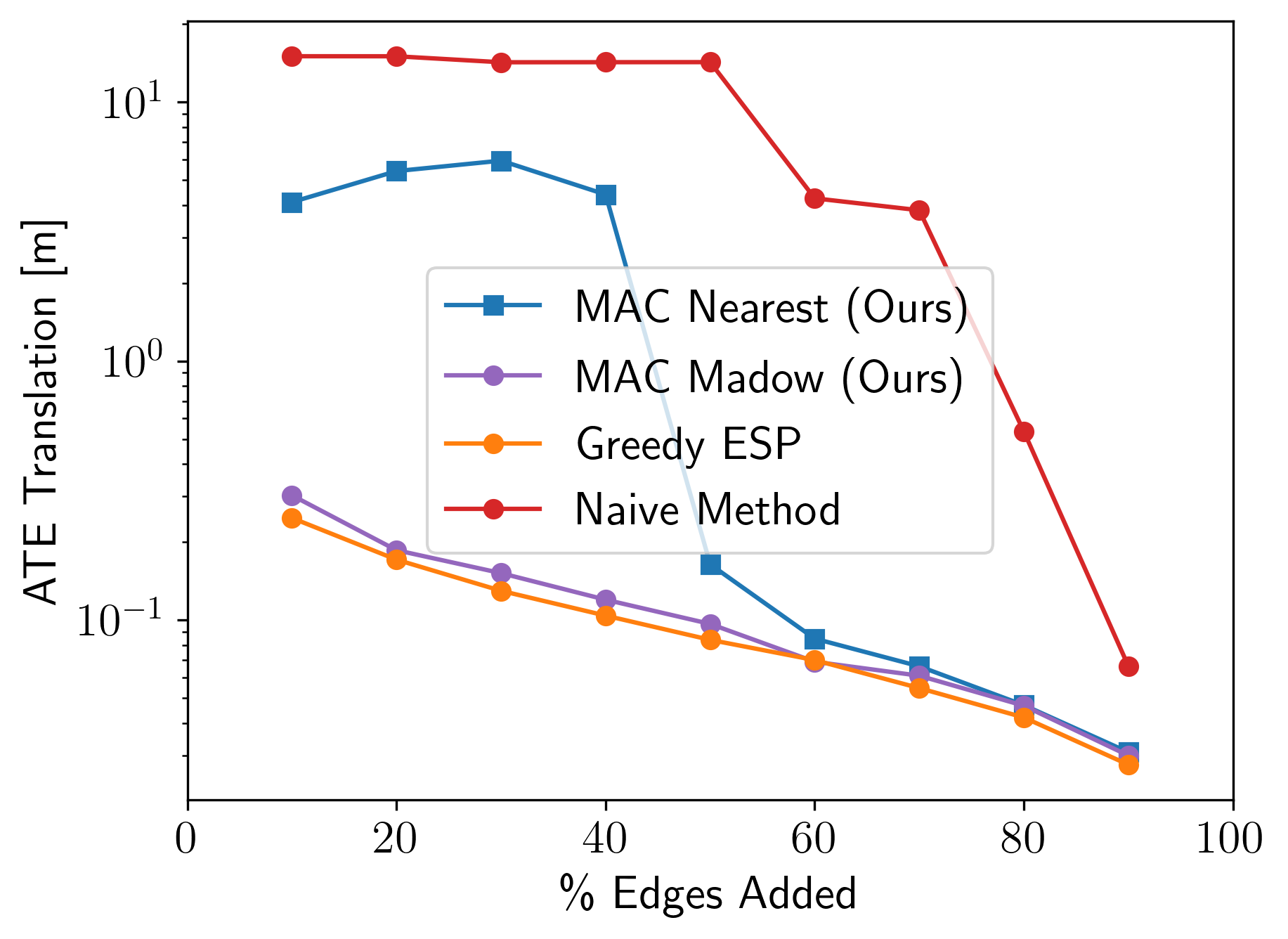}%
    \includegraphics[width=0.25\linewidth]{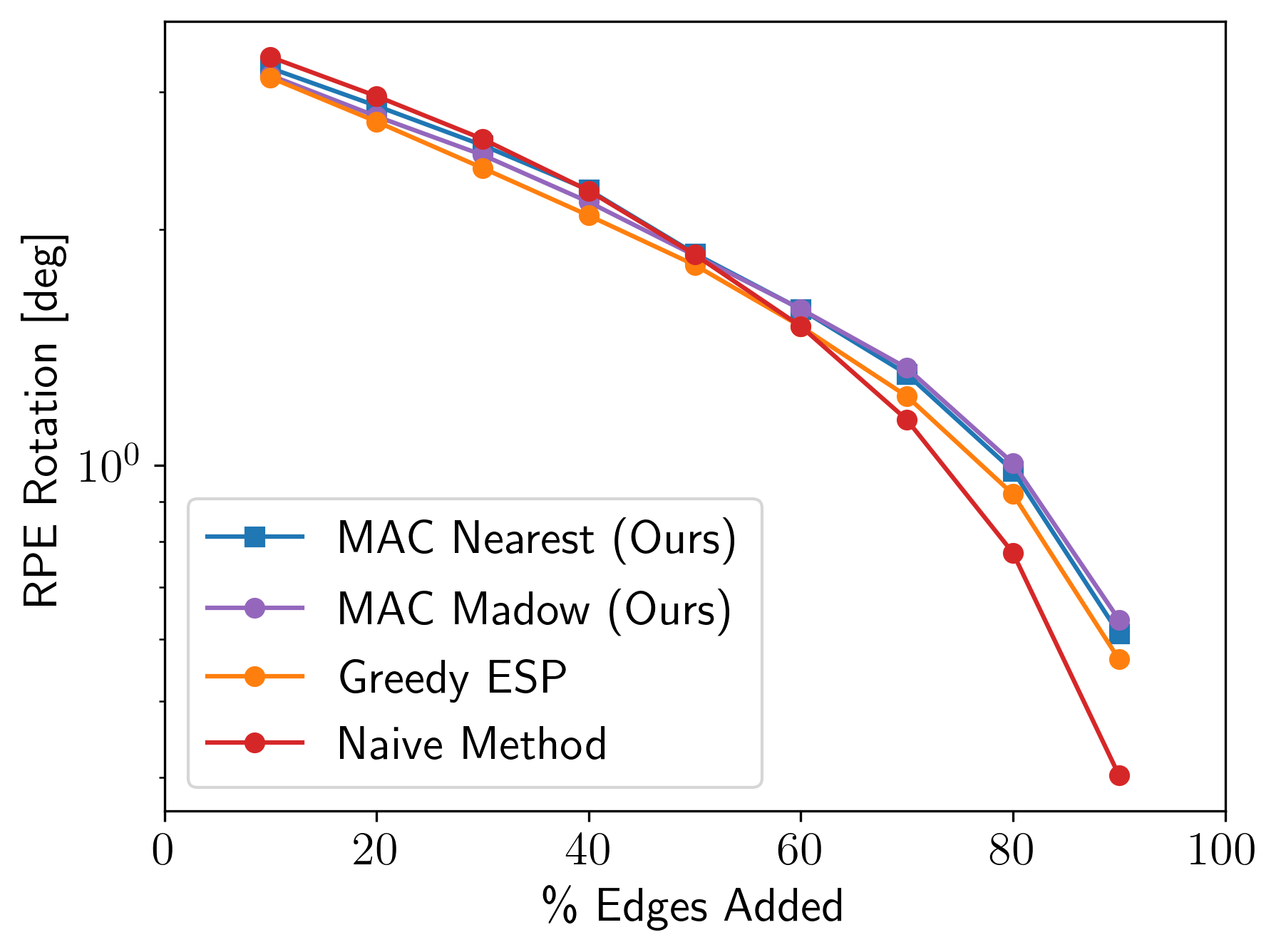}%
    \includegraphics[width=0.25\linewidth]{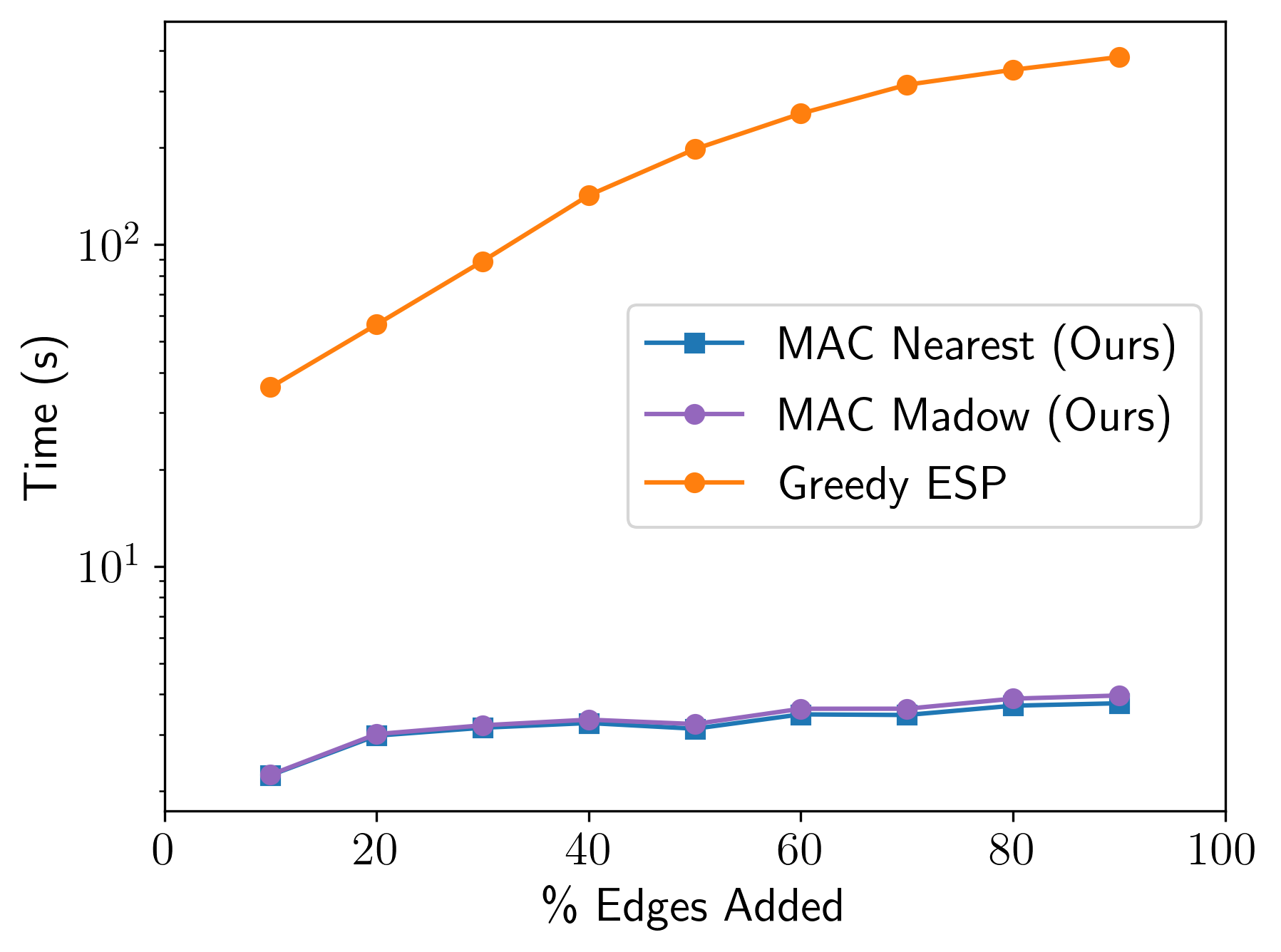}%
    \caption{\Grid{}\\ \label{fig:quantitative:grid}}
  \end{subfigure}
    \begin{subfigure}{1.0\linewidth}
    \centering
    \includegraphics[width=0.25\linewidth]{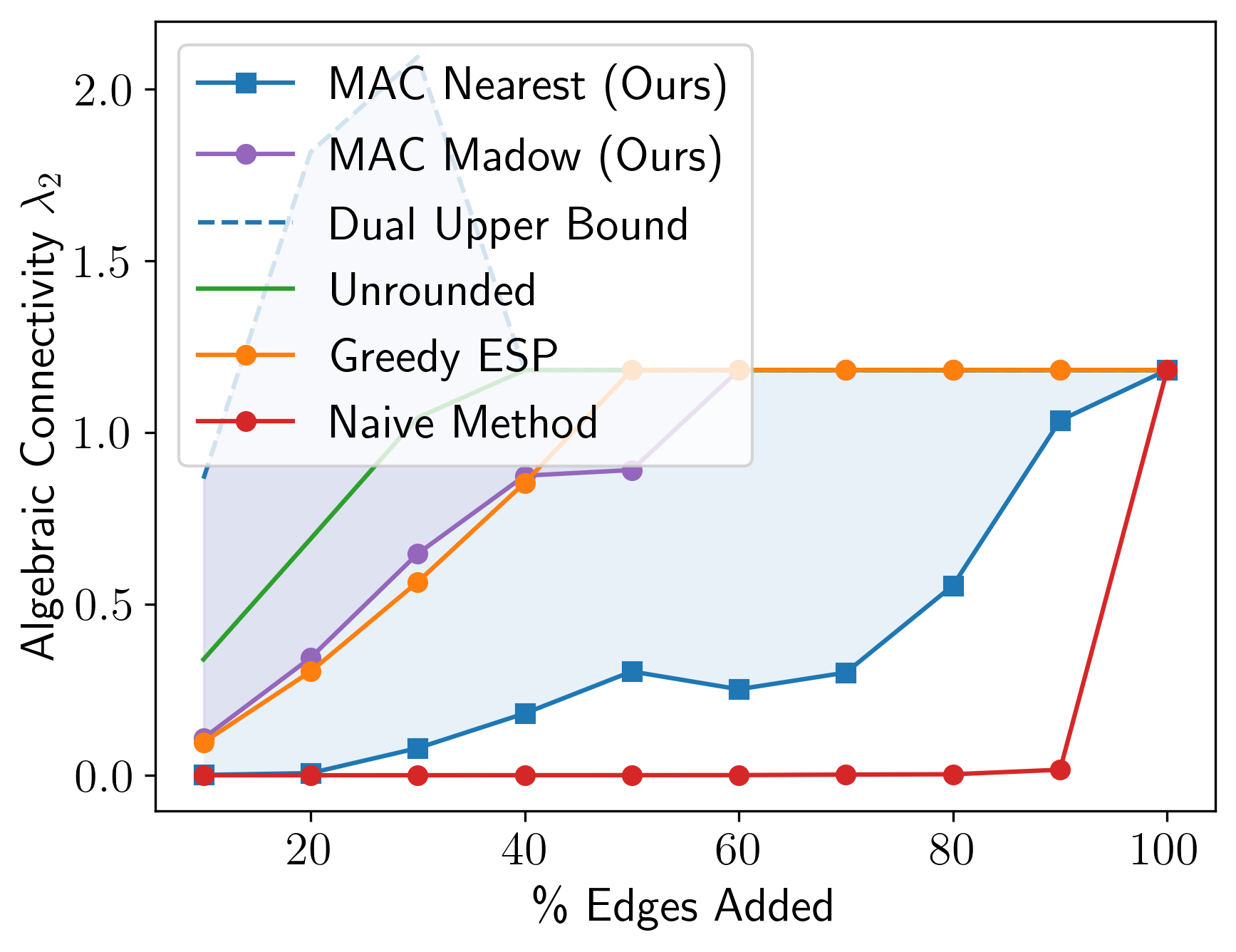}%
    \includegraphics[width=0.25\linewidth]{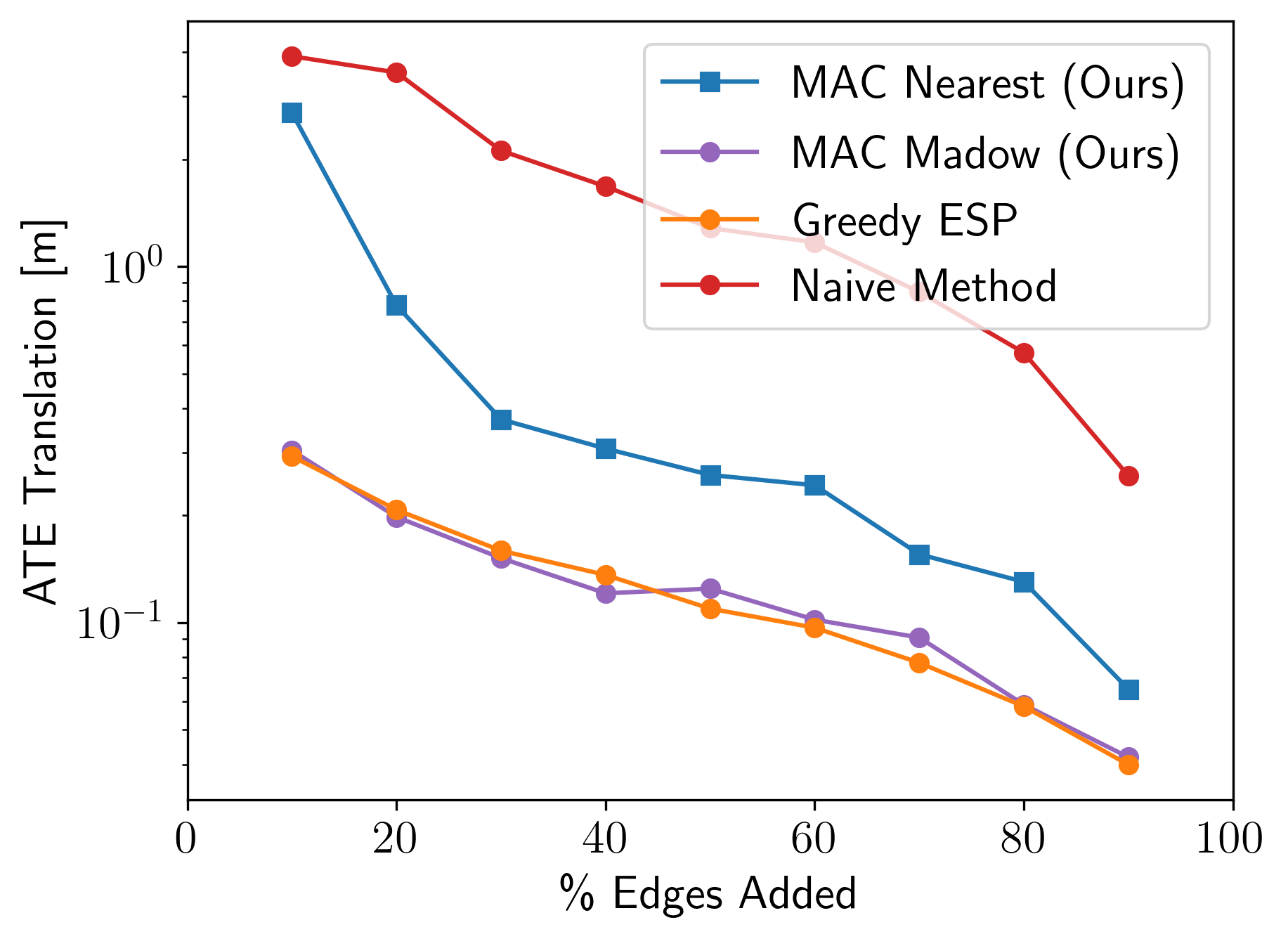}%
    \includegraphics[width=0.25\linewidth]{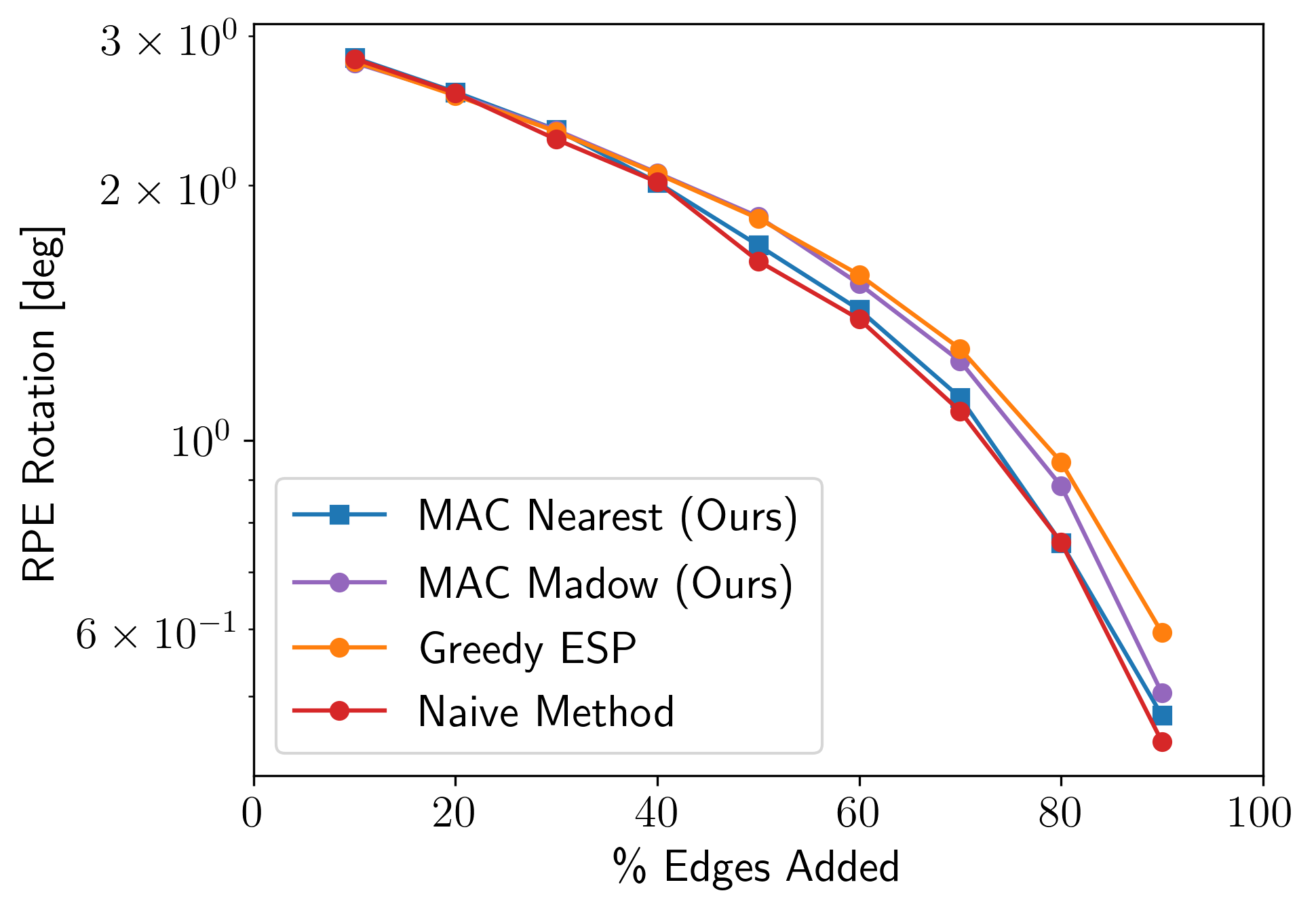}%
    \includegraphics[width=0.25\linewidth]{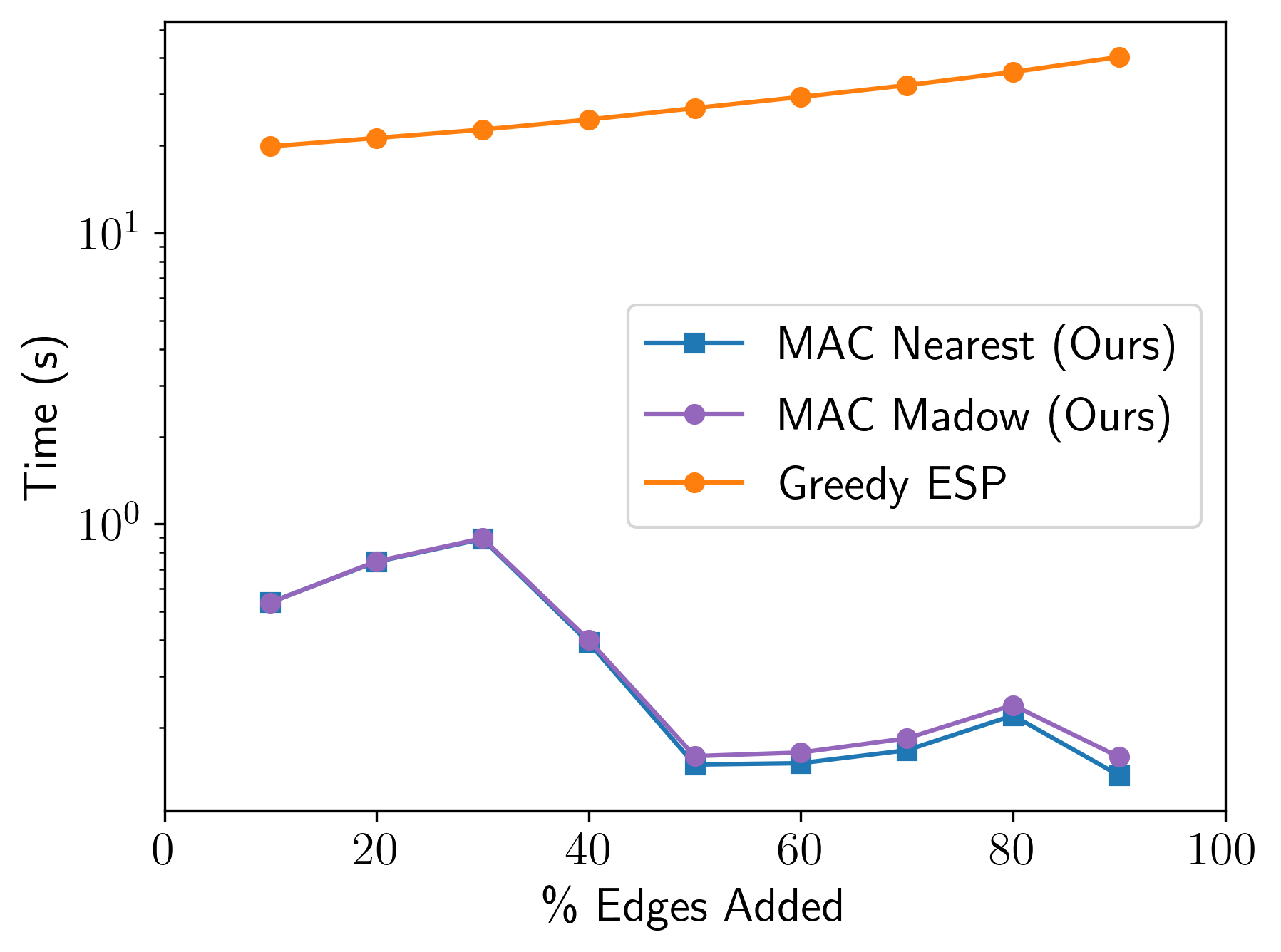}%
    \caption{\Torus{}\\ \label{fig:quantitative:torus}}
  \end{subfigure}
  \begin{subfigure}{1.0\linewidth}
    \centering
    \includegraphics[width=0.25\linewidth]{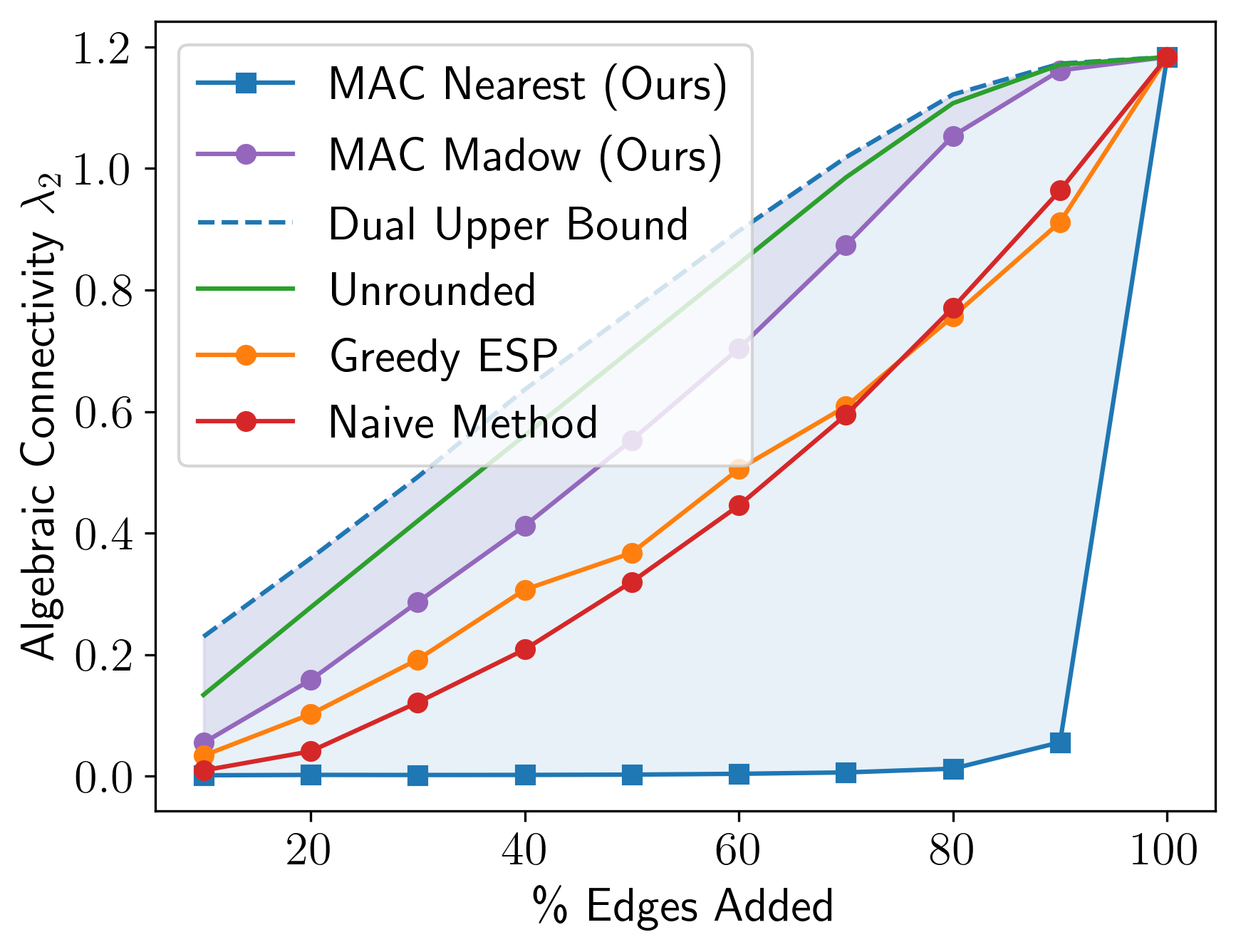}%
    \includegraphics[width=0.25\linewidth]{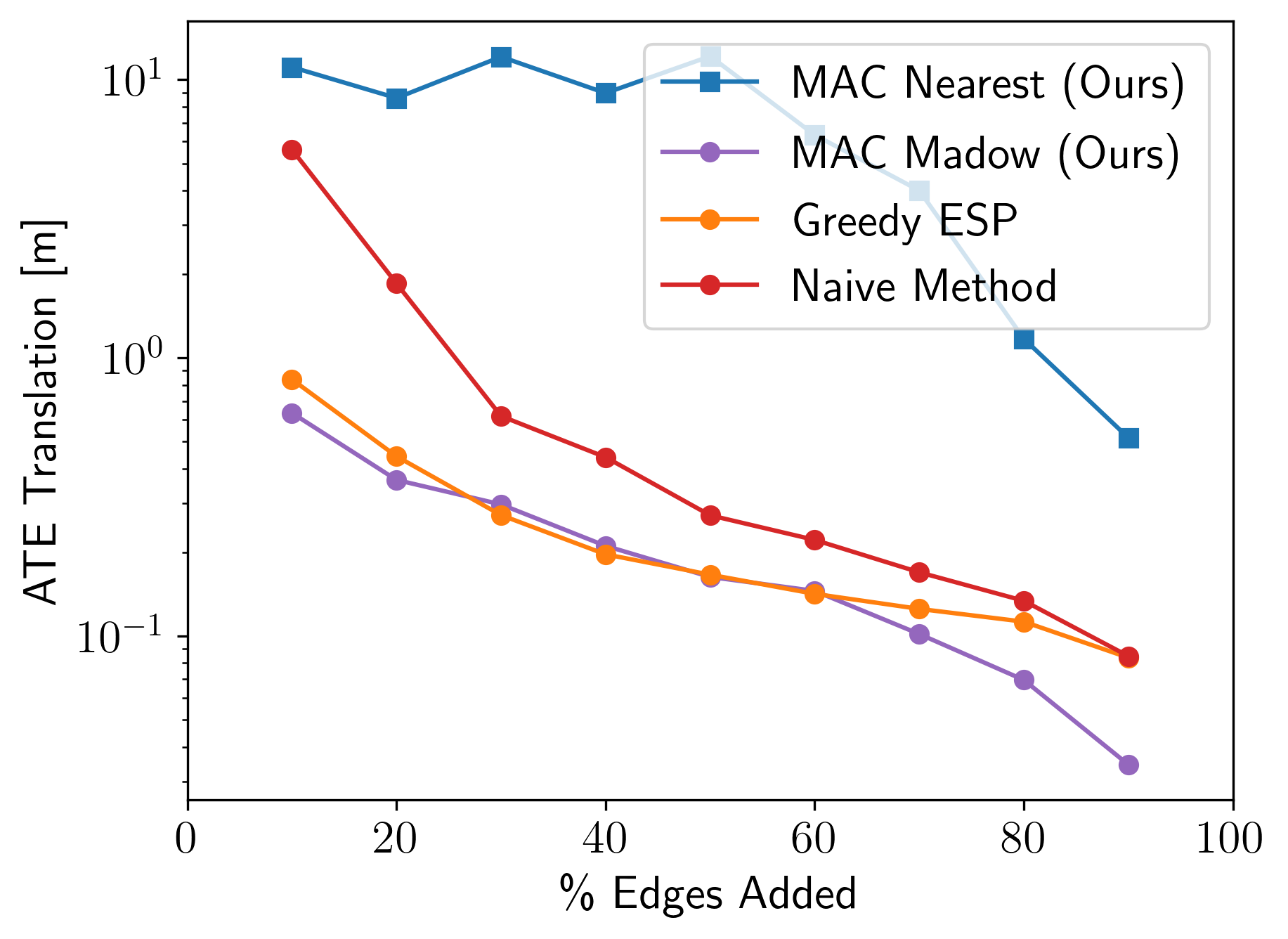}%
    \includegraphics[width=0.25\linewidth]{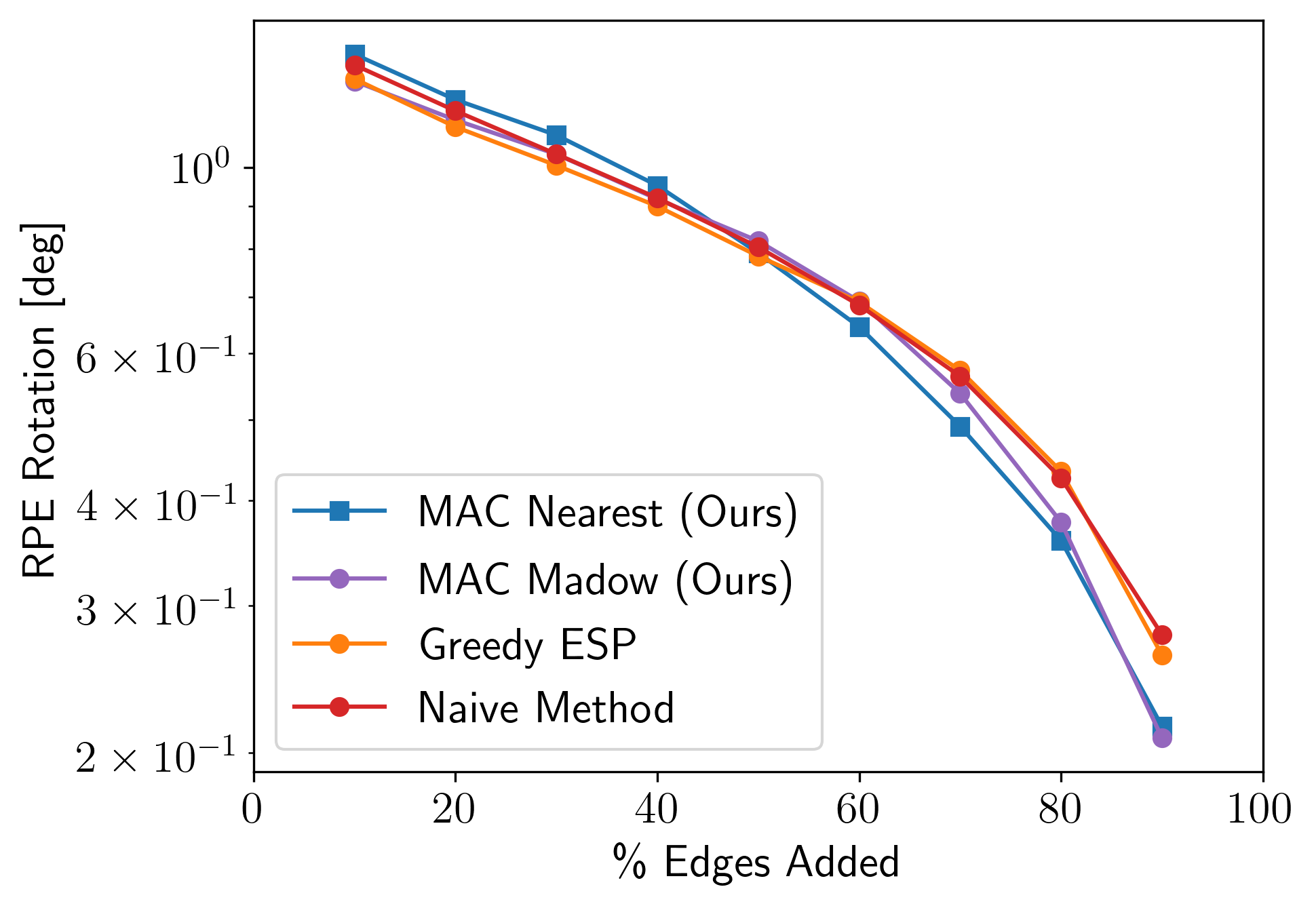}%
    \includegraphics[width=0.25\linewidth]{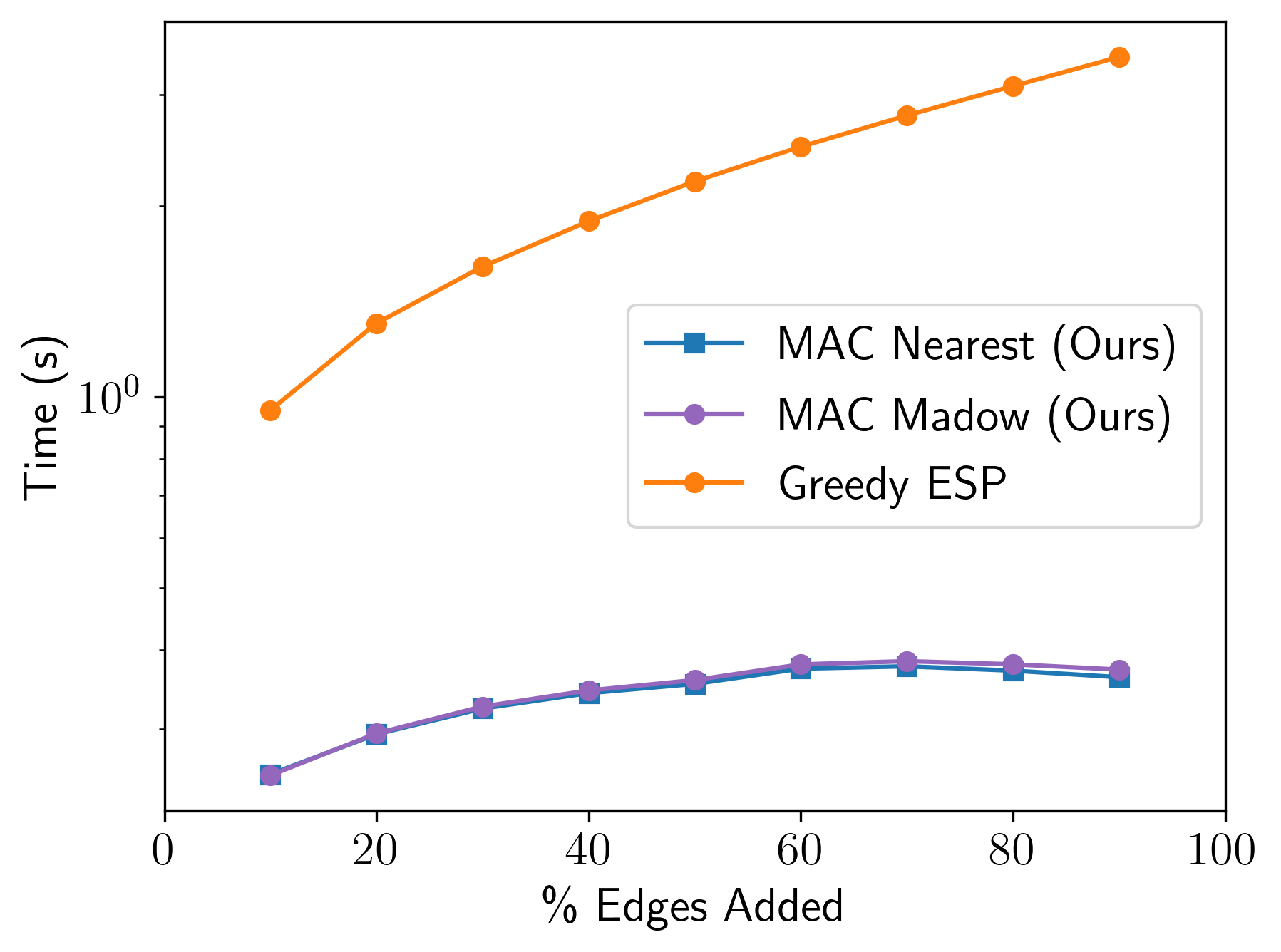}%
    \caption{\Sphere{}\\ \label{fig:quantitative:sphere}}
  \end{subfigure}
    \caption{\textbf{Quantitative results for 3D pose-graph sparsification}.
    Pose-graph optimization results for (a) the \Grid{} dataset, (b) the \Torus{}
    dataset, and (c) the \Sphere{} dataset with varying degrees of sparsity (as
    percent of candidate edges added). Left to right: The algebraic connectivity
    of the graphs obtained using each method (larger is
    better) with the shaded regions indicating the suboptimality gap for each \MAC{} rounding procedure, the mean translation error and relative rotation error compared to a maximum-likelihood estimate computed for the \emph{full graph}, i.e. with all edges
    retained (smaller is better; note the logarithmic scale), and the computation time
    (logarithmic scale)
    for each approach. For 100\% loop closures,
    both algorithms return immediately, so no computation time is
    reported. \label{fig:quantitative-3d}}
\end{figure*}

\begin{table}[t]
  \centering
  \setlength{\tabcolsep}{4pt}
  \begin{tabular}{c | c | c }
    \hline
    Dataset & \# Nodes & \# Candidate Edges  \\
    \hline
    \Intel{} & 1728 & 785 \\
    \Sphere{} & 2500 & 2500 \\
    \Torus{} & 5000 & 4049 \\
    \Grid{} & 8000 & 14237 \\
    \City10K{} & 10000 & 10688 \\
    \AIS2Klinik{} & 15115 & 1614 \\
    \hline   
  \end{tabular}
  \caption{Summary of the datasets used in our experiments. \label{table:datasets}}
\end{table}

We implemented the \MAC{} algorithm in Python and all computational experiments
were performed on a 2.4 GHz Intel i9-9980HK CPU. For computation of the Fiedler
value and the corresponding vector, we use TRACEMIN-Fiedler
\cite{manguoglu2010tracemin, sameh1982trace}. Where in our previous work
\cite{doherty2022spectral} we used an LU decomposition as a subprocedure of
TRACEMIN-Fiedler, in this paper we instead use the Cholesky decomposition
provided by Suitesparse \cite{cholmod}. Both options are exposed by the MAC
library.  We have also improved the computational efficiency of calculating the Fiedler vectors required to form the direction-finding subproblems in each iteration of the Frank-Wolfe method (Algorithm \ref{alg:frank-wolfe-general}): since the Fiedler
vectors obtained in subsequent iterations are often close, \emph{warm-starting} the TRACEMIN-Fiedler method using the Fiedler vector computed in the \emph{previous} iteration can result in substantial computational savings. In all experiments,
we run \MAC{} for a maximum of 20 iterations, or when the duality gap in
equation \eqref{eq:duality-gap-as-bound} reaches a tolerance of $10^{-8}$.

In the following sections, we present experimental results using \MAC{} for
sparsification of benchmark pose-graph SLAM datasets, as well as on real data
from the University of Michigan North Campus Long-Term (NCLT) Dataset
\cite{ncarlevaris-2015a}. In all of our experiments, we use odometry edges
(between successive poses) to form the base graph and loop closure edges as
candidate edges. We consider selection of $10\%, 20\%, \ldots, 100\%$ of the
candidate loop closure edges in the sparsification problem.

We compare both of the rounding procedures (from Section \ref{sec:rounding}),
along with a na\"ive heuristic method which selects simply the subset of
candidate edges with the largest edge weights. This simple heuristic approach
serves two purposes: First, it provides a baseline, topology-agnostic approach
to demonstrate the impact of considering graph connectivity in a sparsification
procedure; second, we use this method to provide a sparse initial estimate to
our algorithm. We also compare our approach with the ``greedy edge selection
problem'' (Greedy ESP) algorithm of \citet{khosoussi2019reliable} for
approximate D-optimal graph sparsification. In an effort to make the comparison
as fair as possible, we have implemented our own version of the Greedy ESP
algorithm in Python which makes use of lazy evaluation
\cite{krause2014submodular, minoux2005accelerated}, as well as incremental,
sparse Cholesky factorization, as proposed in \cite{khosoussi2019reliable}.

\subsection{Benchmark pose-graph SLAM datasets}\label{sec:benchmark}

We evaluated the \MAC{} algorithm using several benchmark pose-graph SLAM
datasets. We present results on six datasets in this paper (summarized in Table
\ref{table:datasets}). Three of the datasets we considered (\Intel{},
\AIS2Klinik{}, and \City10K{}) represent 2D SLAM problems (corresponding to
Problem \ref{se-mle} with $d=2$), and the remaining three (\Sphere{}, \Torus{},
and \Grid{}) represent 3D SLAM problems. The \Intel{} dataset and the
\AIS2Klinik{} dataset are both obtained from real data, while the remaining
datasets are synthetic. We apply \MAC{} to sparsify each graph to varying
degrees of sparsity and apply SE-Sync \cite{rosen2019se} to compute globally
optimal estimates of robot poses for the graphs with edge sets selected by each
method.\footnote{In all of our experiments, SE-Sync returned \emph{certifiably
    optimal} solutions to Problem \ref{se-mle}.}

Figure \ref{fig:qualitative-3d} gives a qualitative comparison of the results
from our approach as compared with the baselines on each of the 3D benchmark
datasets we considered. Interestingly, the Greedy ESP method and MAC with
randomized rounding (denoted MAC (Madow)) produce qualitatively very similar
results, despite optimizing for different criteria. Perhaps more surprisingly,
MAC with the ``nearest neighbors'' style of rounding (denoted MAC (Nearest))
fails quite dramatically on the \Grid{} dataset and similarly on the \Torus{}
dataset. We had not observed this phenomenon in our prior work
\cite{doherty2022spectral} where we had only thus far applied the method to 2D
datasets. This may be a consequence of the synthetic construction
of these datasets (and we have not observed similar failures on any real data). In
particular, each of these datasets is constructed in a manner that produces
symmetries in the graph topology. We do not explicitly consider the presence of
symmetries in our objective (though if these symmetries are known \emph{a
  priori}, it is possible to do so in principle, see, e.g. related discussion in
\cite{boyd2006convex}). In problems exhibiting symmetry, there may be multiple solutions $x^*_1, \ldots, x^*_p$ to Problem \ref{prob:max-aug-alg-conn} that achieve the global minimum.  In these cases, the concavity of the objective in \eqref{eq:relaxation} implies that the average $\bar{x} \triangleq (1/p) \sum_i x^*_i$ of these solutions will always achieve an objective value that is at least as good for the \emph{relaxed} problem (Problem \ref{prob:relaxation}) as that attained by the $x_i$ themselves.  In such cases, applying a simple top-$K$ rounding procedure may select edges that belong to multiple distinct optima $x^*_i$, $x^*_j$.  On the other hand, because the greedy method (by construction) builds its solutions sequentially (by selecting one edge at a time), as soon as a distinguishing edge belonging to one of the distinct optima has been added to the solution under construction (thus breaking the symmetry), the greedy method is likely to prefer selecting other edges belong to that specific configuration in subsequent iterations.  We conjecture that the randomness introduced in the Madow sampler helps to achieve a similar kind of symmetry breaking when rounding relaxed solutions of Problem \ref{prob:relaxation}.

For a quantitative comparison of each method, we report three performance
measures: (1) the algebraic connectivity $\lambda_2(\LapRotW(\selection))$ of
the graphs determined by each edge selection $\selection$; (2) the translational
part of the absolute trajectory error (ATE) between the SLAM solution computed from
the sparsified pose graph and the solution computed using the full pose graph
(i.e. keeping 100\% of the edges); and (3) the average of the rotational
component of the \emph{relative} pose error. We also report the computation time
required to obtain a solution for each method.

Figure \ref{fig:quantitative-2d} summarizes our quantitative results on each 2D
benchmark dataset. Our approach consistently achieves better connected graphs
(as measured by the algebraic connectivity). In most cases, a maximum of 20
iterations was enough to achieve solutions to the relaxation with algebraic
connectivity very close to the dual upper bound (and therefore nearly globally
optimal).

Beyond providing high-quality sparse measurement graphs, our approach is also
fast. For the \Intel{} dataset, all solutions were obtained in less than 250
milliseconds. Sparsifying the (larger) \AIS2Klinik{} dataset required up to 700
ms, but only around 100 ms when larger edge selections were allowed. In those
cases, the duality gap tolerance was reached and optimization could terminate.
The largest dataset (in terms of candidate edges) is the \City10K{} dataset,
with over 10000 loop closure measurements to select from. Despite this, our
approach produces near-optimal solutions in under 2 seconds.

With respect to the suboptimality guarantees of our approach, it is interesting
to note that on both the \Intel{} and \City10K{} datasets, the rounding
procedure introduces fairly significant degradation in algebraic connectivity -
particularly for more aggressive sparsity constraints. In these cases, it seems
that the Boolean relaxation we consider leads to fractional optimal solutions.
It is not clear in these cases whether the integral solutions obtained by
rounding are indeed suboptimal for the Problem \ref{prob:max-aug-alg-conn}, or
whether this is a consequence of the \emph{integrality gap} between
\emph{global} optima of the relaxation and of Problem
\ref{prob:max-aug-alg-conn}.\footnote{In general, even simply \emph{verifying}
  the global optimality of solutions to Problem \ref{prob:max-aug-alg-conn} is
  NP-Hard \cite{mosk2008maximum}.}

Figure \ref{fig:quantitative-3d} summarizes our quantitative results on each of
the 3D benchmark pose-graph optimization datasets. As previously discussed, the
\MAC{} (Nearest) method, which uses the ``nearest neighbors'' rounding
procedure, struggles on these synthetic datasets, whereas the \MAC{} (Madow)
approach using the randomized rounding procedure incurs significantly less
degradation in connectivity compared to the unrounded solution to the
relaxation. Interestingly, on both the \Grid{} and \Torus{} datasets, we also
observe that the dual upper bounds computed during optimization can be quite
loose. This is in contrast to the 2D datasets, where \MAC{} fairly consistently
achieved solutions, at least to the relaxation, which were quite close in
algebraic connectivity to the dual upper bounds.

\subsection{Real multi-session SLAM data}

In this section, we evaluate the performance of MAC on real large-scale
multi-session data from the University of Michigan North Campus Long-Term
Dataset \cite{ncarlevaris-2015a}. To generate pose graphs, we processed LiDAR
data from three sessions using Fast-LIO2 \cite{xu2022fast} for odometry with
loop closure measurements generated based on ScanContext \cite{kim2018scan}. For
loop closures between sessions, we use DiSCo-SLAM \cite{huang2021disco}. We
processed each session in order, so loop closures are generated from the second
session to the first as well as from the third session to both the first and
second sessions. This entire procedure was performed offline as a preprocessing
step to construct multi-session pose graphs. We then sparsified loop closures in
this dataset using \MAC{} (making no distinction between inter- and
intra-session loop closures). We use the same metrics for comparison as in our
evaluation on benchmark pose-graph SLAM datasets, but trajectory errors here are
computed with respect to the \emph{ground truth}, rather than the optimal
solution for the full graph.

An interesting difference between this dataset and the benchmark pose-graph
optimization datasets we considered in the previous section is that the loop
closures generated using ScanContext can be corrupted by outliers. In our
application, the \MAC{} algorithm itself does not account for any possibility of
outliers. In these examples, it was possible to mitigate the influence of
outlier loop closures simply by using a Cauchy robust kernel for all loop
closure measurements while applying \MAC{} without modification.\footnote{As an
  implementation note, we use GTSAM \cite{dellaert2012factor} for pose-graph
  optimization in these experiments, rather than SE-Sync \cite{rosen2019se} as
  in the previous section, in order to simplify the use of robust kernels for
  loop closures. However, the consequence of this is that one cannot verify the
  global optimality of maximum-likelihood estimators computed in this section.}

Qualitative comparisons of the results for the different methods are provided in
Fig. \ref{fig:nclt-pointclouds} and Fig. \ref{fig:nclt-offset-plot}. Both
figures show the case where $20\%$ of the loop closures are retained. Figure
\ref{fig:nclt-pointclouds} provides a top-down view of the LiDAR point clouds
reconstructed based on the pose estimates for graphs obtained using each
sparsification method. \MAC{} (with either rounding approach) and the Greedy ESP
method provide results that are nearly indistinguishable visibly, whereas the
na\"ive baseline fails to accurately anchor the trajectories from different
sessions relative to one another. The offset trajectories visualized in Figure
\ref{fig:nclt-offset-plot} give a closer look at this phenomenon. The na\"ive
method achieves high connectivity between some parts of the three sessions, but
poor connectivity in others. In contrast, the solution obtained by \MAC{}
appears to provide a more uniform distribution of loop closures across the
trajectories.

The corresponding quantitative comparison is given in Figure
\ref{fig:nclt-quantitative}. The comparison of the algebraic connectivity of
each solution once again demonstrates a slight degradation of solution quality
when using ``nearest neighbors'' rounding approach. These solutions similarly
perform worse in terms of trajectory error metrics. Notably, \MAC{} (Madow)
produces \emph{rounded} solutions whose algebraic connectivity is numerically
almost identical to the unrounded solutions. This suggests that the solution to the relaxation for this dataset is very close to the feasible
set for the original problem. Unfortunately, since we did not achieve
convergence of the objective value and the dual upper bound, we cannot verify
the optimality of these solutions. It is also interesting that the Greedy ESP
method achieves an algebraic connectivity that is extremely close to that of
\MAC{} (Madow) across all edge budgets. This is reflected in the trajectory
errors as well, where Greedy ESP and \MAC{} (Madow) are almost indistinguishable
in terms of trajectory error (with \MAC{} (Madow) occasionally performing
marginally better in terms of translation error and Greedy ESP occasionally
providing lower rotation errors). However, \MAC{} is able to achieve these results
two to three orders of magnitude faster than Greedy ESP,
depending upon the edge budget.

\begin{figure*}
\centering
    \begin{subfigure}{0.5\linewidth}
        \includegraphics[width=1.0\linewidth]{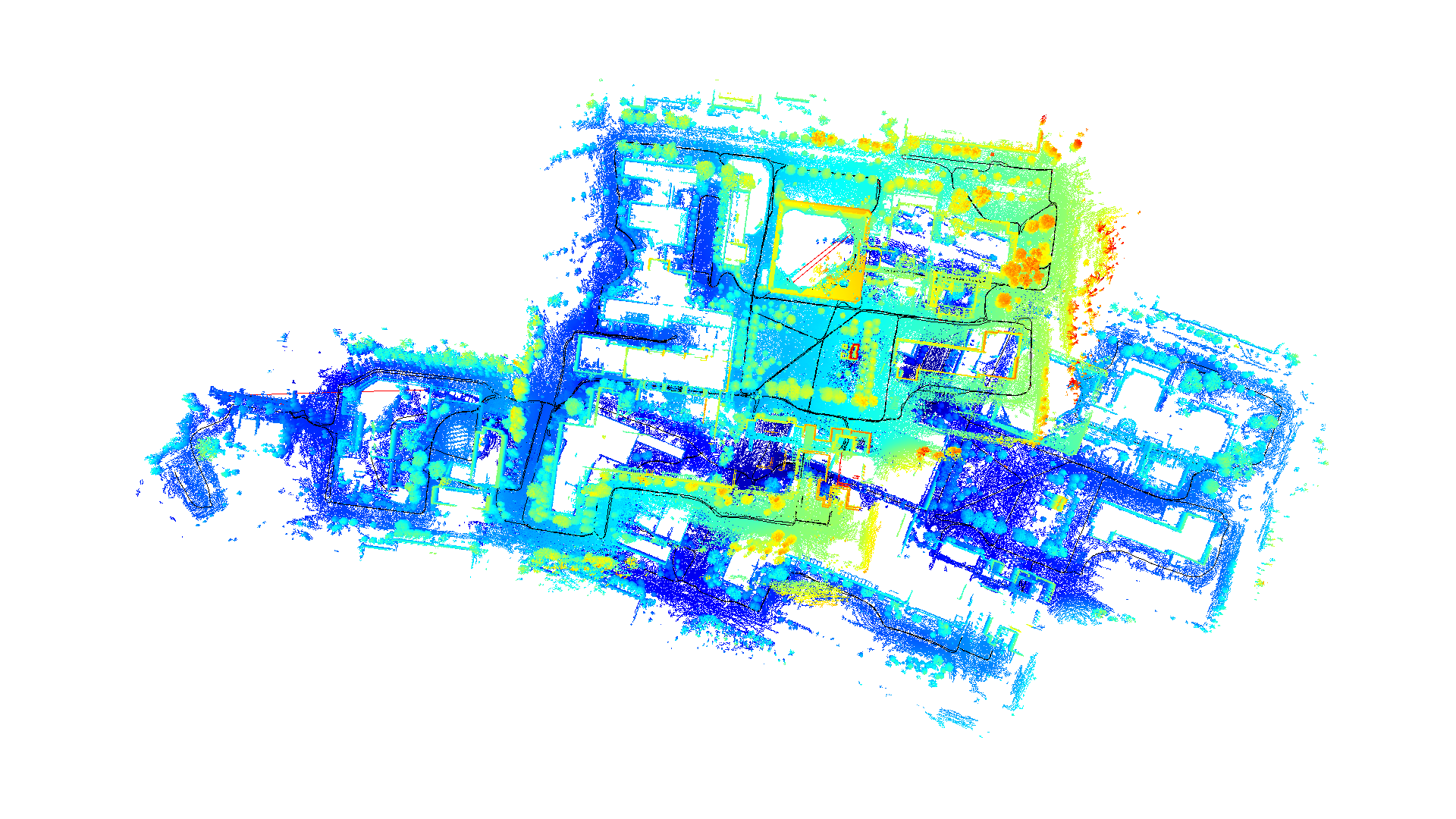}
        \caption{Na\"ive Method.}
    \end{subfigure}%
    \begin{subfigure}{0.5\linewidth}
        \includegraphics[width=1.0\linewidth]{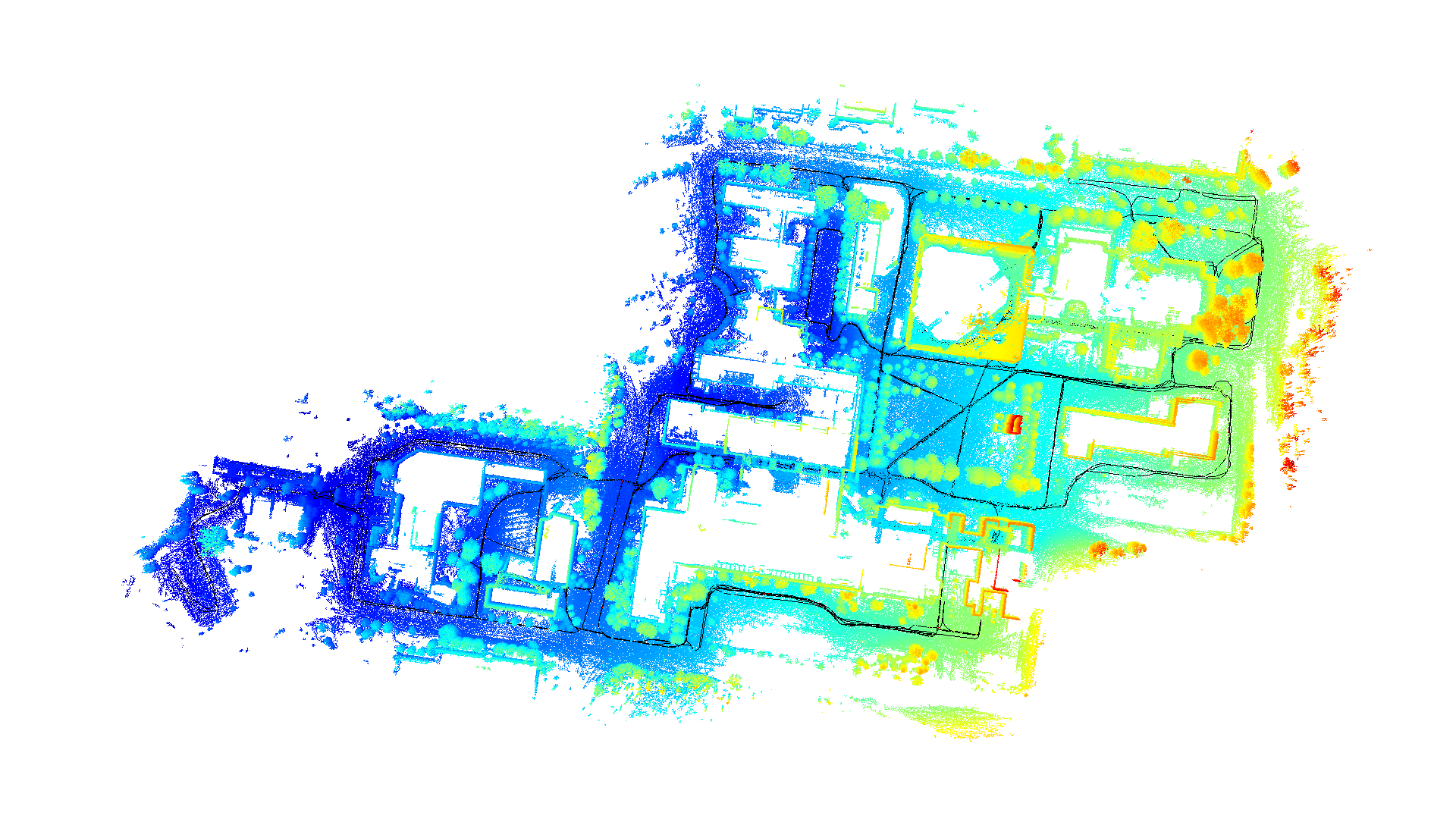}
        \caption{\MAC{} (Nearest).}
    \end{subfigure}
    \begin{subfigure}{0.5\linewidth}
        \includegraphics[width=1.0\linewidth]{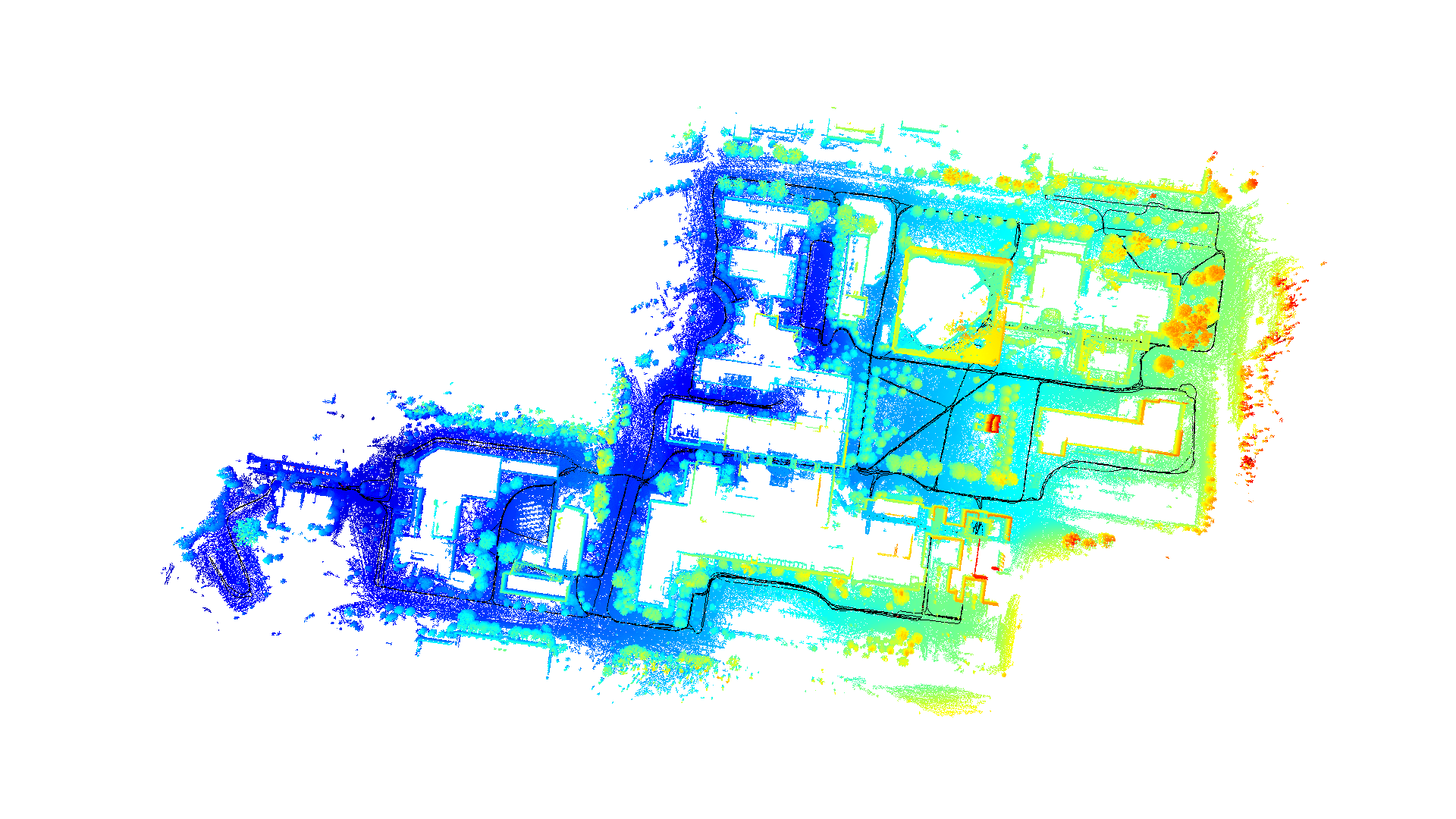}
        \caption{\MAC{} (Madow).}
    \end{subfigure}%
    \begin{subfigure}{0.5\linewidth}
        \includegraphics[width=1.0\linewidth]{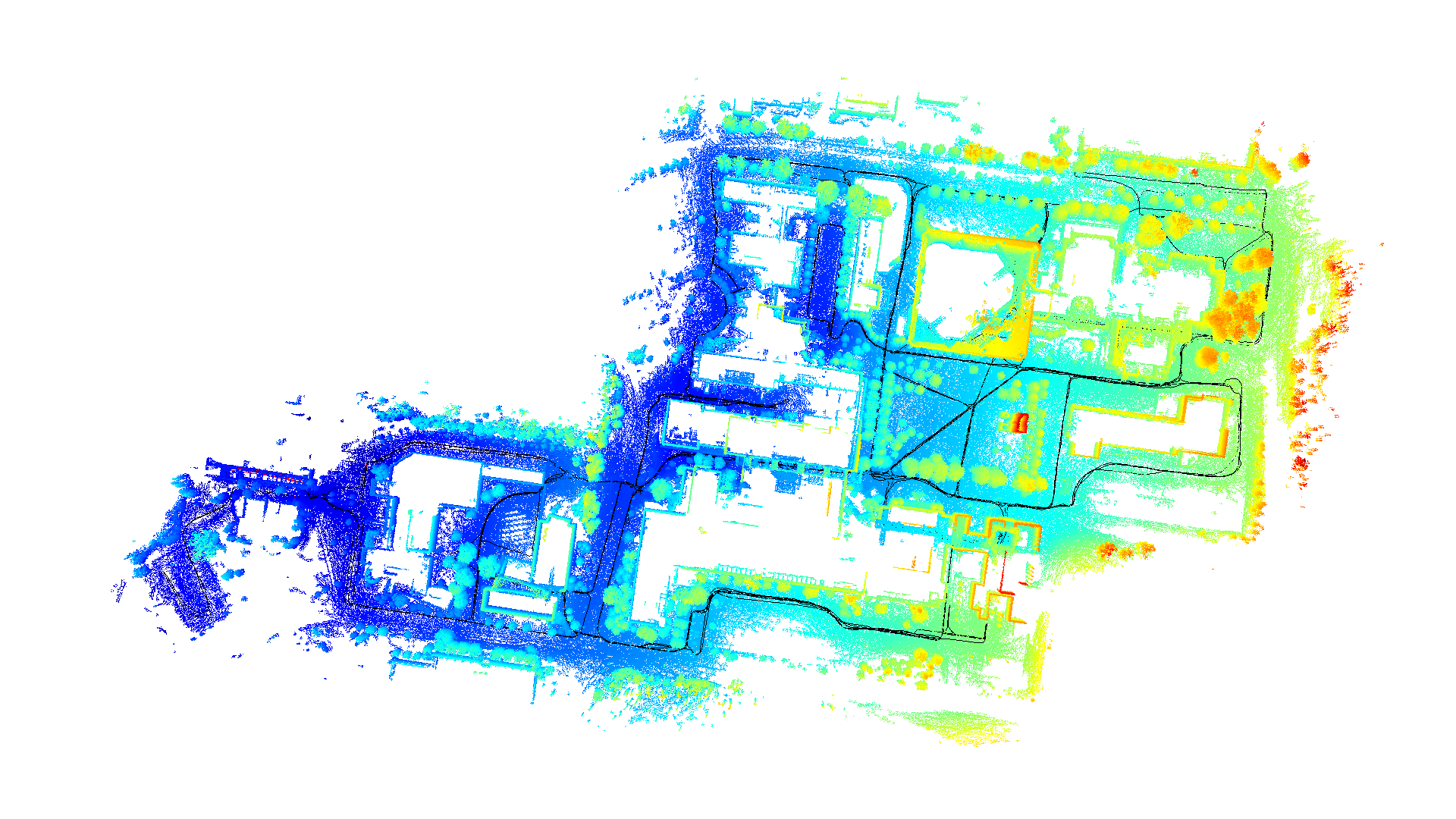}
        \caption{Greedy ESP \cite{khosoussi2019reliable}.}
    \end{subfigure}
    \caption{\textbf{Point cloud reconstructions results for the NCLT Dataset}. Top-down view of
      the LiDAR point clouds obtained from SLAM solutions computed using graphs
      sparsified by each approach to 20\% of the total loop closures in the
      dataset. Point clouds are colored by height in the $z$-axis from blue to
      red. \label{fig:nclt-pointclouds}}
\end{figure*}

\begin{figure*}
    \begin{subfigure}{0.5\linewidth}
        \includegraphics[width=0.9\linewidth]{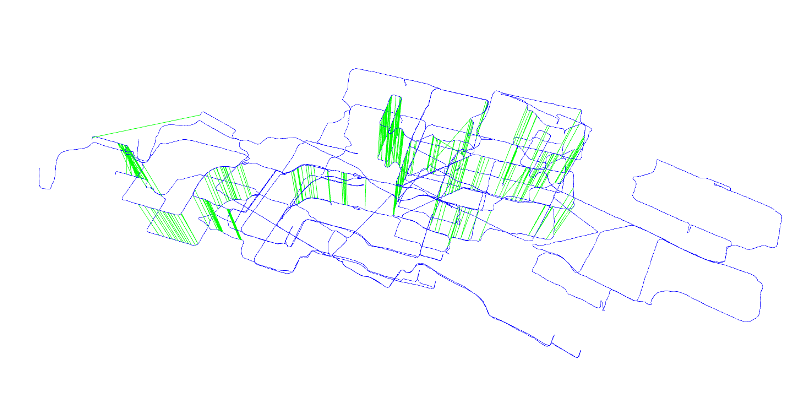}
        \caption{Na\"ive Method.}
    \end{subfigure}%
    \begin{subfigure}{0.5\linewidth}
        \includegraphics[width=1.0\linewidth]{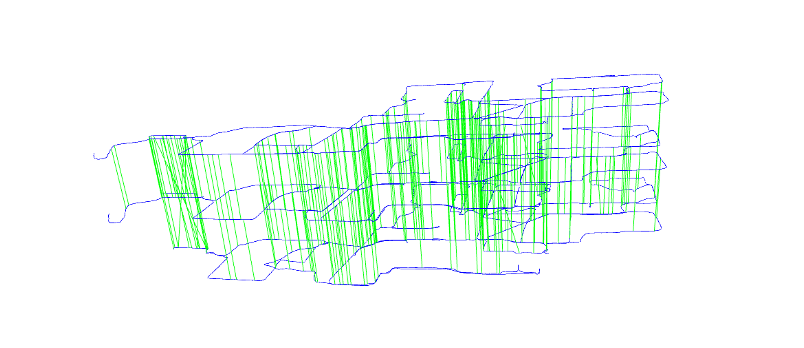}
        \caption{\MAC{} (Madow).}
    \end{subfigure}
    \caption{\textbf{Qualitative results}. Offset trajectory visualization
      comparing the results of the na\"ive baseline approach with \MAC{}
      (Madow). Separate trajectories are distributed on the $z$-axis and ordered
      temporally (i.e. with the trajectory from the first session being the
      lowest and proceeding upward). Odometry edges are displayed in blue, while
      loop closures are shown in green. Different viewpoints are used to
      highlight the impact of poor connectivity in the case of the na\"ive
      method. While the first and second sessions are relatively well-connected
      throughout, the third session is poorly anchored relative to the first
      two. In contrast, the solution obtained by \MAC{} (Madow) results in a
      qualitatively more even distribution of loop closure edges throughout, and
      in turn a higher-quality SLAM solution. \label{fig:nclt-offset-plot}}
\end{figure*}

\begin{figure*}
    \centering
    \includegraphics[width=0.25\linewidth]{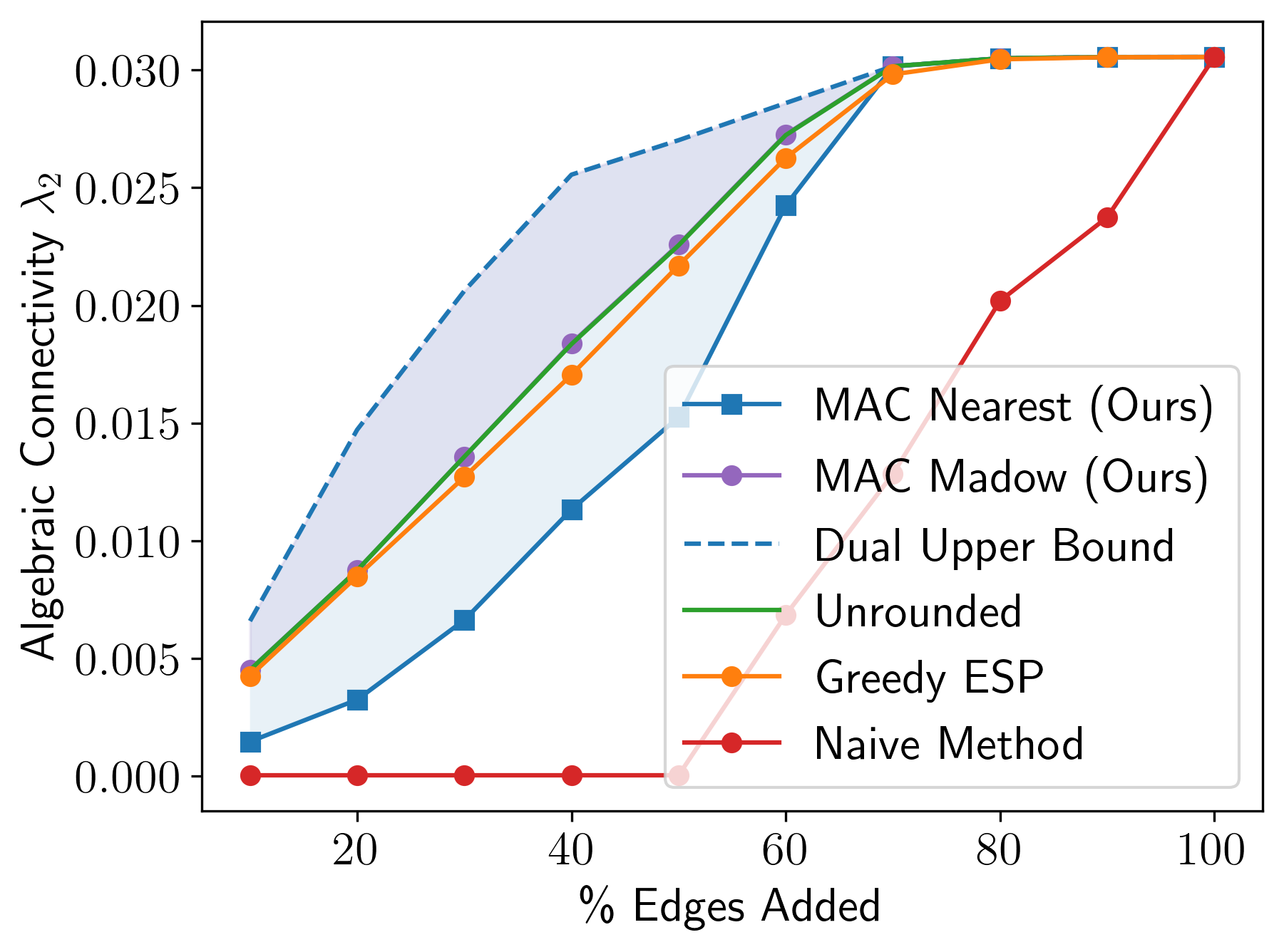}%
    \includegraphics[width=0.25\linewidth]{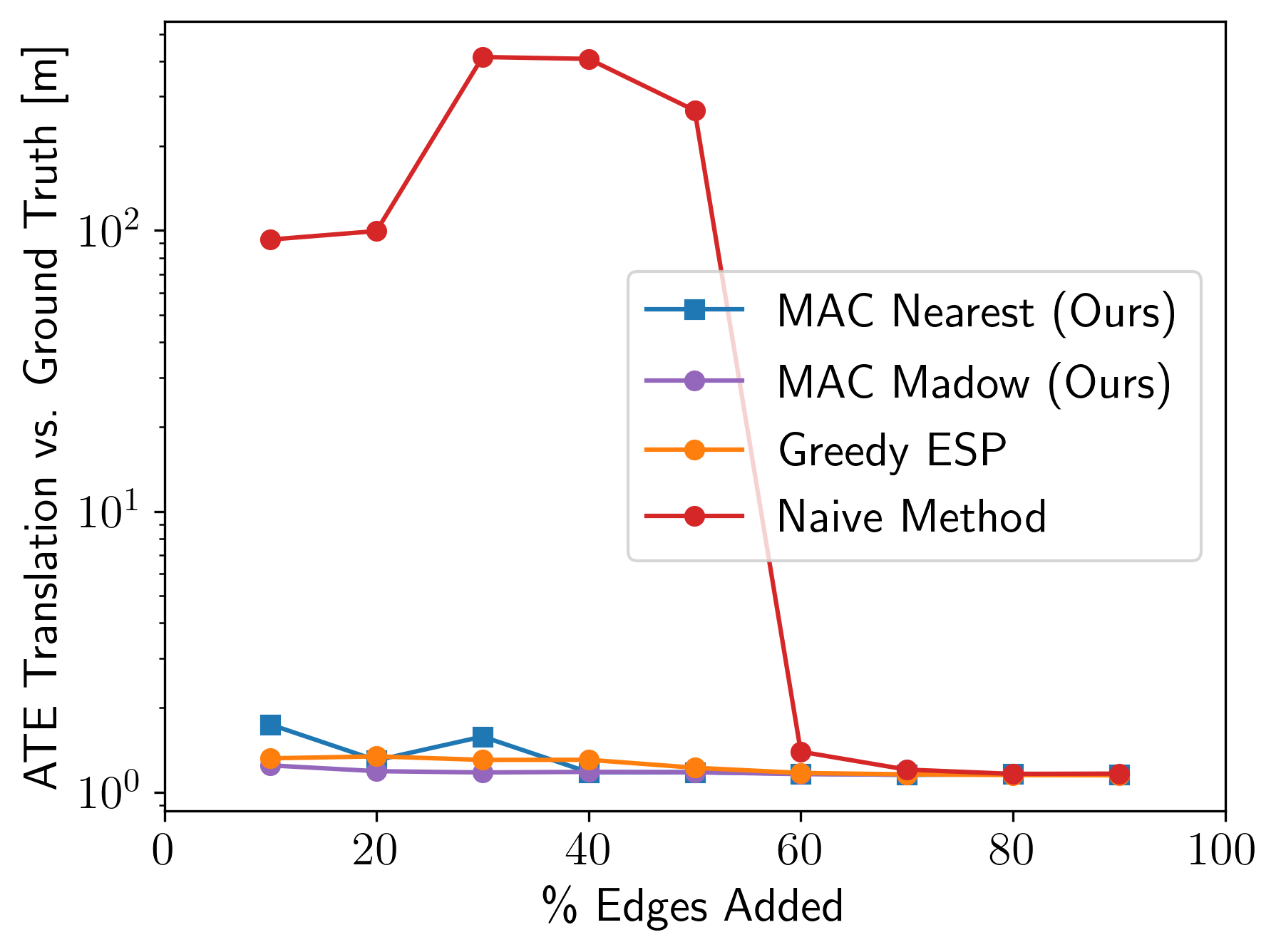}%
    \includegraphics[width=0.25\linewidth]{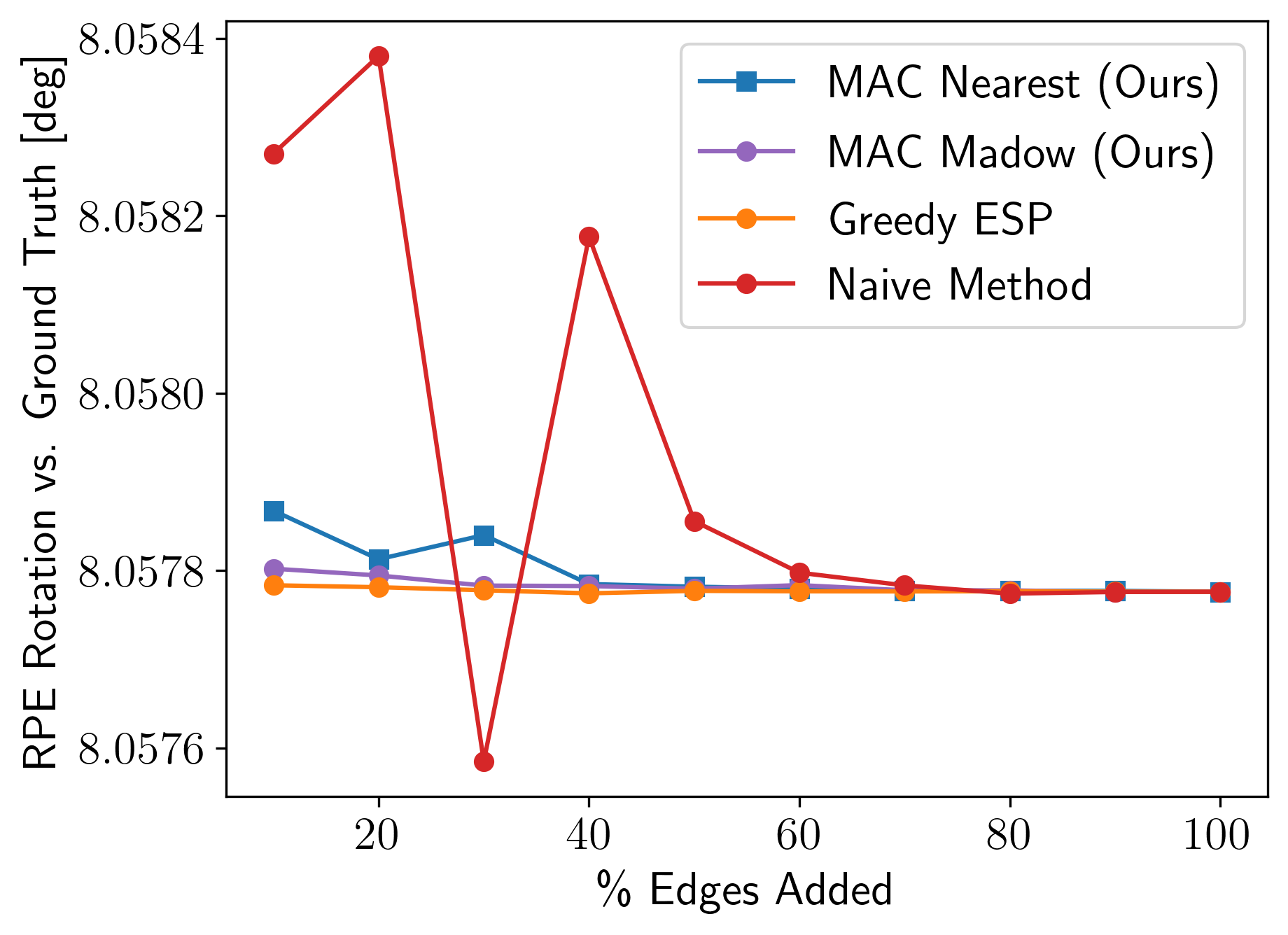}%
    \includegraphics[width=0.25\linewidth]{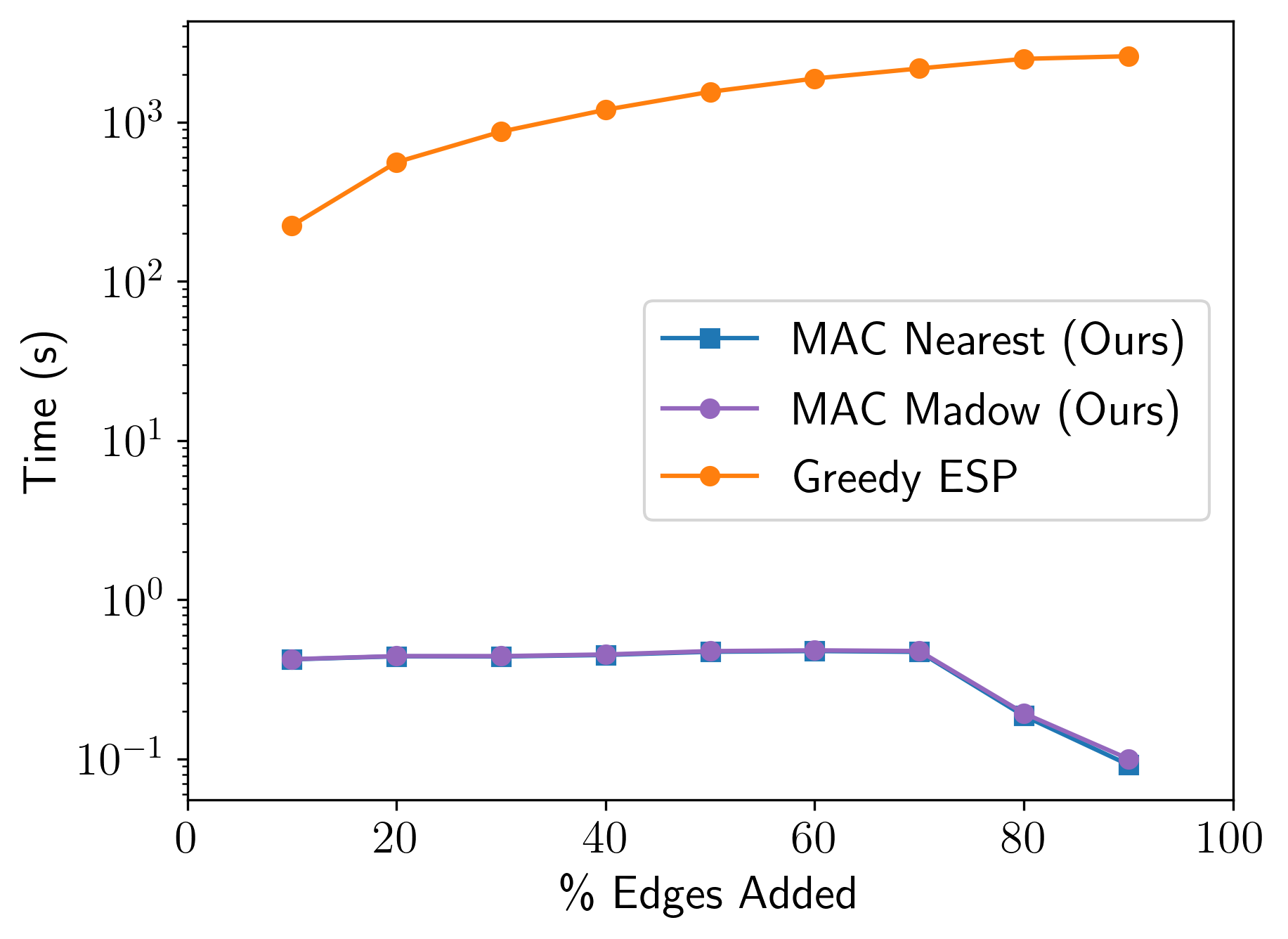}
    \caption{\textbf{Quantitative results for the NCLT Dataset}. Pose-graph
      optimization results for the NCLT dataset with varying degrees of sparsity
      (as percent of candidate edges added). Left to right: The algebraic connectivity
    of the graphs obtained using each method (larger is
    better) with the shaded regions indicating the suboptimality gap for each \MAC{} rounding procedure, the mean translation error and relative
      rotation error compared to the \emph{ground truth} (smaller is better;
      note that a logarithmic scale used for translation errors, but a linear
      scale is used for rotation errors), and the computation time (logarithmic
      scale) for each approach. For 100\% loop closures, both algorithms return
      immediately, so no computation time is reported.}
    \label{fig:nclt-quantitative}
\end{figure*}

\section{Conclusion and Future Work}
\label{sec:conclusion}

In this paper, we proposed an approach for graph sparsification by maximizing
algebraic connectivity, an important quantity appearing in many applications
related to graphs. In particular, the algebraic connectivity has been shown to
control the estimation error of pose-graph SLAM solutions, which motivates the
application of our algorithm to the setting of pose graph sparsification. Our
algorithm, \MAC{}, is based on a first-order optimization approach for solving a
convex relaxation of the maximum algebraic connectivity augmentation problem.
The algorithm itself is simple and computationally inexpensive, and, as we
showed, admits formal \emph{post hoc} performance guarantees on the quality of
the solutions it provides. In experiments on several benchmark pose-graph SLAM
datasets, our approach quickly produces high-quality sparsification results
which better preserve the connectivity of the graph and, consequently, the
quality of SLAM solutions computed using those graphs.

There are many exciting directions for future work on problems related to graph
design, algebraic connectivity, and applications to SLAM (or robotics more
broadly). From a mathematical standpoint, many practical questions about the
relationship between the algebraic connectivity maximization problem
\ref{prob:max-aug-alg-conn} and its relaxation \ref{prob:relaxation} remain. For
example, in several of the cases we observed, we were able to achieve solutions
that appear to be globally optimal with respect to the algebraic connectivity.
When this occurred, it was always in the regime where the number of edges to
select was relatively large. This was also only the case for a few datasets. It
would be interesting to know what properties of the graph under consideration
are related, in some sense, to the apparent ``difficulty'' of solving the
algebraic connectivity maximization problem by way of its relaxation in Problem
\ref{prob:relaxation}. If it is the case that solving problems where more edges
are allowed is somehow ``easier,'' could any performance benefit be gained by
solving a sequence of relaxations with increasingly restrictive values of $K$?
For example, would this help us find better solutions in the case
of smaller $K$, where we generally did not achieve verifiable global optimality
in the relaxation (or do so more quickly)?

It would also be useful to have a better understanding of
the error introduced during the rounding procedure. As we observed, the
introduction of Madow's systematic sampling method for rounding dramatically
improved the quality of our results in some cases. Is there anything we can say,
formally, about the error introduced by different rounding methods (or, for
example, whether an alternative rounding procedure might be preferred)?

The use of a randomized rounding procedure motivates some interesting additional
considerations. For example, by iterating random rounding for a
solution $\selection$ which assigns nonzero probability to every possible
$K$-sparse assignment $\hat{\selection}$, one is guaranteed to \emph{eventually}
obtain a globally optimal solution to Problem \ref{prob:max-aug-alg-conn}. This
provides an alternative perspective on the goal of the relaxation: namely, to
provide a sampling distribution which assigns high probability to globally
optimal solutions for Problem \ref{prob:max-aug-alg-conn}.

Given that our relaxation allows us to quickly compute approximate solutions
(and bounds on solution quality) for large problems, another potential area for
future research would be the use of \MAC{} in conjunction with a combinatorial
search strategy (e.g. branch-and-bound) or a mixed-integer optimization approach
(e.g. as in \cite{nagarajan2018maximizing, somisetty2023optimal}). In this
setting, \MAC{} may be able to accelerate optimization by quickly computing
(potentially coarse) bounds on solution quality while retaining the
\emph{global} optimality properties of a combinatorial method.

From an applications standpoint, we have given no suggestion in this paper about
how to select an edge budget $K$. We concede that the selection of a value of
$K$ will be highly dependent on the application setting of interest. For
example, \citet{lajoie2024swarm} determine their edge budget in Swarm-SLAM based
on communication constraints. In some cases, it may be desirable to dispose
entirely of a hard edge budget constraint, in favor of some smoother
regularization. \citet{nam2023spectral} provide an extension of \MAC{} to this
setting by regularizing based on the spectrum of the adjacency matrix, but in
order for the regularization parameter (trading off solution connectivity for
sparsity) to be interpretable, the eigenvalues of the Laplacian and adjacency
matrix appearing in their objective must be scaled by the same eigenvalues for
the full graph. In applications where the full graph is dense, this computation
may be expensive. Moreover, the nonlinearity of the Laplacian spectrum makes it
difficult to set a value for that regularization parameter which generalizes
across datasets.

Another formulation worth future investigation, inspired by the work of
\citet{khosoussi2019reliable}, would be to consider moving the algebraic
connectivity from the objective to a constraint:
\begin{equation}\label{eq:constraint-formulation}
  \begin{gathered}
    \min_{\selection \in \{0,1\}^m} \sum_{k=1}^m \selection_k \\
    \lambda_2(\selection) \geq \delta,
  \end{gathered}
\end{equation}
where $\delta$ is a parameter to be specified by a user. These kinds of
\emph{spectrally constrained} optimization problems were recently considered by
\citet{garner2023spectrally}, who give some potential solution methods for
general problems of this form, making it potentially a very fruitful starting
point for addressing this variant of the problem. Similar to
\cite{nam2023spectral}, interpreting $\delta$ requires scaling the problem (for
example, we could treat $\delta$ as a fraction of the algebraic connectivity for
the graph obtained by keeping all the edges). It would be helpful to come up
with an interpretable alternative of the problem in equation
\eqref{eq:constraint-formulation} that does not require computations related to
the full (potentially dense) graph.

There are a number of straightforward variants of Problem
\ref{prob:max-aug-alg-conn} that would lend themselves to useful applications.
For example, while we consider uniform costs for each edge, it would not be
difficult to consider edges having different (appropriately normalized) costs.
This would be useful in cases where $K$ corresponds to a communication budget
and different amounts of data are required to transmit information about
different edges, such as in multi-robot SLAM. Similarly motivated by multi-robot
SLAM applications, it is often the case that transmission of a single piece of
data is enough to establish more than one edge. This can occur when a single
keyframe from one robot can ``close the loop'' against multiple keyframes from
another robot (or vice versa). By attaching a single selection variable
$\selection_k$ to a \emph{collection} of edges that one would obtain by sharing
some data, the \MAC{} algorithm can be applied without modification to this
setting of ``coupled'' edges.\footnote{In fact, the sensor package design framework proposed by \citet{kaveti2023oasis} is explicitly formulated in terms of such multi-measurement selection variables, since any sensor selected for inclusion in a robot's sensor suite will generally produce \emph{multiple} measurements at run-time.}

The computational performance of \MAC{} motivates its potential application to
\emph{incremental} sparsification. That is, where graph edges are accumulated
dynamically over time and sparsification must be performed to ensure bounded
computation time and memory requirements. The design of incremental
sparsification methods necessarily introduces several new performance
considerations above and beyond a single edge budget parameter $K$. One might
desire a fixed edge budget, or allow for the number of edges in the graph to
grow at a prescribed rate (perhaps as a function of the number of nodes). The
application of \MAC{} to the incremental setting is also largely unchanged from
its statement in Algorithm \ref{alg:mac}: we may simply run the ``batch''
version of \MAC{} to sparsify the graph periodically as new nodes and edges are
added. Consequently, while the \MAC{} software library implements some
rudimentary examples of incremental sparsification, we do not discuss this
setting in our paper. Nonetheless, this is an exciting area for future
applications of \MAC{}.

Finally, in this work we consider only the removal of measurement graph
\emph{edges}. For applications like long-term SLAM, an important aspect of
future work will be to combine these procedures with methods for \emph{node}
removal (e.g. \cite{johannsson12rssw, carlevarisbianco13iros,
  carlone2014eliminating}). An exciting line of inquiry in this direction will
be whether methods can be developed for node sparsification (or sparsifying
nodes and edges) that preserve some guarantees on graph quality (and, in turn,
SLAM quality). \citet{loukas2019graph} presents some methods that
may serve as starting points in this direction.

\section*{Acknowledgments}

The authors thank Drs. Kasra Khosoussi and Tonio Ter\'an Espinoza for their
feedback and enthusiastic discussions during the development of this work. We
also thank Dr. Khosoussi for his advice during our implementation of the Greedy
ESP method.

\appendices

\section{Proofs of the Main Theorems}

\subsection{Supergradients of the Fiedler value}\label{app:gradients}

In this subsection we prove Theorem \ref{thm:supergradient}, which provides a simple formula for a supergradient of  $\objectiveF(\selection) = \lambda_2(\LapRotW(\selection))$.

Our proof is based upon a characterization of the subdifferential of a sum of maximal eigenvalues due to \citet{overton1993optimality}, which we now briefly recall.  Let $\lambda \colon \Sym^n \to \R^n$ be the mapping that assigns to each symmetric matrix $X$ its vector of eigenvalues, counted with multiplicity and sorted in nondecreasing order:
\begin{equation}
\lambda_1(X) \le \dotsb \le \lambda_n(X).
\end{equation}
Similarly, given $\kappa \in [n]$, we write $\sigma_{\kappa} \colon \Sym^n \to \R$ for the (convex) function that assigns to each symmetric matrix $X$ the sum of its $\kappa$ largest eigenvalues:
\begin{equation}
\sigma_{\kappa}(X) \triangleq \sum_{i = 1}^{\kappa} \lambda_{n - i + 1}(X).
\end{equation}
Finally, for $d \ge 1$ and $0 \le \tau \le d$, define the set:
\begin{equation}
\Phi_{d, \tau} \triangleq \left \lbrace U \in \Sym^d \mid 0 \preceq U \preceq I, \: \tr(U) = \tau \right \rbrace.
\end{equation}

\begin{thm}[Theorem 3.5 of \cite{overton1993optimality}]
\label{subdifferential_of_sum_of_maximum_eigenvalues_thm}
Let $X \in \Sym^n$ and 
\begin{equation}
\label{symmetric_eigendecomposition_for_subdifferential_of_sum_of_maximal_eigenvalues}
X = Q \Diag(\lambda) Q\transpose
\end{equation}
be a symmetric eigendecomposition of $X$.  Fix $\kappa \in [n]$, and suppose that $X$'s eigenvalues $\lambda \triangleq \lambda(X)$ satisfy:
\begin{multline}
\lambda_n \ge \dotsb \ge \lambda_{n - r + 1} \\
> \lambda_{n - r} = \dotsb = \lambda_{n - \kappa + 1} = \dotsb =  \lambda_{n - (r+t) + 1} \\
> \lambda_{n - (r+t) + 2} \ge \dotsb \ge \lambda_1;
\end{multline}
that is, there are $r$ eigenvalues strictly greater than the $\kappa$th-largest eigenvalue $\lambda_{n - \kappa + 1}$, and $\lambda_{n - \kappa + 1}$ has multiplicity $t$.  Finally, let $Q_l \in \R^{n \times r}$ denote the final $r$ columns of $Q$ (i.e.\ the eigenvectors associated with the top $r$-dimensional invariant subspace of $X$), and $Q_p$ the preceding block of $t$ columns (i.e.\ the set of $t$ eigenvectors associated with the $\kappa$th eigenvalue $\lambda_{n - \kappa + 1}$). Then the subdifferential of $\sigma_\kappa$ at $X$ is:
\begin{equation}
\partial \sigma_{\kappa}(X) = \left \lbrace Q_l Q_l\transpose + Q_p U Q_p \transpose \mid U \in \Phi_{t, \kappa - r} \right \rbrace. 
\end{equation}
\end{thm}

\begin{proof}[Proof of Theorem \ref{thm:supergradient}]
Let 
\begin{equation}
\begin{gathered}
\eta \colon [0,1]^m \to \R \\
\eta(x) \triangleq \sigma_2(-L(x))
\end{gathered}
\end{equation}
where $L \colon [0,1]^m \to \PSD^n$ is the PSD-matrix-valued affine map defined in \eqref{graph_laplacian_function_eq}.  Then
\begin{equation}
\label{relation_between_eta_and_negative_lambda_2}
\eta(x) = -\lambda_1(L(x)) - \lambda_2(L(x)) = -\lambda_2(L(x))
\end{equation}
since $\lambda_1(L(x)) = 0$ as $L(x)$ is a graph Laplacian.  Note that \eqref{relation_between_eta_and_negative_lambda_2} implies $f(x) = \lambda_2(L(x))$ is concave, since $\eta(x) = -f(x)$ is convex.  Moreover, since the domain of $\sigma_2(x)$ is all of $\R^n$ and $L(x)$ is smooth, the subdifferential chain rule (cf.\ Example 7.27 and Theorem 10.6 of \cite{Rockafellar2009}) implies:
\begin{equation}
\label{subdifferential_of_eta}
\partial \eta(x) = \left \lbrace -\nabla L(x)^*[W] \mid W \in \partial \sigma_2(-L(x))\right \rbrace,
\end{equation}
where $^*$ denotes the adjoint.

Now let $q_2 \in \R^n$ be any normalized eigenvector of $\lambda_2(L(x))$ that is orthogonal to $\ones$.  Since 
\begin{equation}
\label{relation_between_eigenvalues_of_L_and_negative_L}
\lambda_{n - i + 1}(-L(x)) = -\lambda_{i}(L(x)) \quad \quad \forall i \in [n]
\end{equation}
and $\ones$ is an eigenvector of $\lambda_1(L(x)) = 0$, we may factor $-L(x)$ as:
\begin{equation}
\label{factorization_for_subdifferential}
-L(x) = Q \Diag\left[\lambda(-L(x))\right] Q\transpose,
\end{equation}
where $Q$ has the form:
\begin{equation}
\label{basis_matrix_for_proof_of_subdifferential}
Q = 
\begin{pmatrix}
P & q_2 & \frac{1}{\sqrt{n}} \ones
\end{pmatrix} \in \Orthogonal(n).
\end{equation}

We claim that:
\begin{equation}
\label{element_of_subdifferential_of_sigma2}
V \triangleq \frac{1}{n}\ones \ones\transpose + q_2q_2\transpose \in \partial \sigma_2(-L(x)).
\end{equation}
To prove this, we will consider two cases.

\textbf{Case 1: $\lambda_2(L(x)) > 0$.}  Since $\lambda_1(L(x)) = 0$ (as $L(x)$ is a graph Laplacian), equation \eqref{relation_between_eigenvalues_of_L_and_negative_L} shows that the maximum eigenvalue of $-L(x)$ is simple.  Consequently, applying Theorem \ref{subdifferential_of_sum_of_maximum_eigenvalues_thm} with factorization \eqref{factorization_for_subdifferential}--\eqref{basis_matrix_for_proof_of_subdifferential} (and $r = 1$), we find that:
\begin{equation}
\label{subdifferential_for_lambda_2_greater_than_0}
\partial \sigma_2(-L(x)) =\left \lbrace \frac{1}{n}\ones \ones\transpose + Q_p U Q_p\transpose \mid U \in \Phi_{t,1} \right \rbrace.
\end{equation}
In particular, taking
\begin{equation}
U = 
\begin{pmatrix}
0_{(t-1) \times (t-1)} & \\
& 1
\end{pmatrix} \in \Phi_{t,1}
\end{equation}
in \eqref{subdifferential_for_lambda_2_greater_than_0} and recalling \eqref{basis_matrix_for_proof_of_subdifferential} establishes \eqref{element_of_subdifferential_of_sigma2}.

\textbf{Case 2:  $\lambda_2(L(x)) = 0$.}  In this case $\lambda_1(L(x)) =  \lambda_2(L(x)) = 0$.  Consequently, applying Theorem \ref{subdifferential_of_sum_of_maximum_eigenvalues_thm} with factorization \eqref{factorization_for_subdifferential}--\eqref{basis_matrix_for_proof_of_subdifferential} (and $r = 0$), we find that:
\begin{equation}
\label{subdifferential_for_lambda_2_equal_to_0}
\partial \sigma_2(-L(x)) =\left \lbrace Q_p U Q_p\transpose \mid U \in \Phi_{t,2} \right \rbrace.
\end{equation}
In particular, taking
\begin{equation}
U = 
\begin{pmatrix}
0_{(t-2) \times (t-2)} & \\
& I_2
\end{pmatrix} \in \Phi_{t,2}
\end{equation}
in \eqref{subdifferential_for_lambda_2_equal_to_0} and recalling \eqref{basis_matrix_for_proof_of_subdifferential} establishes \eqref{element_of_subdifferential_of_sigma2}.

Equations \eqref{subdifferential_of_eta} and \eqref{element_of_subdifferential_of_sigma2} together imply that:
\begin{equation}
y \triangleq -\nabla L(x)^*[V] \in \partial \eta(x).
\end{equation}
Moreover, it follows from \eqref{graph_laplacian_function_eq} that the elements of $y$ are given explicitly by:
\begin{equation}
\label{elementwise_definition_for_y}
\begin{split}
y_k &= \left \langle -\LapRotWC_{k}, \: \frac{1}{n}\ones \ones\transpose + q_2 q_2\transpose \right \rangle \\
&= -\frac{1}{n} \ones\transpose \LapRotWC_{k}\ones - q_2\transpose \LapRotWC_{k} q_2\\
&= - q_2\transpose \LapRotWC_{k} q_2,
\end{split}
\end{equation}
since $\ones \in \ker(\LapRotWC_{k})$ for all $k$.  Comparing \eqref{elementwise_definition_for_y} and \eqref{eq:supergradient} shows that $y = -g$, and thus $-g \in \partial \eta(x)$.  Finally, recalling that $\eta(x) = -\lambda_2(L(x))$ [cf.\ \eqref{relation_between_eta_and_negative_lambda_2}] shows that $g$ as defined in \eqref{eq:supergradient} is a supergradient of $\lambda_2(L(x))$, as claimed.
\end{proof}

\subsection{Solving the direction-finding subproblem (Problem \ref{prob:dir-subproblem})}\label{app:dir-subproblem}

This appendix aims to prove the claim that \eqref{eq:dir-subproblem-opt}
provides an optimal solution to the linear program in Problem
\ref{prob:dir-subproblem}.

\begin{proof}[Proof of Theorem \ref{thm:dir-subproblem}]
Rewriting the objective from Problem \ref{prob:dir-subproblem} in terms
of the elements of $s$ and $\supergradient(\selection)$, we have:
\begin{equation}\label{eq:dir-subproblem-objective-deriv}
  \begin{aligned}
    \supergradient\transpose s &= \sum_{k=1}^m \supergradient_k s_k, \\
    &= \sum_{k=1}^m s_k q_2\transpose \LapRotWC_k q_2,
  \end{aligned}
\end{equation}
where in the last line we have used the definition of $\supergradient_k$ in
\eqref{eq:supergradient}. Now, since the component-wise Laplacian matrices
$\LapRotWC_k \succeq 0$ for each edge $k$, every component of the supergradient
must always be nonnegative, i.e. $\supergradient_k \geq 0$. Further,
since $0 \leq s_k \leq 1$, the objective in
\eqref{eq:dir-subproblem-objective-deriv} is itself a sum of nonnegative terms.
From this, it follows directly that the objective in
\eqref{eq:dir-subproblem-objective-deriv} is maximized (subject to the
constraint that $\sum_{k=1}^m s_k = K$) specifically by selecting (i.e., by
setting $s_k = 1$) each of the $K$ largest components of
$\supergradient(\selection)$, giving the result in
\eqref{eq:dir-subproblem-opt}.
\end{proof}

\bibliographystyle{IEEEtranN}
\begin{footnotesize}
\bibliography{references}
\end{footnotesize}

\vfill

\end{document}